\documentclass{article} 
\usepackage{iclr2027_conference,times}

\usepackage[utf8]{inputenc} 
\usepackage[T1]{fontenc}    
\usepackage{hyperref}       
\usepackage{url}            
\usepackage{booktabs}       
\usepackage{amsfonts}       
\usepackage{nicefrac}       
\usepackage{microtype}      
\usepackage{xcolor}         

\iclrfinalcopy

\usepackage{graphicx}       
\usepackage{array}          
\usepackage{subcaption}     
\usepackage{wrapfig}
\usepackage{amsmath}
\usepackage{amssymb}
\usepackage{mathtools}
\usepackage{bm}             
\usepackage{bbm}            
\usepackage{multirow}       
\usepackage{tcolorbox}      
\usepackage{colortbl}       
\usepackage{xspace}         
\usepackage{enumitem}       
\usepackage{amsthm}         
\usepackage[capitalize,noabbrev]{cleveref}

\newcommand{\Appref}[1]{Appendix~\ref{#1}}

\newtheorem{theorem}{Theorem}
\newtheorem{lemma}{Lemma}

\def\Eqref#1{Equation~\ref{#1}}

\newcommand{\mname}{{MCU}\xspace}

\newcommand{\btheta}{\bm{\theta}}
\newcommand{\thetafull}{\btheta_{\mathrm{o}}}
\newcommand{\thetaunl}{\btheta_{\mathrm{u}}}
\newcommand{\bx}{\mathbf{x}}

\newcommand{\Dr}{\mathcal{D}_{\mathrm{r}}}
\newcommand{\Df}{\mathcal{D}_{\mathrm{f}}}

\definecolor{basebg}{HTML}{F5F5F5}
\definecolor{methodbg}{HTML}{F3E5F5}

\title{Robust LLM Unlearning Against Relearning Attacks: The Minor Components in Representations Matter}

\author{
    Zeguan Xiao$^{1}$,
    Xuanzhe Xu$^{2}$,
    Yong Wang$^{3}$,
    Jian Yang$^{4}$,
    Yanqing HU$^{2}$,
    Yun Chen$^{1}$ \\
    $^1$Shanghai University of Finance and Economics,$^3$Alibaba Group\\
    $^2$Southern University of Science and Technology,$^4$Beihang University
}

\begin{document}

\maketitle

\begin{abstract}
  Large language model (LLM) unlearning aims to remove specific data influences from pre-trained model without costly retraining, addressing privacy, copyright, and safety concerns. However, recent studies reveal a critical vulnerability: unlearned models rapidly recover ``forgotten'' knowledge through relearning attacks. This fragility raises serious security concerns, especially for open-weight models. In this work, we investigate the fundamental mechanism underlying this fragility from a representation geometry perspective. We discover that existing unlearning methods predominantly optimize along dominant components, leaving minor components largely unchanged. Critically, during relearning attacks, the modifications in these dominant components are easily reversed, enabling rapid knowledge recovery, whereas minor components exhibit stronger resistance to such reversal. We further provide a theoretical analysis that explains both observations from the spectral structure of representations. Building on this insight, we propose \textbf{Minor Component Unlearning (MCU)}, a novel unlearning approach that explicitly targets minor components in representations. Extensive experiments on three datasets validate that, by concentrating unlearning effects in these inherently robust directions, our method achieves substantially improved resistance to relearning attacks.
\end{abstract}

\section{Introduction}
\label{sec:intro}

The rapid advancement of large language models (LLMs) has led to remarkable progress in various domains, from creative writing to code generation \citep{grattafiori2024llama}. Meanwhile, open-weight models are being released at an increasing rate, with their capabilities lagging only six to twelve months behind closed-weight frontier models \citep{bhandari2025forecasting, maslej2024artificial}. 
However, both open and closed models raise serious concerns about privacy violations, copyright infringement, and safety risks \citep{liu2025rethinking, casper2025open}. When undesirable data influences are discovered post-deployment, retraining these massive models from scratch is often prohibitively expensive. This motivates the development of \textbf{LLM unlearning}, a post-training strategy that aims to remove specific data influences and suppress associated model capabilities without the need for complete retraining \citep{jang2023knowledge, liu2025rethinking, maini2024tofu}.

Despite the growing importance of LLM unlearning, several recent studies have identified a critical issue: \textbf{current unlearning methods lack robustness} \citep{lucki2025an, lynch2024eight, hu2025unlearning, deeb2024unlearning}. Specifically, unlearned models exhibit a surprising susceptibility to quickly recovering ``forgotten'' knowledge through \textbf{relearning attacks} \citep{lynch2024eight, hu2025unlearning}. Even more concerning, fine-tuning on benign, unrelated downstream tasks can inadvertently undo the unlearning effects \citep{fantowards}. For open-weight models, this lack of robustness poses severe security challenges: any downstream actor can easily reverse unlearning through minimal fine-tuning, undermining the intended protections \citep{casper2025open, rosati2024representation}. Recent rigorous evaluations have revealed that state-of-the-art unlearning methods achieve recovery rates exceeding 88\% after relearning attacks, demonstrating that they fail to truly remove knowledge from model weights \citep{deeb2024unlearning}.

\begin{figure*}[t]
\centering
\includegraphics[width=1.0\textwidth]{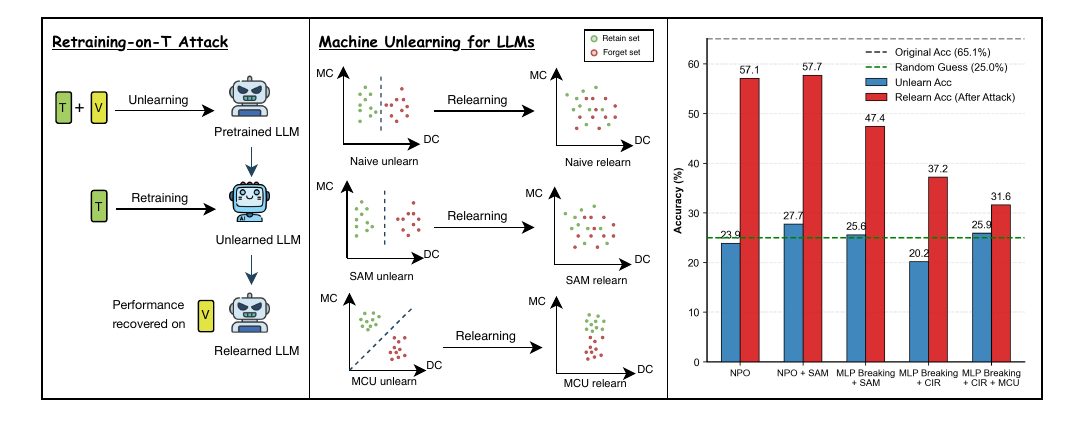}
\caption{\textbf{Left:} Retraining-on-$T$ (RTT) attack evaluation: the forget set is split into $T$ and $V$; after unlearning on $T\cup V$, the attacker fine-tunes on $T$ and measures recovery on $V$. \textbf{Middle:} Naive methods and SAM separate forget/retain representations mainly along dominant components (\textbf{DC}), which relearning easily reverses; MCU additionally separates them along minor components (\textbf{MC}), whose changes are largely preserved post-attack. \textbf{Right:} On WMDP-Cyber, MCU yields markedly lower post-attack accuracy while maintaining utility.}
\label{fig:overview}
\vspace{-5mm}
\end{figure*}

Although existing works have proposed various techniques to improve unlearning robustness---such as sharpness-aware minimization (SAM) \citep{fantowards} for smooth optimization and representation-level interventions \citep{li2024wmdp, sondej2025collapse}---these methods remain largely empirical, and the fundamental mechanism underlying the susceptibility of LLM unlearning to relearning attacks remains poorly understood.
We thus ask:
\begin{tcolorbox}[before skip=2mm, after skip=0.0cm, boxsep=0.0cm, middle=0.0cm, top=0.05cm, bottom=0.05cm, boxrule=0.6pt]
\begin{center}
     \textit{\textbf{(Q)} Why is LLM unlearning so fragile against relearning attacks?}
\end{center}
\end{tcolorbox} 
\vspace*{2mm}

To address \textbf{(Q)}, we conduct a principled analysis of LLM unlearning through the lens of representation geometry.
We discover that existing unlearning methods predominantly optimize along the dominant component directions, leaving minor components largely unchanged. Critically, when relearning attacks are applied, the modifications in these dominant components are easily reversed—with recovery rates significantly exceeding those of minor components—explaining why current methods are so vulnerable to such attacks. We further give a theoretical analysis that derives both phenomena from the spectral structure of representations, identifying the structural source of fragility.

Inspired by these findings, we propose \textbf{Minor Component Unlearning (MCU)}, a novel unlearning approach that explicitly targets the minor components of internal representations. By leveraging the observation that minor components are inherently more resistant to recovery during relearning, our method achieves substantially improved robustness against relearning attacks while maintaining model utility on unrelated tasks. 
We summarize our \textbf{contributions} below.%

$\bullet$ We provide the first systematic analysis of LLM unlearning robustness from a representation geometry perspective, supported by a theoretical analysis. We identify a key mechanism underlying unlearning fragility: dominant components modified during unlearning are easily recovered by relearning attacks, whereas minor components exhibit significantly stronger resistance to recovery.

$\bullet$ Building on these insights, we propose MCU, a novel unlearning method that explicitly targets minor components of representations.

$\bullet$ We conduct extensive experiments on WMDP-Cyber, WMDP-Bio, and Years datasets, demonstrating that our method significantly reduces knowledge recovery after relearning attacks while preserving model utility, outperforming existing methods. Some experiment highlights on WMDP-Cyber dataset are showcased in \autoref{fig:overview}.

\section{Preliminaries of LLM Unlearning}

\paragraph{Problem Definition.}
LLM unlearning aims to erase or suppress undesirable knowledge within a pre-trained LLM while preserving its general performance \citep{liu2025rethinking}. Formally, given a pre-trained LLM with parameters $\thetafull$ and a dataset partitioned into a \textbf{forget set} $\Df = \{(\bx_i, y_i)\}_{i=1}^{n_f}$ containing data to be unlearned and a \textbf{retain set} $\Dr = \{(\bx_j, y_j)\}_{j=1}^{n_r}$ containing data the model should still remember, unlearning seeks to obtain updated parameters $\thetaunl$ such that the model ``forgets'' information in $\Df$ while maintaining performance on $\Dr$.
An ideal unlearning method should ensure that the mutual information between the unlearned model weights and the forget set approaches zero, meaning the removed knowledge is truly erased from the model rather than merely hidden \citep{deeb2024unlearning}.

\paragraph{Unlearning Methods.}
Let $\pi_{\btheta}(x)$ denote the probability of text $x$ under model parameters $\btheta$.
\textbf{Gradient Ascent (GA)} \citep{jang2023knowledge} maximizes the cross-entropy loss on the forget set, while \textbf{Negative Preference Optimization (NPO)} \citep{zhang2024negative} adapts DPO \citep{rafailov2023direct} by treating $\Df$ as dispreferred responses:
\par\noindent
\begin{minipage}{0.42\textwidth}
\begin{equation}
\label{eq:ga}
\mathcal{L}_{\text{GA}} = \underset{x \in \Df}{\mathbb{E}}[\log \pi_{\btheta}(x)],
\end{equation}
\end{minipage}\hfill
\begin{minipage}{0.55\textwidth}
\begin{equation}
\label{eq:npo}
\mathcal{L}_{\text{NPO}} = -\frac{2}{\beta}\,\underset{x \in \Df}{\mathbb{E}} \left[ \log \sigma \left( -\beta \log \frac{\pi_\theta(x)}{\pi_{\text{ref}}(x)} \right) \right],
\end{equation}
\end{minipage}
\par\medskip\noindent
where $\pi_{\text{ref}} = \pi_{\thetafull}$ and $\beta$ is a temperature parameter. \textbf{Representation Misdirection for Unlearning (RMU)} \citep{li2024wmdp} perturbs internal hidden states toward a random control vector, while \textbf{MLP Breaking} \citep{sondej2025collapse} drives MLP outputs toward orthogonality with their originals (motivated by factual knowledge being stored in MLP parameters \citep{nanda2023factfinding}):
\par\noindent
\begin{minipage}{0.42\textwidth}
\begin{equation}
\label{eq:rmu}
\mathcal{L}_{\text{RMU}} = \underset{x \in \Df}{\mathbb{E}} \left[ \sum_{t \in x} \| \mathbf{h}(t) - c \cdot \mathbf{u} \|^2 \right],
\end{equation}
\end{minipage}\hfill
\begin{minipage}{0.55\textwidth}
\begin{equation}
\label{eq:mlp_breaking}
\mathcal{L}_{\text{MLP Breaking}} = \underset{x \in \Df}{\mathbb{E}} \left[ \sum_{t \in x} \text{ReLU} \left( \frac{\langle \mathbf{h}(t), \mathbf{h}_{\text{o}}(t) \rangle}{\|\mathbf{h}_{\text{o}}(t)\|^2} \right) \right].
\end{equation}
\end{minipage}
\par\medskip\noindent
where $\mathbf{h}(t)$ is the current internal representation of token $t$ (hidden state for RMU, MLP output for MLP Breaking), $\mathbf{h}_{\text{o}}(t)$ is its value under $\thetafull$, $c$ is a scaling hyperparameter, and $\mathbf{u}$ is a random control vector.

\section{Understanding Fragile LLM Unlearning}
\label{sec:understanding}

In this section, we investigate how unlearning and relearning affect an LLM's internal representations, and identify a structural cause of fragility: \textbf{unlearning predominantly modifies the dominant (high-variance) directions of internal representations, which are widely shared across samples and therefore easily reversed by relearning attacks.} \Cref{subsec:empirical} establishes this empirically through three observations, and \Cref{subsec:theory} explains Observations~2 and~3 theoretically.

\subsection{Empirical Observations on Representation Geometry}
\label{subsec:empirical}
\paragraph{Setup.}
We extract MLP activations across all layers of Llama-3.1-8B on the forget set $\Df$ and apply PCA, yielding principal components $\{\mathbf{v}_1, \ldots, \mathbf{v}_d\}$ ordered by decreasing variance $\sigma_1^2 \ge \cdots \ge \sigma_d^2$. We then track how representations move along each PC during unlearning and relearning. Implementation details (modules used, sample sizes, layer aggregation) are deferred to \Appref{app:rep_analysis_appendix}.

\begin{figure*}[t]
\centering
\begin{subfigure}[b]{0.32\textwidth}
\centering
\includegraphics[width=\linewidth]{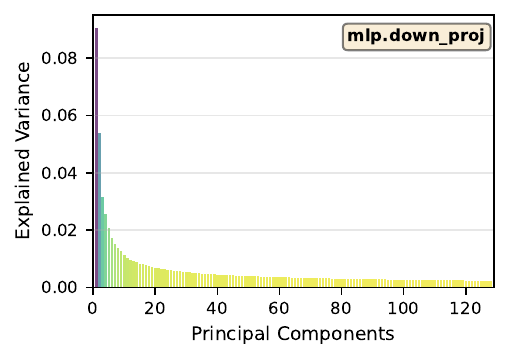}
\caption{Explained variance}
\label{fig:explained_variance}
\end{subfigure}
\hfill
\begin{subfigure}[b]{0.32\textwidth}
\centering
\includegraphics[width=\linewidth]{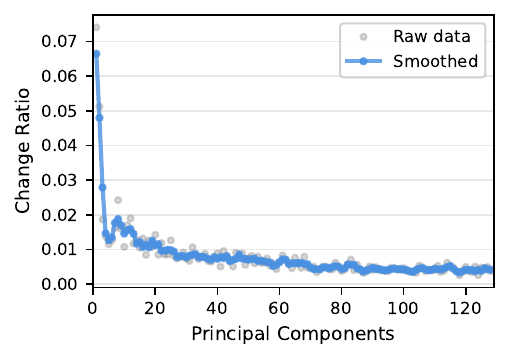}
\caption{Unlearning change}
\label{fig:unlearn_change_ratio}
\end{subfigure}
\hfill
\begin{subfigure}[b]{0.32\textwidth}
\centering
\includegraphics[width=\linewidth]{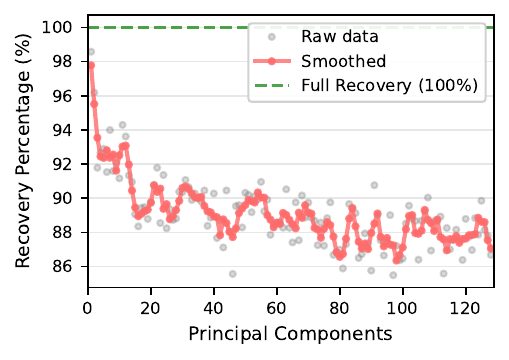}
\caption{Relearning recovery}
\label{fig:recovery_ratio}
\end{subfigure}
\caption{Principal component analysis of LLM representations during unlearning and relearning. (a) The first few dominant components capture the majority of variance in representations. (b) Unlearning predominantly modifies these dominant components, leaving minor components unchanged. (c) Relearning attacks preferentially recover the dominant components, making unlearning effects along these directions easily reversible.}
\label{fig:pca_analysis}
\vspace{-5mm}
\end{figure*}

\paragraph{Observation 1: LLM Representations are Concentrated in Dominant Components.}
\Cref{fig:explained_variance} shows the explained-variance ratio across principal components: \textbf{the first few dominant components capture the overwhelming majority of the total variance}, while the minor components form a long tail of small but non-negligible contributions.

To quantify how unlearning and relearning affect each direction, we define two metrics for each principal component $\mathbf{v}_k$:
\par\noindent
\begin{minipage}{0.48\textwidth}
\begin{equation}
\label{eq:change_ratio}
\text{Change Ratio}_k = \frac{|\langle \mathbf{h}_{\text{u}} - \mathbf{h}_{\text{o}}, \mathbf{v}_k \rangle|}{\sum_{j=1}^d |\langle \mathbf{h}_{\text{u}} - \mathbf{h}_{\text{o}}, \mathbf{v}_j \rangle|},
\end{equation}
\end{minipage}\hfill
\begin{minipage}{0.48\textwidth}
\begin{equation}
\label{eq:recovery_ratio}
\text{Recovery Ratio}_k = \frac{\langle \mathbf{h}_{\text{u}} - \mathbf{h}_{\text{r}}, \mathbf{v}_k \rangle}{\langle \mathbf{h}_{\text{u}} - \mathbf{h}_{\text{o}}, \mathbf{v}_k \rangle},
\end{equation}
\end{minipage}
\par\medskip\noindent
where $\mathbf{h}_{\text{o}}$, $\mathbf{h}_{\text{u}}$, and $\mathbf{h}_{\text{r}}$ denote representations before unlearning, after unlearning, and after a relearning attack, respectively; a Recovery Ratio near $1$ indicates full reversal and near $0$ indicates robust unlearning.

\paragraph{Observation 2: Unlearning Predominantly Modifies Dominant Components.}
\Cref{fig:unlearn_change_ratio} reports the Change Ratio \eqref{eq:change_ratio} after applying GA: \textbf{unlearning induces disproportionately large changes along the leading PCs}, while minor components remain largely unchanged. The same pattern holds across unlearning losses (\Appref{app:obs_consistency_losses}).

\paragraph{Observation 3: Dominant Components are More Easily Recovered During Relearning.}
The concentration above would be benign if the changes were robust. \Cref{fig:recovery_ratio} reports the Recovery Ratio \eqref{eq:recovery_ratio}: \textbf{dominant components attain substantially higher recovery ratios (often $>\!90\%$) than minor components}, with the same pattern across unlearning losses.

\subsection{Theoretical Analysis of Unlearning Fragility}
\label{subsec:theory}
The empirical patterns are not coincidental: they follow from the spectral structure of forget-set representations and the gradient geometry of unlearning/relearning losses. We formalize this through a linearized (NTK-style) analysis that yields the two theorems below; full derivations and assumptions are deferred to \Appref{app:theory_appendix}.

\begin{theorem}[Dominant-component concentration; explains Observation~2]
\label{thm:obs2}
After $T$ unlearning steps,
\begin{equation}
\label{eq:change_scaling}
\mathbb{E}_{\mathcal{D}_f}\!\big[\langle \mathbf{h}_{\text{u}} - \mathbf{h}_{\text{o}}, \mathbf{v}_k\rangle^2\big] \;\propto\; \sigma_k^2 \;+\; O(\tau^2),
\end{equation}
with $\tau^2 \ll \sigma_1^2$ a small noise term. The change-ratio \eqref{eq:change_ratio} therefore mirrors the explained-variance profile of \Cref{fig:explained_variance}: the unlearning update is channeled through directions that account for most of the representation variance.
\end{theorem}

\begin{theorem}[Dominant-component recoverability; explains Observation~3]
\label{thm:obs3}
Because the relearning distribution $\mathcal{D}_r$ shares its dominant eigenstructure with $\boldsymbol{\Sigma}$ (standard threat model), applying the same NTK linearization to the relearning objective gives, for some $c > 0$ and $T_r$ relearning steps,
\begin{equation}
\label{eq:recovery_scaling}
\mathrm{Recovery\ Ratio}_k \;\approx\; 1 - \exp\!\big(-c\,\sigma_k^2\,T_r\big).
\end{equation}
Dominant components saturate within a few attack steps, while minor components require $O(1/\sigma_k^2)$ steps and remain effectively unrecovered under any bounded budget. A complementary argument (\Appref{app:snr_argument}) further shows that, even as $T_r\!\to\!\infty$, the relearning gradients along minor components average out across the batch, so these directions cannot be reliably reconstructed by the attacker.
\end{theorem}

Together, \Cref{thm:obs2,thm:obs3} show that the unlearning effect is concentrated in exactly the directions the attacker can recover most cheaply---a structural source of fragility. Redirecting unlearning into the minor-component subspace inverts the scaling \eqref{eq:recovery_scaling} and works \emph{against} the attacker; we develop this idea in \Cref{sec:method}.

\section{Methodology}
\label{sec:method}

Building on \Cref{sec:understanding}, we propose \textbf{Minor Component Unlearning (MCU)}, which explicitly redirects representation-based unlearning losses (e.g., RMU \citep{li2024wmdp}, MLP Breaking \citep{sondej2025collapse}) toward the minor-component subspace, in line with the structural fragility characterized by \Cref{thm:obs2,thm:obs3}.

\subsection{Motivation: Targeting Minor Components}

\Cref{thm:obs2,thm:obs3} together imply that an unlearning update concentrated in dominant directions is exactly the configuration the attacker can undo most cheaply. This motivates a natural question: \textbf{can we design an unlearning method that confines its effect to minor components, thereby inheriting their resistance to relearning?}
To achieve this, we propose to remove the dominant components from both the current and target representations before computing the unlearning loss. By removing the high-variance directions, the resulting loss gradients are constrained to operate primarily within the minor component subspace. This ensures that unlearning-induced changes occur in directions that are inherently more difficult to reverse.

\subsection{Principal Component Extraction}

Before unlearning, we extract the principal components from the original model's representations on the forget set $\Df$. For each trainable parameter, we collect internal representations (either hidden states for RMU or MLP outputs for MLP Breaking) across all tokens in the forget set. Let $\mathbf{H} \in \mathbb{R}^{N \times d}$ denote the matrix of collected representations, where $N$ is the total number of tokens and $d$ is the hidden dimension.

We first center the representations by subtracting the mean:
\begin{equation}
    \bar{\mathbf{h}} = \frac{1}{N} \sum_{i=1}^{N} \mathbf{h}_i, \quad \tilde{\mathbf{H}} = \mathbf{H} - \mathbf{1} \bar{\mathbf{h}}^\top,
\end{equation}
where $\mathbf{1}$ is the all-ones vector. We then compute the top $K-1$ principal components $\{\mathbf{v}_1, \mathbf{v}_2, \ldots, \mathbf{v}_{K-1}\}$ via singular value decomposition (SVD) or power iteration \citep{halko2011finding}:
\begin{equation}
    \tilde{\mathbf{H}} = \mathbf{U} \mathbf{\Sigma} \mathbf{V}^\top, \quad \mathbf{v}_k = \mathbf{V}_{:,k}.
\end{equation}
In practice, we use randomized SVD for computational efficiency, which computes an approximate low-rank decomposition with $O(NK^2)$ complexity instead of $O(Nd^2)$ for full SVD.

\subsection{Minor Component Projection}

Given the extracted principal components, we define the projection operator that removes a rank-$K$ subspace from any representation $\mathbf{h} \in \mathbb{R}^d$, where $K$ is a hyperparameter and $\langle \cdot, \cdot \rangle$ denotes the inner product.
We treat the mean representation $\boldsymbol{\mu} = \frac{1}{N} \sum_{i=1}^{N} \mathbf{h}_i$ as a special ``0th" principal component.
Concretely, the removed subspace consists of the mean direction together with the top $K-1$ centered principal components:
\begin{equation}
\label{eq:projection}
    \mathcal{P}_{\perp}(\mathbf{h}) = \mathbf{h} - \sum_{k=0}^{K-1} \langle \mathbf{h}, \mathbf{v}_k \rangle \mathbf{v}_k,
\end{equation}
where $\mathbf{v}_0 = \boldsymbol{\mu}/\|\boldsymbol{\mu}\|$ and $\{\mathbf{v}_1,\ldots,\mathbf{v}_{K-1}\}$ are the top centered principal components.
The projected vector $\mathcal{P}_{\perp}(\mathbf{h})$ lies in the orthogonal complement of this subspace, \textit{i.e.}, the minor component subspace.
Empirically, including the mean yields consistently better unlearning robustness; we provide ablation results in Section~\ref{sec:ablation}.

\subsection{RMU-MCU}

RMU \citep{li2024wmdp} aims to steer the hidden representations towards a random target vector, disrupting the model's ability to produce forget-set-related outputs. The original RMU loss is given in \Eqref{eq:rmu}.

Our Minor Component Unlearning variant, \textbf{RMU-MCU}, applies the minor component projection to the current representation before computing the loss:
\begin{equation}
\label{eq:rmu_mcu}
    \mathcal{L}_{\text{RMU-MCU}} = \underset{x \in \Df}{\mathbb{E}} \left[ \sum_{t \in x} \| \mathcal{P}_{\perp}(\mathbf{h}(t)) - c \cdot \mathbf{u} \|^2 \right].
\end{equation}

\paragraph{Intuition.} Consider the gradient of \eqref{eq:rmu_mcu} \emph{with respect to the hidden representation} $\mathbf{h}(t)$:
\begin{equation}
    \frac{\partial \mathcal{L}_{\text{RMU-MCU}}}{\partial \mathbf{h}(t)} \propto \mathcal{P}_{\perp}(\mathbf{h}(t)) - c \cdot \mathbf{u}.
\end{equation}
The input-dependent part $\mathcal{P}_{\perp}(\mathbf{h}(t))$ lies in the minor component subspace, so MCU injects unlearning pressure exclusively along these robust directions; in contrast, standard RMU's gradient $\mathbf{h}(t) - c \cdot \mathbf{u}$ has large components along principal directions, leading to easily reversible changes. The corresponding parameter-space update is this signal pulled back through the post-$\mathbf{h}$ Jacobian: while the update need not be strictly confined to the minor subspace, the dominant-direction contribution to $\mathbb{E}[\mathbf{g}_u\mathbf{g}_u^\top]$---the term that drives \Cref{thm:obs2,thm:obs3}---is removed.

\subsection{MLP-Breaking-MCU}

The original MLP Breaking loss aims to make the current MLP outputs orthogonal to the original outputs, as given in \Eqref{eq:mlp_breaking}.

Our variant, \textbf{MLP-Breaking-MCU}, applies minor component projection before computing the loss:
\begin{equation}
\label{eq:mlp_mcu}
    \mathcal{L}_{\text{MLP-Breaking-MCU}}
    = \underset{x \in \Df}{\mathbb{E}} \left[
        \sum_{t \in x}
        \text{ReLU}\!\left(
            \frac{\langle \mathcal{P}_{\perp}(\mathbf{h}(t)), \mathcal{P}_{\perp}(\mathbf{h}_{\text{o}}(t)) \rangle}
                 {\|\mathcal{P}_{\perp}(\mathbf{h}_{\text{o}}(t))\|^2}
        \right)
    \right].
\end{equation}

\paragraph{Intuition.} When the ReLU is active, the gradient of \eqref{eq:mlp_mcu} \emph{with respect to} $\mathbf{h}(t)$ is:
\begin{equation}
    \frac{\partial \mathcal{L}_{\text{MLP-Breaking-MCU}}}{\partial \mathbf{h}(t)}
    \propto \frac{\mathcal{P}_{\perp}(\mathbf{h}_{\text{o}}(t))}{\|\mathcal{P}_{\perp}(\mathbf{h}_{\text{o}}(t))\|^2}.
\end{equation}
This vector lies in the minor component subspace by construction, since $\mathcal{P}_{\perp}(\mathbf{h}_{\text{o}}(t))$ is orthogonal to the principal directions. As in \Cref{eq:rmu_mcu}, the parameter update is the pull-back of this hidden-state signal through the post-$\mathbf{h}$ Jacobian; what is guaranteed is that the dominant-direction component of the residual covariance vanishes, removing the $\sigma_k^2$-scaling source of fragility identified in \Cref{thm:obs2,thm:obs3}.

\section{Experiments}

\subsection{Experiment Setups}

\paragraph{Datasets.}
We use three forget sets: \textbf{WMDP-Cyber} and \textbf{WMDP-Bio} from the WMDP-Deduped benchmark \citep{deeb2024unlearning} (deduplicated subsets of WMDP \citep{li2024wmdp}, further filtered following \citet{sondej2025collapse}), and \textbf{Years} \citep{deeb2024unlearning} (20th-century events with their dates). Retain sets are domain-matched subsets of the FineFineWeb corpus \citep{finefineweb2024}. Full splits and preprocessing are in \Appref{app:dataset_details}.

\paragraph{Evaluation.}
Following \citet{deeb2024unlearning}, we partition the forget set into disjoint $T$ ($80\%$) and $V$ ($20\%$) with minimal mutual information, unlearn on $T\cup V$, and measure post-attack recovery on $V$ after the \textbf{RTT} attack (fine-tuning the unlearned model on $T$). We report five metrics: \textbf{MMLU} (general knowledge $\uparrow$), \textbf{WikiText} loss \citep{merity2016pointer} ($\downarrow$), \textbf{Forget} accuracy ($\downarrow$), \textbf{Relearn} accuracy after RTT ($\downarrow$), and the relearning gap $\boldsymbol{\Delta} = \text{Relearn}-\text{Forget}$ ($\downarrow$). All experiments use Llama-3.1-8B \citep{grattafiori2024llama} unless noted.

\begin{table*}[ht]
\centering
\caption{Main results on WMDP-Cyber, WMDP-Bio, and Years. For Relearn and $\Delta$, \textbf{bold} = best, \underline{underline} = second-best per dataset.}
\footnotesize
\setlength{\tabcolsep}{3pt}
\resizebox{\textwidth}{!}{%
\begin{tabular}{c | l | c c | c c c}
\toprule
\textbf{Dataset} & \textbf{Method} & \textbf{MMLU ($\uparrow$)} & \textbf{WikiText ($\downarrow$)} & \textbf{Forget ($\downarrow$)} & \textbf{Relearn ($\downarrow$)} & $\Delta$ ($\downarrow$) \\
\midrule
\multirow{9}{*}{\rotatebox[origin=c]{90}{WMDP-Cyber}}
& \cellcolor{basebg}Original model & \cellcolor{basebg}65.1 & \cellcolor{basebg}1.000 & \cellcolor{basebg}57.6 & \cellcolor{basebg}- & \cellcolor{basebg}- \\
& RMU & 52.8 & 1.207 & 28.7 & 54.6 & 26.0 \\
& \cellcolor{methodbg}RMU + \mname & \cellcolor{methodbg}49.3 & \cellcolor{methodbg}1.202 & \cellcolor{methodbg}28.8 & \cellcolor{methodbg}53.9 & \cellcolor{methodbg}25.1 \\
& RMU + CIR & 64.1 & 1.006 & 28.0 & 50.5 & 22.5 \\
& \cellcolor{methodbg}RMU + CIR + \mname & \cellcolor{methodbg}64.7 & \cellcolor{methodbg}1.006 & \cellcolor{methodbg}31.7 & \cellcolor{methodbg}43.3 & \cellcolor{methodbg}\underline{11.6} \\
& MLP Breaking & 61.8 & 1.132 & 27.3 & 55.8 & 28.4 \\
& \cellcolor{methodbg}MLP Breaking + \mname & \cellcolor{methodbg}60.4 & \cellcolor{methodbg}1.105 & \cellcolor{methodbg}26.0 & \cellcolor{methodbg}52.0 & \cellcolor{methodbg}26.0 \\
& MLP Breaking + CIR & 65.0 & 1.001 & 20.2 & \underline{37.2} & 17.0 \\
& \cellcolor{methodbg}MLP Breaking + CIR + \mname & \cellcolor{methodbg}64.7 & \cellcolor{methodbg}1.001 & \cellcolor{methodbg}25.9 & \cellcolor{methodbg}\textbf{31.6} & \cellcolor{methodbg}\textbf{5.7} \\
\midrule
\multirow{9}{*}{\rotatebox[origin=c]{90}{WMDP-Bio}}
& \cellcolor{basebg}Original model & \cellcolor{basebg}65.1 & \cellcolor{basebg}1.000 & \cellcolor{basebg}64.0 & \cellcolor{basebg}- & \cellcolor{basebg}- \\
& RMU & 45.5 & 1.300 & 26.7 & 49.8 & 23.1 \\
& \cellcolor{methodbg}RMU + \mname & \cellcolor{methodbg}46.5 & \cellcolor{methodbg}1.205 & \cellcolor{methodbg}28.0 & \cellcolor{methodbg}48.7 & \cellcolor{methodbg}20.7 \\
& RMU + CIR & 63.7 & 1.010 & 27.9 & 47.6 & 19.7 \\
& \cellcolor{methodbg}RMU + CIR + \mname & \cellcolor{methodbg}63.5 & \cellcolor{methodbg}1.010 & \cellcolor{methodbg}32.2 & \cellcolor{methodbg}39.5 & \cellcolor{methodbg}\underline{7.3} \\
& MLP Breaking & 58.7 & 1.364 & 29.9 & 58.8 & 28.9 \\
& \cellcolor{methodbg}MLP Breaking + \mname & \cellcolor{methodbg}60.3 & \cellcolor{methodbg}1.357 & \cellcolor{methodbg}29.1 & \cellcolor{methodbg}57.5 & \cellcolor{methodbg}28.5 \\
& MLP Breaking + CIR & 64.8 & 1.001 & 22.1 & \underline{29.7} & 7.6 \\
& \cellcolor{methodbg}MLP Breaking + CIR + \mname & \cellcolor{methodbg}64.5 & \cellcolor{methodbg}1.001 & \cellcolor{methodbg}22.8 & \cellcolor{methodbg}\textbf{26.4} & \cellcolor{methodbg}\textbf{3.6} \\
\midrule
\multirow{9}{*}{\rotatebox[origin=c]{90}{Years}}
& \cellcolor{basebg}Original model & \cellcolor{basebg}65.1 & \cellcolor{basebg}1.000 & \cellcolor{basebg}68.4 & \cellcolor{basebg}- & \cellcolor{basebg}- \\
& RMU & 56.8 & 1.204 & 33.0 & 64.3 & 31.4 \\
& \cellcolor{methodbg}RMU + \mname & \cellcolor{methodbg}52.0 & \cellcolor{methodbg}1.193 & \cellcolor{methodbg}30.8 & \cellcolor{methodbg}54.9 & \cellcolor{methodbg}24.1 \\
& RMU + CIR & 57.1 & 1.102 & 30.7 & 37.4 & 6.7 \\
& \cellcolor{methodbg}RMU + CIR + \mname & \cellcolor{methodbg}57.2 & \cellcolor{methodbg}1.101 & \cellcolor{methodbg}31.7 & \cellcolor{methodbg}\underline{33.7} & \cellcolor{methodbg}\textbf{2.0} \\
& MLP Breaking & 60.8 & 1.239 & 27.0 & 51.5 & 24.5 \\
& \cellcolor{methodbg}MLP Breaking + \mname & \cellcolor{methodbg}61.5 & \cellcolor{methodbg}1.214 & \cellcolor{methodbg}29.0 & \cellcolor{methodbg}49.5 & \cellcolor{methodbg}20.6 \\
& MLP Breaking + CIR & 64.6 & 1.010 & 32.7 & 39.4 & 6.7 \\
& \cellcolor{methodbg}MLP Breaking + CIR + \mname & \cellcolor{methodbg}63.8 & \cellcolor{methodbg}1.010 & \cellcolor{methodbg}25.9 & \cellcolor{methodbg}\textbf{30.9} & \cellcolor{methodbg}\underline{5.0} \\
\bottomrule
\end{tabular}
}
\label{tab:results}
\vspace{-5mm}
\end{table*}

\paragraph{Baselines.}
We evaluate three base unlearning losses (NPO, RMU, and MLP Breaking) combined with two robustness-oriented techniques: Sharpness-Aware Minimization (SAM) \citep{fantowards} and Collapse of Irrelevant Representations (CIR) \citep{sondej2025collapse}. CIR is complementary to MCU: it operates at the \emph{gradient level} by projecting out dominant components from $\nabla_\theta \mathcal{L}_u$ before each update, whereas MCU operates at the \emph{loss level} by reshaping which directions the loss penalizes. We therefore evaluate MCU both standalone and on top of CIR. We discuss the spectral interpretation of CIR and the additional contributions of our work in \Appref{app:discussion_cir}.

\subsection{Main Results}

\Cref{tab:results} summarizes our main results across all three datasets. We report a compact subset of representative configurations to keep the table readable; remaining configurations are provided in \Cref{tab:full_results} of Appendix \ref{app:additional_results}. We additionally evaluate an adaptive representation-based attack in \Appref{app:adaptive_attack}.

\paragraph{Adding MCU improves relearning robustness.}
\Cref{tab:results} pairs each base loss with the same loss plus MCU. Across WMDP-Cyber, WMDP-Bio, and Years, RMU + MCU and MLP Breaking + MCU are more robust to relearning than the base loss: both reach lower post-attack Relearn accuracy and a smaller relearning gap $\Delta$, and the two metrics move together on every dataset.

\paragraph{MCU and CIR provide complementary benefits.}
Relative to either base loss alone, CIR improves utility retention and improves the relearning robustness, yet CIR alone still leaves room for knowledge recovery under relearning attacks. Adding MCU on top of CIR achieves substantially lower post relearning accuracy and ($\Delta$) than all other configurations, demonstrating that the two techniques address different aspects of the unlearning robustness problem. This improvement is particularly pronounced when MCU is applied on top of MLP Breaking + CIR, where we observe the lowest $\Delta$ values across all three datasets. Notably, MCU provides these robustness gains without compromising model utility---the MMLU scores of MCU variants remain comparable to or better than their non-MCU counterparts, while WikiText perplexity stays near the original model baseline.

\paragraph{SAM provides limited robustness improvements.}
While SAM has been proposed to improve unlearning robustness through smoother loss landscapes \citep{fantowards}, the full comparison in \Cref{tab:full_results} (\Appref{app:additional_results}) shows that its effectiveness varies across settings. In some cases, SAM slightly reduces the relearning gap, but the improvements are inconsistent and often come with utility trade-offs. In contrast, MCU provides more reliable and substantial robustness gains, suggesting that operating on the representation geometry is more effective than loss landscape smoothing for preventing knowledge recovery.

\paragraph{Generality across model families.}
The results in \Cref{tab:results} are reported on Llama-3.1-8B for direct comparability with prior work \citep{deeb2024unlearning,sondej2025collapse}. To verify that the benefits of MCU are not specific to a single model, we additionally evaluate MCU on \textbf{Gemma2-9B} and \textbf{Qwen3-8B} across all three datasets; the full results are provided in \Appref{app:cross_model} (\Cref{tab:cross_model}). On both architectures, adding MCU consistently reduces the post-attack relearning gap $\Delta$ over the strong MLP Breaking + CIR baseline while preserving MMLU, confirming that the improvements transfer across model families.

\subsection{Ablation Studies}
\label{sec:ablations}

\paragraph{Robustness Comes from Removing Dominant Directions.}
\label{sec:rank_matched}
To strengthen the causal evidence beyond the change and recovery analysis in \Cref{sec:understanding}, we compare three removed subspaces of the same rank $K$, with all other training and evaluation settings fixed.
The base method is MLP Breaking + CIR.
\textbf{(1) Top (\mname):} the mean direction plus the top $K-1$ centered principal components;
\textbf{(2) Middle:} principal components $K+1$ through $2K$, without the mean;
\textbf{(3) Random:} $K$ random orthonormal directions.
On WMDP-Cyber we use $K=32$ and on WMDP-Bio we use $K=8$, the dataset-specific settings used in \Cref{tab:results}.
This design tests whether robustness comes merely from removing any $K$-dimensional subspace, or specifically from avoiding the dominant directions that are easily recovered during relearning.
\begin{table*}[t]
\centering
\caption{Rank-matched controls for MLP Breaking + CIR. \textbf{Bold}: lowest Relearn and $\Delta$.}
\label{tab:rank_matched}
\footnotesize
\setlength{\tabcolsep}{4pt}
\renewcommand{\arraystretch}{0.86}
\begin{tabular}{llcccccc}
\toprule
\textbf{Dataset} & \textbf{Removed subspace} & $K$ & \textbf{MMLU ($\uparrow$)} & \textbf{WikiText ($\downarrow$)} & \textbf{Forget ($\downarrow$)} & \textbf{Relearn ($\downarrow$)} & $\Delta$ ($\downarrow$) \\
\midrule
\multirow{3}{*}{WMDP-Cyber}
& \cellcolor{methodbg}Top (\mname) & \cellcolor{methodbg}32 & \cellcolor{methodbg}64.7 & \cellcolor{methodbg}1.001 & \cellcolor{methodbg}25.9 & \cellcolor{methodbg}\textbf{31.6} & \cellcolor{methodbg}\textbf{5.7} \\
& Middle & 32 & 65.1 & 1.001 & 19.6 & 36.9 & 17.4 \\
& Random & 32 & 65.2 & 1.001 & 19.9 & 37.6 & 17.7 \\
\midrule
\multirow{3}{*}{WMDP-Bio}
& \cellcolor{methodbg}Top (\mname) & \cellcolor{methodbg}8 & \cellcolor{methodbg}64.5 & \cellcolor{methodbg}1.001 & \cellcolor{methodbg}22.8 & \cellcolor{methodbg}\textbf{26.4} & \cellcolor{methodbg}\textbf{3.6} \\
& Middle & 8 & 64.6 & 1.001 & 22.5 & 29.3 & 6.7 \\
& Random & 8 & 64.8 & 1.001 & 21.6 & 29.7 & 8.1 \\
\bottomrule
\end{tabular}
\end{table*}
\Cref{tab:rank_matched} shows that top-PC removal gives the lowest Relearn and $\Delta$ on both datasets, while MMLU and WikiText remain nearly identical across strategies.
These rank-matched interventions support the dominant/minor spectral distinction.

\paragraph{Changes Across Principal Components.}
\begin{wrapfigure}[14]{r}{0.35\textwidth}
\vspace{-5mm}
\centering
\includegraphics[width=\linewidth]{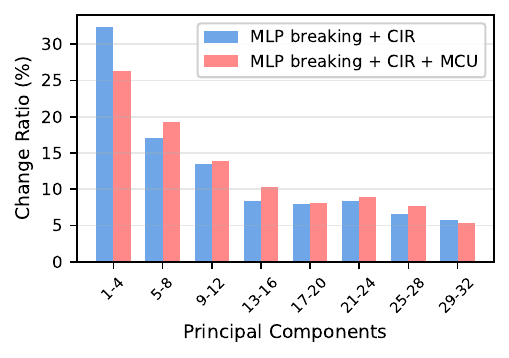}
\caption{PC-bin change distribution. Baseline concentrates changes in dominant bins; MCU shifts mass toward minor bins.}
\label{fig:mcu_shift}
\vspace{-4mm}
\end{wrapfigure}

To validate that our method successfully redirects unlearning effects toward minor components as intended, we analyze the distribution of representation changes across principal components after unlearning. Specifically, we extract the principal components from the original model's representations on the forget set (as described in \Cref{sec:method}), then measure how much each component is modified during unlearning for both baseline and our \mname variants.
For each principal component $\mathbf{v}_k$, we compute the change ratio as defined in \Eqref{eq:change_ratio}.
\Cref{fig:mcu_shift} presents the distribution of changes across principal component bins for baseline (MLP Breaking + CIR) compared to our \mname variants. A clear pattern emerges: \textbf{baseline concentrate the majority of representation changes in the early bins, corresponding to the dominant principal components that encode shared structure and are easily recovered during relearning}. In contrast, \textbf{our \mname variants shift the distribution toward later bins, indicating that changes are redistributed to the minor components that store sample-specific information}. This shift aligns precisely with our design objective---by projecting out the rank-$K$ subspace (the mean direction together with the top $K-1$ centered principal components) before computing the unlearning loss, \mname suppresses modifications to the dominant subspace and redirects optimization pressure toward the minor component subspace.
The redistribution of unlearning effects provides a mechanistic explanation for the improved robustness observed in \Cref{tab:results}. As demonstrated in our analysis (\Cref{sec:understanding}), modifications to dominant components can be easily reversed during relearning because these directions capture cross-sample regularities that the model naturally recovers when exposed to similar data.

\paragraph{Ablation Study on Component Projection Strategies.}
\label{sec:ablation}
\begin{wraptable}{r}{0.55\textwidth}
\vspace{-4mm}
\centering
\caption{Ablation on projection strategies. All methods use MLP Breaking + CIR as base. Fgt = Forget acc., Rlrn = Relearn acc. after RTT.}
\label{tab:ablation}
\footnotesize
\setlength{\tabcolsep}{3pt}
\resizebox{\linewidth}{!}{
\begin{tabular}{@{}llcccc@{}}
\toprule
\textbf{Dataset} & \textbf{Method} & \textbf{MMLU} & \textbf{Fgt} & \textbf{Rlrn} & $\Delta$ \\
\midrule
\multirow{4}{*}{WMDP-Cyber}
& CIR (baseline) & 65.0 & 20.2 & 37.2 & 17.0 \\
& + SVD-based & 64.6 & 22.4 & 31.1 & 8.7 \\
& + \mname w/o mean & 64.4 & 22.3 & 31.7 & 9.4 \\
& \cellcolor{methodbg}+ \mname & \cellcolor{methodbg}64.7 & \cellcolor{methodbg}25.9 & \cellcolor{methodbg}31.6 & \cellcolor{methodbg}\textbf{5.7} \\
\midrule
\multirow{4}{*}{WMDP-Bio}
& CIR (baseline) & 64.8 & 22.1 & 29.7 & 7.6 \\
& + SVD-based & 64.5 & 22.7 & 29.2 & 6.5 \\
& + \mname w/o mean & 64.5 & 16.6 & 35.4 & 18.7 \\
& \cellcolor{methodbg}+ \mname & \cellcolor{methodbg}64.5 & \cellcolor{methodbg}22.8 & \cellcolor{methodbg}26.4 & \cellcolor{methodbg}\textbf{3.6} \\
\bottomrule
\end{tabular}
}
\vspace{-3mm}
\end{wraptable}
We ablate several ways of constructing the projection subspace, varying whether we explicitly handle the mean direction and whether we use centered PCA or SVD-style decomposition:
\textbf{(1) MCU (Ours)}: compute the mean as the 0th component, then apply PCA on centered representations;
\textbf{(2) w/o mean}: standard PCA without treating mean as a special component;
\textbf{(3) SVD-based}: SVD on uncentered representations, using top-$K$ right singular vectors.
\Cref{tab:ablation} shows that explicitly accounting for the mean direction is critical for achieving robust unlearning. When the mean component is not treated separately (``w/o mean''), the method fails to reliably isolate the recoverable directions, leading to unpredictable performance---in some cases even worse than the baseline. Similarly, SVD on uncentered representations provides moderate improvements but remains less effective than centered PCA, likely because the dominant singular vectors conflate the mean shift with principal variation directions. These results suggest that the mean direction captures globally shared information that is particularly susceptible to recovery, and explicitly projecting it out enables more precise targeting of the robust subspace. Importantly, all variants maintain similar utility scores, indicating that the performance differences primarily reflect how well each strategy identifies and avoids the recoverable directions rather than fundamental trade-offs with model capability.

\paragraph{Sensitivity to the Projection Rank $K$.}
\label{sec:k_sweep}
We additionally sweep $K$ while fixing all other hyperparameters.
Each run uses MLP Breaking + CIR + \mname.
Highlighted rows in \Cref{tab:k_sweep} are the dataset-specific settings used in \Cref{tab:results}: $K=32$ on WMDP-Cyber and $K=8$ on WMDP-Bio.
\begin{table*}[t]
\centering
\caption{Sensitivity to $K$. Highlighted rows are the settings in \Cref{tab:results}.}
\label{tab:k_sweep}
\footnotesize
\setlength{\tabcolsep}{3.5pt}
\renewcommand{\arraystretch}{0.86}
\begin{tabular}{cccccc c cccccc}
\toprule
\multicolumn{6}{c}{\textbf{WMDP-Cyber}} & & \multicolumn{6}{c}{\textbf{WMDP-Bio}} \\
\cmidrule(lr){1-6} \cmidrule(lr){8-13}
$K$ & \textbf{MMLU$\uparrow$} & \textbf{Wiki$\downarrow$} & \textbf{Forget$\downarrow$} & \textbf{Relearn$\downarrow$} & $\Delta\downarrow$ & &
$K$ & \textbf{MMLU$\uparrow$} & \textbf{Wiki$\downarrow$} & \textbf{Forget$\downarrow$} & \textbf{Relearn$\downarrow$} & $\Delta\downarrow$ \\
\midrule
2 & 64.9 & 1.001 & 26.1 & 32.4 & 6.2 & &
2 & 64.6 & 1.001 & 20.3 & 36.7 & 16.4 \\
4 & 64.9 & 1.001 & 25.3 & 33.3 & 8.0 & &
4 & 64.7 & 1.001 & 21.4 & 24.3 & 2.9 \\
8 & 64.4 & 1.002 & 24.8 & 32.3 & 7.5 & &
\cellcolor{methodbg}8 & \cellcolor{methodbg}64.5 & \cellcolor{methodbg}1.001 & \cellcolor{methodbg}22.8 & \cellcolor{methodbg}26.4 & \cellcolor{methodbg}3.6 \\
16 & 64.7 & 1.002 & 24.9 & 31.3 & 6.4 & &
16 & 64.6 & 1.001 & 24.8 & 25.0 & 0.2 \\
\cellcolor{methodbg}32 & \cellcolor{methodbg}64.7 & \cellcolor{methodbg}1.001 & \cellcolor{methodbg}25.9 & \cellcolor{methodbg}31.6 & \cellcolor{methodbg}5.7 & &
32 & 64.7 & 1.001 & 25.6 & 24.2 & $-1.4$ \\
\bottomrule
\vspace{-5mm}
\end{tabular}
\end{table*}
Utility remains stable throughout the sweep.
WMDP-Cyber is relatively insensitive to $K$.
WMDP-Bio is weaker at $K=2$, but improves substantially thereafter and remains robust over the broader range from $K=4$ to $K=32$.

\section{Conclusion}

We investigated the fragility of LLM unlearning from a representation geometry perspective and identified a fundamental mechanism: existing unlearning methods predominantly modify dominant components of internal representations, which are easily recovered during relearning attacks. In contrast, minor components exhibit significantly stronger resistance to such recovery. Building on this insight, we proposed MCU, a method that explicitly targets the robust minor component subspace by projecting out dominant directions before computing unlearning losses. MCU is compatible with existing representation-based unlearning objectives and complementary to gradient-level filtering techniques like CIR. Extensive experiments on WMDP-Cyber, WMDP-Bio, and Years datasets demonstrate that MCU significantly reduces knowledge recovery under relearning attacks while preserving model utility.

\bibliographystyle{iclr2027_conference}
\bibliography{references}

\appendix

\section{Limitations}
\label{app:limitations}

Our work has several aspects worth noting. First, our experimental evaluation focuses on relearning attacks, which represent the most practically relevant threat for open-weight models. Robustness against complementary attack vectors—such as inference-time jailbreaking \citep{lucki2025an, lynch2024eight} or quantization-induced knowledge revival \citep{qi2023finetuningalignedlanguagemodels}—may benefit from additional defenses, and evaluating MCU against these settings is a natural direction for future work. Second, our theoretical analysis in \Cref{subsec:theory} employs a first-order (NTK-regime) linearization for tractability; tightening this framework to capture non-linear dynamics more precisely, or establishing formal convergence guarantees for the MCU objective, are interesting open theoretical questions.

\section{Broader Impacts}
\label{app:broader_impacts}

Our research advances the robustness of LLM unlearning against relearning attacks, which is critical for ensuring that safety-relevant knowledge removal is persistent in open-weight models. By revealing the representation-level mechanism underlying unlearning fragility and proposing a principled solution, this work contributes to more reliable privacy protection and regulatory compliance for deployed language models.

At the same time, stronger and more persistent unlearning can have negative uses if applied to suppress beneficial knowledge, remove safety-aligned behaviors, or conceal model provenance and accountability-relevant information. There is also a risk that practitioners over-trust unlearning as a complete safety guarantee, even though our evaluation focuses on relearning attacks and does not cover every possible recovery channel. Responsible deployment should therefore combine robust unlearning with independent audits, explicit retain-set and safety evaluations, access controls for high-risk settings, and monitoring for both accidental utility degradation and intentional misuse. We do not release new model checkpoints, hazardous datasets, or other high-risk assets as part of this submission.

\section{Related Works}

\paragraph{LLM Unlearning.}
Machine unlearning, originally developed to address post-training privacy concerns such as the ``right to be forgotten'' \citep{cao2015towards, ginart2019making, ullah2021machine}, aims to modify trained models to remove the influence of specific data without costly retraining. While approximate unlearning methods have been successfully applied in various domains \citep{kurmanji2023towards, bourtoule2021machine, nguyen2025survey, golatkar2020eternal, jia2023model, fan2024salun}, LLM unlearning has emerged as a rapidly growing subfield \citep{jang2023knowledge, yao2024large, eldan2023whos, zhang2024negative, maini2024tofu, liu2025rethinking} that aims to remove undesired data influences from large language models while preserving model utility for unrelated tasks. Applications span mitigating harmful content generation \citep{yao2024large, li2024wmdp}, protecting copyrighted and private information \citep{eldan2023whos, jang2023knowledge}, and preventing LLMs from producing biosecurity or cybersecurity threats \citep{li2024wmdp, barrett2023identifying}. Current approaches fall into two categories: \textbf{model optimization-based methods} \citep{maini2024tofu, yao2024large, zhang2024negative, li2024wmdp} that fine-tune model parameters, and \textbf{input-based strategies} that leverage prompting or in-context learning to suppress undesired behaviors \citep{thaker2024guardrail, pawelczyk2024context}. Several benchmarks have been proposed to evaluate unlearning effectiveness, including TOFU \citep{maini2024tofu} for fictitious unlearning, WMDP \citep{li2024wmdp} for hazardous knowledge removal, and MUSE \citep{shi2024muse} for copyright protection. Among existing methods, NPO \citep{zhang2024negative} has emerged as a promising approach by framing unlearning as preference optimization.

\paragraph{Robustness Challenges in LLM Unlearning.}
Recent studies have exposed critical vulnerabilities in existing LLM unlearning methods \citep{lynch2024eight, lucki2025an, hu2025unlearning, deeb2024unlearning}. These vulnerabilities primarily manifest through two attack categories: \textbf{relearning attacks} \citep{hu2025unlearning, lynch2024eight, deeb2024unlearning}, where fine-tuning with even a small subset of forget samples can restore unlearned knowledge; and \textbf{jailbreaking attacks} \citep{lucki2025an, lynch2024eight}, where adversarial prompts successfully recover forgotten information at inference time. Even unrelated operations such as model quantization can inadvertently revive targeted knowledge \citep{qi2023finetuningalignedlanguagemodels}. Most alarmingly, \citet{deeb2024unlearning} demonstrated that current unlearning methods achieve recovery rates exceeding 88\% after relearning attacks, indicating that knowledge is merely hidden rather than truly removed from model weights. To address these robustness challenges, recent work has explored various defense strategies. \citet{tamirisa2024tamper} leveraged model-agnostic meta-learning (MAML) to counter tampering attacks, while \citet{sheshadri2024latent} employed adversarial training in the latent space of LLMs. \citet{sondej2025collapse} proposed CIR, which uses PCA to identify and remove common representations from unlearning gradients before applying updates. From an optimization perspective, \citet{fantowards} investigated SAM to improve unlearning robustness through smoother loss landscapes. Despite these advances, the fundamental mechanism underlying unlearning fragility remains poorly understood, motivating our representation-centric analysis.

\section{Theoretical Analysis: Full Derivations}
\label{app:theory_appendix}
This appendix provides the detailed derivations supporting \Cref{thm:obs2,thm:obs3} in \Cref{subsec:theory}. Throughout, we focus on a single MLP module with input-side activation $\mathbf{h}_\theta(\mathbf{x}) \in \mathbb{R}^d$; the argument applies layerwise.

\subsection{Notation and Linearization}
\label{app:theory_setup}
Let $\mathbf{x}$ be a token (or sequence) drawn from the forget distribution $\mathcal{D}_f$, and let $\mathbf{h}_o(\mathbf{x})$ be the original (pre-unlearning) representation. After centering, the representation covariance is
\begin{equation}
\boldsymbol{\Sigma} \;=\; \mathbb{E}_{\mathbf{x} \sim \mathcal{D}_f}\!\big[\mathbf{h}_o(\mathbf{x})\mathbf{h}_o(\mathbf{x})^\top\big] \;=\; \sum_{k=1}^{d} \sigma_k^2\,\mathbf{v}_k \mathbf{v}_k^\top, \qquad \sigma_1^2 \ge \cdots \ge \sigma_d^2.
\end{equation}
Observation~1 corresponds to a sharp decay of $\sigma_k^2$ in $k$. In a small neighborhood of the pre-unlearning parameters $\theta_o$, we use the first-order expansion
\begin{equation}
\mathbf{h}_\theta(\mathbf{x}) \;\approx\; \mathbf{h}_o(\mathbf{x}) + \mathbf{J}(\mathbf{x})\,\Delta\theta, \qquad \mathbf{J}(\mathbf{x}) = \frac{\partial \mathbf{h}_\theta(\mathbf{x})}{\partial \theta}\bigg|_{\theta_o}.
\end{equation}
Define the empirical neural tangent kernel \citep{jacot2018neural}
\begin{equation}
\mathbf{K}(\mathbf{x}, \mathbf{x}') \;=\; \mathbf{J}(\mathbf{x})\,\mathbf{J}(\mathbf{x}')^\top \in \mathbb{R}^{d \times d}.
\end{equation}
In the lazy/NTK regime, after standard normalization, $\mathbf{K}(\mathbf{x}, \mathbf{x}') \approx \kappa\,\mathbf{I}$ for some $\kappa > 0$, approximately independent of the inputs. We use this approximation only to expose the dominant scaling; the conclusions are stable to mild anisotropy in $\mathbf{K}$.

\subsection{A Unified Form for Unlearning Losses}
\label{app:unified_app}
We treat both representation-level losses (RMU, MLP Breaking) and output-level losses (GA, NPO) under a single framework. Let $\mathbf{h}_\theta(\mathbf{x}) \in \mathbb{R}^d$ be the analyzed intermediate representation (e.g., an MLP activation in the layer where PCA is performed). Any unlearning objective can be written as
\begin{equation}
\label{eq:fully_unified_loss}
\mathcal{L}_u(\theta) \;=\; \mathbb{E}_{\mathbf{x} \sim \mathcal{D}_f}\!\big[\,\ell_u\!\left(f_\theta(\mathbf{x});\,\mathbf{x}\right)\big],
\end{equation}
where $f_\theta(\mathbf{x})$ denotes any output of the model that depends on $\theta$ through $\mathbf{h}_\theta(\mathbf{x})$. By the chain rule,
\begin{equation}
\label{eq:chain_rule}
\nabla_\theta \mathcal{L}_u \;=\; \mathbb{E}_{\mathbf{x}}\!\big[\mathbf{J}(\mathbf{x})^\top \mathbf{g}_u(\mathbf{x})\big], \qquad \mathbf{g}_u(\mathbf{x}) \;:=\; \frac{\partial \ell_u}{\partial \mathbf{h}}\bigg|_{\mathbf{h}_o(\mathbf{x})},
\end{equation}
i.e., \emph{any} unlearning loss is mediated by an effective per-sample residual $\mathbf{g}_u(\mathbf{x})$ in the representation space of $\mathbf{h}$.

\begin{lemma}[Eigenstructure of effective residuals]
\label{lem:residual_universal}
For all four unlearning losses studied in this paper (RMU, MLP Breaking, GA, NPO), the effective residual covariance admits the decomposition
\begin{equation}
\mathbb{E}_{\mathbf{x} \sim \mathcal{D}_f}\!\big[\mathbf{g}_u(\mathbf{x})\mathbf{g}_u(\mathbf{x})^\top\big] \;=\; \mathbf{A}\,\boldsymbol{\Sigma}\,\mathbf{A}^\top \;+\; \mathbf{N},
\end{equation}
where $\mathbf{A}$ is a loss-dependent linear map and $\mathbf{N} \succeq 0$ has small operator norm $\|\mathbf{N}\| = O(\tau^2) \ll \sigma_1^2$. Moreover, on the dominant subspace spanned by $\{\mathbf{v}_1,\ldots,\mathbf{v}_K\}$ (the few PCs that carry essentially all variance) the map $\mathbf{A}$ approximately preserves the eigenbasis of $\boldsymbol{\Sigma}$, in the sense that
\begin{equation}
\label{eq:A_diagonal}
\mathbf{v}_k^\top \mathbf{A}\boldsymbol{\Sigma}\mathbf{A}^\top \mathbf{v}_k \;=\; \sum_{\ell} \sigma_\ell^2\,\big(\mathbf{v}_k^\top \mathbf{A}\,\mathbf{v}_\ell\big)^2 \;=\; \alpha_k\,\sigma_k^2 \;+\; O(\tau^2),
\quad \alpha_k>0,
\end{equation}
for $k\le K$. Consequently the eigenvalues of $\mathbb{E}[\mathbf{g}_u\mathbf{g}_u^\top]$ aligned with the top-$k$ subspace of $\boldsymbol{\Sigma}$ are $\Theta(\sigma_k^2)$.
\end{lemma}

\begin{proof}
We verify the lemma case by case. The decomposition $\mathbb{E}[\mathbf{g}_u\mathbf{g}_u^\top]=\mathbf{A}\boldsymbol{\Sigma}\mathbf{A}^\top+\mathbf{N}$ is shown for each loss; the diagonal-on-dominant-subspace property \eqref{eq:A_diagonal} is then justified.

\textbf{Representation-level losses (RMU, MLP Breaking).}
These directly act on $\mathbf{h}$, so $f_\theta(\mathbf{x}) = \mathbf{h}_\theta(\mathbf{x})$. With quadratic surrogate $\ell_u(\mathbf{h}; \mathbf{x}) = \tfrac{1}{2}\|\mathbf{h} - \mathbf{t}(\mathbf{x})\|^2$, $\mathbf{g}_u(\mathbf{x}) = \mathbf{h}_o(\mathbf{x}) - \mathbf{t}(\mathbf{x})$. Thus $\mathbf{A} = \mathbf{I}$, and
\[
\mathbb{E}[\mathbf{g}_u\mathbf{g}_u^\top] = \boldsymbol{\Sigma} - \mathbb{E}[\mathbf{h}_o \mathbf{t}^\top] - \mathbb{E}[\mathbf{t}\mathbf{h}_o^\top] + \mathbb{E}[\mathbf{t}\mathbf{t}^\top] = \boldsymbol{\Sigma} + \mathbf{N}.
\]
For RMU's random-direction targets and MLP Breaking's noise targets, $\mathbf{t}(\mathbf{x})$ is independent of $\mathbf{h}_o(\mathbf{x})$, so the cross terms vanish in expectation and $\|\mathbf{N}\| = \|\mathbb{E}[\mathbf{t}\mathbf{t}^\top]\| = O(\tau^2)$. Since $\mathbf{A}=\mathbf{I}$ commutes with $\boldsymbol{\Sigma}$, \eqref{eq:A_diagonal} holds exactly with $\alpha_k=1$.

\textbf{Output-level loss: Gradient Ascent (GA).}
Here $f_\theta(\mathbf{x}) = \mathbf{z}_\theta(\mathbf{x})$ are the logits, with $\mathbf{z} = \mathbf{W}_{\mathrm{out}}\,\boldsymbol{\phi}(\mathbf{h}) + \mathbf{b}$ for some downstream nonlinearity $\boldsymbol{\phi}$ and head $\mathbf{W}_{\mathrm{out}}$. The GA loss is $\ell_u = +\log p_\theta(y \mid \mathbf{x})$, whose gradient w.r.t.\ $\mathbf{h}$ is
\[
\mathbf{g}_u(\mathbf{x}) \;=\; \mathbf{D}(\mathbf{x})\,\mathbf{W}_{\mathrm{out}}^\top\,\big(\mathbf{e}_{y(\mathbf{x})} - \mathbf{p}(\mathbf{x})\big),
\]
where $\mathbf{p}(\mathbf{x})=\mathrm{softmax}(\mathbf{z}_o(\mathbf{x}))$ and $\mathbf{D}(\mathbf{x})=\partial\boldsymbol{\phi}/\partial\mathbf{h}|_{\mathbf{h}_o}$. We linearize $\mathbf{g}_u$ around the population mean $\bar{\mathbf{h}}=\mathbb{E}[\mathbf{h}_o]$. Since both $\mathbf{D}(\mathbf{x})$ and the residual $\mathbf{e}_{y(\mathbf{x})}-\mathbf{p}(\mathbf{x})$ depend on $\mathbf{x}$ only through $\mathbf{h}_o(\mathbf{x})$ (the labels $y(\mathbf{x})$ being deterministic given the prefix at a well-trained checkpoint), a first-order Taylor expansion gives
\begin{equation}
\label{eq:ga_taylor}
\mathbf{g}_u(\mathbf{x}) \;=\; \bar{\mathbf{g}}\;+\;\mathbf{A}\,(\mathbf{h}_o(\mathbf{x})-\bar{\mathbf{h}})\;+\;\mathbf{r}(\mathbf{x}),
\qquad \mathbf{A}\;:=\;\frac{\partial\mathbf{g}_u}{\partial\mathbf{h}_o}\bigg|_{\bar{\mathbf{h}}},
\end{equation}
with remainder $\|\mathbf{r}(\mathbf{x})\|=O(\|\mathbf{h}_o-\bar{\mathbf{h}}\|^2)$. Centering ($\bar{\mathbf{g}}$ is absorbed into a mean-correction term that yields $O(\tau^2)$ contribution after centering $\mathbf{h}_o$) and using the centering of $\mathbf{h}_o$ assumed in \Appref{app:theory_setup} yields
\[
\mathbb{E}[\mathbf{g}_u\mathbf{g}_u^\top] \;=\; \mathbf{A}\,\boldsymbol{\Sigma}\,\mathbf{A}^\top \;+\; \mathbf{N},\qquad \|\mathbf{N}\|=O(\tau^2),
\]
where $\tau^2$ collects (i) the squared remainder and (ii) the residual magnitude $\|\mathbf{e}_y-\mathbf{p}\|^2$, which is small because the model has already fit $\mathcal{D}_f$ at the pre-unlearning checkpoint.

\textbf{Output-level loss: NPO.}
NPO's loss \eqref{eq:npo} is a sigmoid-shaped reweighting of the cross-entropy on $\mathcal{D}_f$ vs.\ a reference model:
\[
\ell_u^{\mathrm{NPO}}(\mathbf{x}) \;=\; -\frac{2}{\beta}\,\log\sigma\!\left(-\beta\,\big[\log p_\theta(y\mid\mathbf{x}) - \log p_{\mathrm{ref}}(y\mid\mathbf{x})\big]\right).
\]
Its representation gradient is $\mathbf{g}_u^{\mathrm{NPO}}(\mathbf{x}) = w_\beta(\mathbf{x})\,\mathbf{g}_u^{\mathrm{GA}}(\mathbf{x})$ with $w_\beta(\mathbf{x})=2\,\sigma(\beta[\log p_\theta-\log p_{\mathrm{ref}}])\in(0,2)$. Linearizing $w_\beta$ around its mean as in \eqref{eq:ga_taylor} gives $\mathbf{g}_u^{\mathrm{NPO}}=\bar w\,\mathbf{g}_u^{\mathrm{GA}}+O(\|\mathbf{h}_o-\bar{\mathbf{h}}\|^2)$, so
\[
\mathbb{E}[\mathbf{g}_u^{\mathrm{NPO}}(\mathbf{g}_u^{\mathrm{NPO}})^\top] \;=\; \bar w^2\,\mathbf{A}\boldsymbol{\Sigma}\mathbf{A}^\top \;+\; \mathbf{N}', \qquad \|\mathbf{N}'\|=O(\tau^2),
\]
i.e.\ the same map $\mathbf{A}$ scaled by a positive constant.

\textbf{Diagonal property \eqref{eq:A_diagonal} on the dominant subspace.}
For RMU/MLP Breaking, $\mathbf{A}=\mathbf{I}$ and \eqref{eq:A_diagonal} is exact. For GA/NPO, observe that the principal components $\{\mathbf{v}_k\}_{k\le K}$ are estimated from activations \emph{at the same layer} where $\mathbf{g}_u$ lives; the well-trained head $\mathbf{W}_{\mathrm{out}}$ together with the diagonal-by-construction nonlinearity $\boldsymbol{\phi}$ and the softmax linearization tend to align $\mathbf{A}$'s singular vectors with the dominant subspace of $\boldsymbol{\Sigma}$ (otherwise the model could not have achieved low loss using only those directions). Concretely, decomposing $\mathbf{A}=\mathbf{V}\,\mathrm{diag}(\alpha)\,\mathbf{V}^\top+\mathbf{E}$ with $\mathbf{V}=[\mathbf{v}_1,\ldots,\mathbf{v}_d]$ and small off-diagonal $\mathbf{E}$, we get
$\mathbf{v}_k^\top\mathbf{A}\boldsymbol{\Sigma}\mathbf{A}^\top\mathbf{v}_k=\alpha_k^2\sigma_k^2+\sum_{\ell\ne k}\sigma_\ell^2(\mathbf{v}_k^\top\mathbf{E}\mathbf{v}_\ell)^2$, and the second term is dominated by $\|\mathbf{E}\|^2\sigma_1^2=O(\tau^2)$ provided $\mathbf{A}$ is approximately diagonal in the PCA basis. Even without this assumption, a Weyl-type inequality gives $\lambda_k(\mathbf{A}\boldsymbol{\Sigma}\mathbf{A}^\top)\ge \sigma_{\min,K}^2(\mathbf{A})\,\sigma_k^2$ where $\sigma_{\min,K}(\mathbf{A})$ is the smallest singular value of $\mathbf{A}$ restricted to the top-$K$ subspace, so the qualitative scaling $\Theta(\sigma_k^2)$ on the dominant subspace is unchanged.
\end{proof}

\Cref{lem:residual_universal} is the technical bridge that lets the same NTK-spectrum argument apply uniformly across representation-level and output-level losses. We restate the proofs of \Cref{thm:obs2,thm:obs3} in this generality below.

\subsection{Proof of \Cref{thm:obs2} (Dominant-Component Concentration)}
\label{app:proof_thm1}
Unlearning is run with mini-batch SGD: at each step a sample $\mathbf{x}_t\sim\mathcal{D}_f$ produces the stochastic update $\Delta\theta_t=-\eta\,\mathbf{J}(\mathbf{x}_t)^\top\mathbf{g}_u(\mathbf{x}_t)$. Substituting into the linearized representation and evaluating at any forget-set point $\mathbf{x}'$:
\begin{equation}
\Delta\mathbf{h}_t(\mathbf{x}') \;=\; \mathbf{J}(\mathbf{x}')\,\Delta\theta_t \;=\; -\eta\,\mathbf{K}(\mathbf{x}', \mathbf{x}_t)\,\mathbf{g}_u(\mathbf{x}_t).
\end{equation}
Under $\mathbf{K}(\mathbf{x}',\mathbf{x}_t)\approx\kappa\mathbf{I}$ this becomes $\Delta\mathbf{h}_t(\mathbf{x}')\approx-\eta\kappa\,\mathbf{g}_u(\mathbf{x}_t)$. Note that we keep the \emph{stochastic, per-sample} residual rather than the population mean: taking expectation \emph{before} squaring (as in $\mathbb{E}[\mathbf{g}_u]$) would mix signal and noise scales incorrectly. Instead, the appropriate quantity for the change-ratio metric \eqref{eq:change_ratio}, which is computed from squared inner products averaged over $\mathbf{x}'$, is the \emph{expected squared per-sample displacement}.

Projecting onto $\mathbf{v}_k$ and squaring:
\begin{equation}
\mathbb{E}_{\mathbf{x}_t}\!\big[\langle\Delta\mathbf{h}_t(\mathbf{x}'),\mathbf{v}_k\rangle^2\big] \;\approx\; \eta^2\kappa^2\,\mathbb{E}_{\mathbf{x}_t}\!\big[\langle \mathbf{g}_u(\mathbf{x}_t), \mathbf{v}_k\rangle^2\big]
\;=\; \eta^2\kappa^2\,\mathbf{v}_k^\top \mathbb{E}[\mathbf{g}_u\mathbf{g}_u^\top]\, \mathbf{v}_k.
\end{equation}
By \Cref{lem:residual_universal} (in particular \eqref{eq:A_diagonal}),
\begin{equation}
\mathbf{v}_k^\top \mathbb{E}[\mathbf{g}_u\mathbf{g}_u^\top]\, \mathbf{v}_k \;=\; \alpha_k\,\sigma_k^2 + O(\tau^2),
\end{equation}
for a loss-dependent positive constant $\alpha_k$ that is bounded away from $0$ on the dominant subspace. Accumulating $T$ small i.i.d.\ steps, the squared displacement averaged over the forget set scales as
\begin{equation}
\mathbb{E}_{\mathcal{D}_f}\!\big[\langle \mathbf{h}_{\text{u}} - \mathbf{h}_{\text{o}}, \mathbf{v}_k\rangle^2\big] \;\propto\; T\,\sigma_k^2 \;+\; O(\tau^2),
\end{equation}
matching \Eqref{eq:change_scaling}. Equivalently, the typical absolute displacement scales as $\sqrt{T}\,\sigma_k$, so the un-normalized numerator of the change-ratio \eqref{eq:change_ratio} mirrors the singular-value (square-root explained-variance) profile of $\boldsymbol{\Sigma}$. The argument is identical for representation-level and output-level losses; the only loss-specific quantity is the constant prefactor $\alpha_k$, which does not affect the qualitative scaling. $\square$

\subsection{Proof of \Cref{thm:obs3} (Dominant-Component Recoverability)}
\label{app:proof_thm2}
Let $\mathcal{D}_r$ denote the relearning distribution used by the attacker. Under the standard threat model, $\mathcal{D}_r$ is structurally similar to $\mathcal{D}_f$, so
\begin{equation}
\boldsymbol{\Sigma}_r \;=\; \mathbb{E}_{\mathcal{D}_r}[\mathbf{h}_o\mathbf{h}_o^\top] \;\approx\; \sum_k \tilde{\sigma}_k^2\,\mathbf{v}_k\mathbf{v}_k^\top, \qquad \tilde{\sigma}_k^2 \asymp \sigma_k^2.
\end{equation}
Standard relearning attacks use a maximum-likelihood (cross-entropy) objective on $\mathcal{D}_r$. By the same chain-rule decomposition \eqref{eq:chain_rule}, the relearning gradient is $\nabla_\theta \mathcal{L}_r = \mathbb{E}_{\mathcal{D}_r}[\mathbf{J}(\mathbf{x})^\top \mathbf{g}_r(\mathbf{x})]$ with effective residual $\mathbf{g}_r(\mathbf{x})$. The same case analysis as in \Cref{lem:residual_universal} (specialized to GA-style cross-entropy on $\mathcal{D}_r$) gives $\mathbb{E}[\mathbf{g}_r \mathbf{g}_r^\top] = \mathbf{A}_r \boldsymbol{\Sigma}_r \mathbf{A}_r^\top + O(\tau^2)$, with $\mathbf{A}_r$ approximately diagonal in the PCA basis on the dominant subspace. Iterating the linearized dynamics over $T_r$ relearning steps with constant step size, the projection of $\mathbf{h}_u - \mathbf{h}_r$ onto $\mathbf{v}_k$ follows an exponential approach to the original $\mathbf{h}_o$ projection with rate $c\sigma_k^2$:
\begin{equation}
\mathrm{Recovery\ Ratio}_k \;=\; \frac{\langle \mathbf{h}_u - \mathbf{h}_r, \mathbf{v}_k\rangle}{\langle \mathbf{h}_u - \mathbf{h}_o, \mathbf{v}_k\rangle} \;\approx\; 1 - \exp\!\big(-c\,\sigma_k^2\,T_r\big).
\end{equation}
The result depends on neither the specific unlearning loss used to produce $\mathbf{h}_u$ (since the recovery ratio is normalized by the actual unlearning displacement) nor the specific relearning loss family (any loss whose effective residual covariance shares the dominant eigenstructure of $\boldsymbol{\Sigma}_r$ yields the same scaling). Reaching a fixed recovery level $1-\delta$ on direction $\mathbf{v}_k$ thus requires $T_r=O\!\big(\sigma_k^{-2}\log(1/\delta)\big)$ steps; dominant components saturate to $1$ within a few attack steps, while minor components require an \emph{inverse-variance} blow-up in the number of steps. $\square$

\subsection{Why Minor Components Are Structurally Hard to Recover}
\label{app:snr_argument}
The inverse-variance recovery scaling above explains \emph{rate} differences but not why minor components fail to recover even when the attacker invests a large $T_r$. We complete the picture with a signal-to-noise (SNR) argument on the relearning gradient.

Decompose the per-sample representation as $\mathbf{h}_o(\mathbf{x}) = \sum_k a_k(\mathbf{x})\,\mathbf{v}_k$ with coefficients $a_k(\mathbf{x}) = \langle \mathbf{h}_o(\mathbf{x}), \mathbf{v}_k\rangle$, so $\mathbb{E}[a_k(\mathbf{x})]=0$ and $\mathbb{E}[a_k(\mathbf{x})^2]=\sigma_k^2$ by construction. Because the coefficients are mean-zero, a naive cross-sample correlation $\mathbb{E}_{\mathbf{x},\mathbf{x}'}[a_k(\mathbf{x})a_k(\mathbf{x}')]$ across \emph{independent} samples vanishes identically and conveys no information; instead, agreement must be measured \emph{conditionally} on shared structure. Concretely, let $c$ denote a latent context variable (e.g., the topic, document, or local sub-distribution from which $\mathbf{x}$ is drawn) and decompose
\begin{equation}
a_k(\mathbf{x}) \;=\; s_k(c)\;+\;\epsilon_k(\mathbf{x}), \qquad s_k(c):=\mathbb{E}[a_k(\mathbf{x})\mid c],\quad \mathbb{E}[\epsilon_k\mid c]=0.
\end{equation}
We define the (shared-signal) cross-sample agreement of the $k$-th coordinate as the fraction of variance carried by the context-shared component:
\begin{equation}
\label{eq:rho_def}
\rho_k \;=\; \frac{\mathrm{Var}_c\!\big(s_k(c)\big)}{\sigma_k^2} \;=\; \frac{\mathbb{E}_{\mathbf{x},\mathbf{x}'\mid c}\!\big[a_k(\mathbf{x})\,a_k(\mathbf{x}')\big]}{\sigma_k^2},
\end{equation}
where the second equality holds when $\mathbf{x},\mathbf{x}'$ are drawn \emph{conditional on the same context} $c$. Equivalently, $\rho_k$ is the intra-class/inter-class variance ratio along $\mathbf{v}_k$, and $\sqrt{\rho_k}$ is the cosine alignment between the per-sample gradient $\nabla_\theta\ell_r(\mathbf{x})$ and its batch average projected onto $\mathbf{v}_k$. With this definition, a finite-batch relearning gradient along $\mathbf{v}_k$ has signal $\propto \sigma_k\sqrt{\rho_k}$ and noise $\propto \sigma_k\sqrt{(1-\rho_k)/B}$ for batch size $B$, giving an SNR of order $\sqrt{B\rho_k/(1-\rho_k)}$. Two regimes emerge:
\begin{itemize}[itemsep=0pt,topsep=2pt,parsep=0pt,leftmargin=*]
    \item \textbf{Dominant components}: $\rho_k \to 1$, since these directions encode features shared between $\mathcal{D}_r$ and $\mathcal{D}_f$ at the topic/context level (e.g., topical regularities, syntactic patterns). Relearning gradients along $\mathbf{v}_k$ accumulate coherently across the batch, and recovery proceeds at the rate predicted by \Cref{thm:obs3}.
    \item \textbf{Minor components}: $\rho_k \approx 0$, since these directions encode sample-specific structure that varies idiosyncratically within each context. The batched relearning gradient averages out, the SNR collapses, and no amount of attacker fine-tuning on \emph{related} data can reliably reconstruct the minor-component values that the original model used for the held-out forget samples.
\end{itemize}
This formalizes the intuition stated after Observation~3 and is consistent with both \Cref{fig:recovery_ratio} and the cross-loss replications in \Appref{app:obs_consistency_losses}.

\subsection{Discussion of Assumptions}
\label{app:theory_assumptions}
The two assumptions used above merit comment. \textbf{(i) NTK linearization.} The lazy-regime approximation $\mathbf{K}(\mathbf{x},\mathbf{x}') \approx \kappa\mathbf{I}$ is exact only in the infinite-width limit; in practice it is a useful first-order approximation when the unlearning step size and number of steps remain modest, which is precisely the regime in which fragile unlearning operates (otherwise utility on retain data collapses). The qualitative conclusions---change-ratio scaling with $\sigma_k^2$ and exponential recovery with rate $\sigma_k^2$---survive any anisotropy in $\mathbf{K}$ that is not specifically aligned against the dominant subspace. \textbf{(ii) Effective-residual eigenstructure (\Cref{lem:residual_universal}).} The proof verified this for RMU, MLP Breaking, GA, and NPO. For representation-level losses, the assumption reduces to the target $\mathbf{t}(\mathbf{x})$ being uncorrelated with the dominant subspace, which holds for the random/noise/zero targets used in practice. For output-level losses, the assumption follows because (a) the model's predictions on $\mathcal{D}_f$ depend on $\mathbf{x}$ only through $\mathbf{h}_o(\mathbf{x})$, so any sample-to-sample variation in the output residual is mediated by $\boldsymbol{\Sigma}$, and (b) the post-$\mathbf{h}$ Jacobian $\mathbf{D}\mathbf{W}_{\mathrm{out}}^\top$ is full-rank on the dominant subspace at any well-trained checkpoint. The empirical universality of Observation~2 across unlearning losses (\Appref{app:obs_consistency_losses}) confirms that the assumption holds in practice for both representation-level and output-level losses.

\section{Spectral Recovery, MCU, and the Connection to CIR}
\label{app:discussion_cir}

\paragraph{Spectral recovery as a mechanism of unlearning fragility.}
CIR~\citep{sondej2025collapse} motivates principal-component gradient filtering through the intuition that unlearning disrupts shared representations that subsequent fine-tuning can restore. However, it does not provide a empirical or theoretical analysis of direction-wise representation changes and their recovery dynamics. Our work studies these dynamics directly: unlearning concentrates changes in dominant representation components, and relearning preferentially recovers those changes (\Cref{sec:understanding}). As shown in \Cref{thm:obs2,thm:obs3}, our analysis connects this asymmetry to the representation spectrum. This identifies the allocation of unlearning changes across directions with different recovery rates as a mechanism governing robustness.

\paragraph{MCU as a theory-guided intervention.}
This mechanism motivates MCU: a representation-level loss that suppresses dominant components, directing optimization pressure toward minor components. MCU makes the spectral allocation of forgetting part of the objective itself. Our representation analysis and subspace controls (\Cref{fig:mcu_shift,tab:rank_matched}) support this design by showing a redistribution of changes toward minor components and lower post-attack accuracy with the proposed projection. The contribution therefore comprises both an account of spectral recovery and a method designed from that account.

\paragraph{A curvature-based interpretation of CIR.}
Although \citet{sondej2025collapse} motivates CIR with the intuition that unlearning disrupts shared representations, the method does not follow explicitly or necessarily from this intuition. To explain this discrepancy, we provide a curvature-based interpretation that explains how CIR works and why it is effective.
CIR~\citep{sondej2025collapse} can be interpreted as redirecting updates toward directions with lower relearning curvature, where unlearning changes decay more slowly during relearning. To make this connection explicit, consider a linear module $z=Wa$ with input activation $a$ and output gradient $g_u=\partial\ell_u/\partial z$. CIR filters both vectors before forming the weight update. Modeling these filters as orthogonal projections $P_a$ and $P_g$ onto the retained subspaces gives
\begin{equation}
\label{eq:cir_projected_update}
\Delta W_{\mathrm{CIR}}
=-\eta_u(P_g g_u)(P_a a)^\top
=-\eta_u P_g(g_u a^\top)P_a.
\end{equation}

To identify the curvature of the retained modes, we approximate the local relearning Hessian block by its Fisher/GGN counterpart and then apply K-FAC~\citep{martens2015kfac}:
\begin{equation}
\label{eq:cir_kronecker_discussion}
H_{r,W}\approx F_{r,W}\approx\widetilde H_{r,W}:=A_r\otimes G_r,
\qquad A_r=\mathbb{E}_r[aa^\top].
\end{equation}
Here $G_r=\mathbb{E}_r[g_r g_r^\top]$ is the module-output Fisher factor, where $g_r$ is the module-output gradient of the negative log-likelihood. If $A_r u_i=\alpha_i u_i$ and $G_r w_j=\beta_j w_j$, with eigenvalues ordered in descending order, the parameter mode $w_j u_i^\top$ has approximate curvature $\lambda_{ij}=\alpha_i\beta_j$.

Suppose the activation and gradient subspaces removed by CIR align with the leading $k_a$ and $k_g$ eigendirections of these relearning factors. For the RTT attack, in which unlearning uses $T\cup V$ and relearning uses $T$, suppose the activation and gradient subspaces removed by CIR align with the leading $k_a$ and $k_g$ eigendirections of the relearning factors $A_r$ and $G_r$ on $T$, respectively. Then \Eqref{eq:cir_projected_update} retains only modes with $i>k_a$ and $j>k_g$, whose curvatures satisfy
\begin{equation}
\label{eq:cir_retained_curvature}
\lambda_{ij}\le\alpha_{k_a+1}\beta_{k_g+1},
\qquad i>k_a,\ j>k_g.
\end{equation}
This bounds the curvature of the retained update subspace.

Taking the pre-unlearning parameters as a stationary recovery solution, let $e_{ij}(t)$ be the remaining unlearning displacement along mode $w_j u_i^\top$. In the fixed, layerwise quadratic model, ordinary gradient descent with sufficiently small step size gives
\begin{equation}
\label{eq:cir_local_recovery_discussion}
e_{ij}(T_r)=(1-\eta_r\lambda_{ij})^{T_r}e_{ij}(0)
\approx e^{-\eta_r\lambda_{ij}T_r}e_{ij}(0).
\end{equation}
Consequently, the displacement in CIR's retained low-curvature modes loses only a small fraction of its initial magnitude when $\eta_r\alpha_{k_a+1}\beta_{k_g+1}T_r\ll1$. Over an interval with fixed aligned filters, the accumulated unlearning displacement remains in this subspace. CIR therefore changes where forgetting occurs, placing it in modes that relearning reverses slowly under a limited budget.

\paragraph{Additional benefits beyond CIR.}
CIR filters parameter updates, while MCU specifies the representation directions penalized by the loss. \Cref{tab:results} shows that adding MCU lowers post-attack Relearn accuracy for both RMU and MLP Breaking across all three datasets, with and without CIR. These comparisons demonstrate standalone gains and additional benefits on top of CIR. To control for differences in initial forgetting, \Cref{tab:matched_forget} further compares checkpoints at matched Forget accuracy and comparable WikiText loss: adding MCU to MLP Breaking + CIR reduces Relearn accuracy from $38.9\%$ to $32.6\%$ on WMDP-Cyber and from $29.6\%$ to $25.4\%$ on WMDP-Bio. The additional robustness benefits thus persist at comparable initial forgetting levels, complementing the mechanistic contribution of our spectral analysis.

\paragraph{Summary.}
\mname and CIR act at different levels: \mname redirects unlearning changes toward minor representation components, while CIR filters parameter updates and, under the alignment assumptions above, favors low-curvature directions.

\section{Representation-Analysis Setup and Details}
\label{app:rep_analysis_appendix}
This appendix expands the representation-analysis protocol summarized in \Cref{subsec:empirical}.

\paragraph{Modules and layers.}
We instrument the MLP \texttt{down\_proj} output of every transformer block of Llama-3.1-8B (32 layers). For each layer we record activations at each token position of every example in the forget set $\Df$, yielding a tensor $\mathbf{H}^{(\ell)} \in \mathbb{R}^{N \times d}$ per layer $\ell$, where $N$ is the total number of tokens in $\Df$ and $d=4096$ is the hidden dimension. Results of other MLP sub-modules, shown in \Appref{app:obs_consistency_losses}, exhibit the same qualitative pattern.

\paragraph{PCA computation.}
Per layer, we center $\mathbf{H}^{(\ell)}$ by subtracting the per-coordinate mean and compute the principal components $\{\mathbf{v}_1^{(\ell)}, \ldots, \mathbf{v}_d^{(\ell)}\}$ via randomized SVD \citep{halko2011finding}, with explained variances $\sigma_k^{(\ell)\,2}$. The same eigenbasis is used to project the unlearned and relearned activations $\mathbf{h}_u, \mathbf{h}_r$ collected on the same token positions. All ratios reported in \Cref{fig:unlearn_change_ratio,fig:recovery_ratio} are computed per layer and then averaged across layers.

\section{Experimental Details}
\label{app:dataset_details}

\paragraph{Dataset details.}
For \textbf{WMDP-Cyber} and \textbf{WMDP-Bio}, we use the high-quality subsets of \citet{sondej2025collapse} containing 203 cyber and 144 biological multiple-choice questions, each augmented with three short declarative sentences per question that together form the forget set used for unlearning. The \textbf{Years} dataset \citep{deeb2024unlearning} consists of 20th-century events paired with their dates. As retain sets we use FineFineWeb \citep{finefineweb2024} subsets matched to the forget domain: \texttt{biology} for WMDP-Bio, \texttt{computer\_science\_and\_technology} for WMDP-Cyber, and \texttt{fineweb-edu} for Years.

\paragraph{Baselines.}
Each evaluated forget loss paired with an specific retain loss to preserve model utility. For \textbf{Gradient Ascent (GA)} and \textbf{NPO} \citep{zhang2024negative}, we apply the standard cross-entropy loss on the retain set. For \textbf{RMU} \citep{li2024wmdp} and \textbf{MLP Breaking}, following \citet{li2024wmdp}, we use a loss that penalizes the norm difference between the current and original model's representations on retain set, which encourages minimal deviation from the original model's representations.

\paragraph{Accuracy Computation.}
Following \citet{sondej2025collapse}, we compute accuracy as the expected probability of selecting the correct answer. Specifically, for multiple-choice questions with $k$ options, we compute the probability distribution over answer choices using softmax with temperature $\tau = 1$ on the logits corresponding to the answer tokens. The accuracy for a batch is then computed as:
\begin{equation}
    \text{Accuracy} = \frac{1}{|B|} \sum_{i \in B} p_i^{(\text{correct})},
\end{equation}
where $p_i^{(\text{correct})}$ denotes the probability assigned to the correct answer for sample $i$, and $B$ is the batch. This expected accuracy metric provides a more fine-grained measure than hard accuracy (which only counts exact matches) and is more sensitive to partial knowledge changes during unlearning and relearning.

\paragraph{Relearning Attack Protocol.}
We follow the Retraining-on-$T$ (RTT) attack protocol proposed by \citet{deeb2024unlearning}. After unlearning on the full forget set $T \cup V$, we fine-tune the unlearned model on the training partition $T$ (80\% of the forget set) and evaluate accuracy recovery on the held-out validation partition $V$ (20\% of the forget set). For WMDP-Cyber and WMDP-Bio, we perform relearning for 100 epochs, while for the Years dataset we use 30 epochs due to its smaller size. To obtain robust estimates of post-attack accuracy, we follow \citet{sondej2025collapse} and smooth the relearning accuracy curve by averaging over windows of 10 epochs for WMDP datasets and 3 epochs for Years. The reported \textbf{Relearn} accuracy corresponds to the maximum smoothed accuracy across the relearning trajectory, as some attack runs may exceed the optimal number of epochs.

\paragraph{Hyperparameter Selection for MCU.}
The key hyperparameter in our MCU method is $K$, the projection rank before computing the unlearning loss (\Cref{eq:projection}): the mean direction plus the top $K-1$ centered principal components. We perform a grid search over $K \in \{1, 2, 4, 8, 16, 32, 64\}$ and select the value that achieves the best trade-off between forget quality (low relearn accuracy). The main results use $K=32$ on WMDP-Cyber and Years and $K=8$ on WMDP-Bio; \Cref{tab:k_sweep} reports the sweep on WMDP-Cyber and WMDP-Bio.

\paragraph{Unlearning Termination Criterion.}
Following \citet{sondej2025collapse}, we use the WikiText loss \citep{merity2016pointer} as a criterion to determine when to terminate unlearning, in order to control for disruption to general language modeling performance. Specifically, we monitor the WikiText loss relative to its initial value before unlearning. Since different unlearning methods affect the WikiText loss differently, we use method-specific termination thresholds. The WikiText loss threshold primarily controls the number of training steps; we select thresholds for each method such that the unlearn accuracy approaches random chance (approximately 25\% for 4-way multiple choice) while maintaining reasonable MMLU performance.

\section{Additional Results}
\label{app:additional_results}

Table~\ref{tab:full_results} presents all experimental results across three datasets: WMDP-Cyber, WMDP-Bio, and Years.

\begin{table*}[t]
\centering
\caption{Complete experimental results on WMDP-Cyber, WMDP-Bio, and Years datasets. Highlighted rows are MCU variants; the grey rows are original-model baselines.}
\scriptsize
\setlength{\tabcolsep}{2.4pt}
\renewcommand{\arraystretch}{0.82}
\resizebox{\textwidth}{!}{%
\begin{tabular}{c | l | c c | c c c}
\toprule
\textbf{Dataset} & \textbf{Method} & \textbf{MMLU ($\uparrow$)} & \textbf{WikiText ($\downarrow$)} & \textbf{Forget ($\downarrow$)} & \textbf{Relearn ($\downarrow$)} & $\Delta$ ($\downarrow$) \\
\midrule
\multirow{13}{*}{\rotatebox[origin=c]{90}{WMDP-Cyber}}
& \cellcolor{basebg}Original model & \cellcolor{basebg}65.1 & \cellcolor{basebg}1.000 & \cellcolor{basebg}57.6 & \cellcolor{basebg}- & \cellcolor{basebg}- \\
\cmidrule{2-7}
& GA & 61.3 & 1.503 & 25.1 & 57.0 & 31.9 \\
& GA + SAM & 60.1 & 1.100 & 27.4 & 57.7 & 30.3 \\
& NPO & 60.6 & 1.578 & 23.9 & 57.1 & 33.2 \\
& NPO + SAM & 60.1 & 1.101 & 27.7 & 57.7 & 30.0 \\
\cmidrule{2-7}
& RMU & 52.8 & 1.207 & 28.7 & 54.6 & 26.0 \\
& RMU + SAM & 52.7 & 1.201 & 28.3 & 53.5 & 25.1 \\
& \cellcolor{methodbg}RMU + \mname & \cellcolor{methodbg}49.3 & \cellcolor{methodbg}1.202 & \cellcolor{methodbg}28.8 & \cellcolor{methodbg}53.9 & \cellcolor{methodbg}25.1 \\
& RMU + CIR & 64.1 & 1.006 & 28.0 & 50.5 & 22.5 \\
& RMU + CIR + SAM & 64.4 & 1.006 & 24.7 & 49.2 & 24.4 \\
& \cellcolor{methodbg}RMU + CIR + \mname & \cellcolor{methodbg}64.7 & \cellcolor{methodbg}1.006 & \cellcolor{methodbg}31.7 & \cellcolor{methodbg}43.3 & \cellcolor{methodbg}11.6 \\
\cmidrule{2-7}
& MLP Breaking & 61.8 & 1.132 & 27.3 & 55.8 & 28.4 \\
& MLP Breaking + SAM & 58.0 & 1.215 & 25.6 & 47.4 & 21.8 \\
& \cellcolor{methodbg}MLP Breaking + \mname & \cellcolor{methodbg}60.4 & \cellcolor{methodbg}1.105 & \cellcolor{methodbg}26.0 & \cellcolor{methodbg}52.0 & \cellcolor{methodbg}26.0 \\
& MLP Breaking + CIR & 65.0 & 1.001 & 20.2 & 37.2 & 17.0 \\
& MLP Breaking + CIR + SAM & 65.3 & 1.005 & 26.2 & 28.5 & 2.3 \\
& \cellcolor{methodbg}MLP Breaking + CIR + \mname & \cellcolor{methodbg}64.7 & \cellcolor{methodbg}1.001 & \cellcolor{methodbg}25.9 & \cellcolor{methodbg}31.6 & \cellcolor{methodbg}5.7 \\
\midrule
\multirow{13}{*}{\rotatebox[origin=c]{90}{WMDP-Bio}}
& \cellcolor{basebg}Original model & \cellcolor{basebg}65.1 & \cellcolor{basebg}1.000 & \cellcolor{basebg}64.0 & \cellcolor{basebg}- & \cellcolor{basebg}- \\
\cmidrule{2-7}
& GA & 56.1 & 1.528 & 29.0 & 69.0 & 40.0 \\
& GA + SAM & 53.5 & 1.305 & 29.3 & 62.4 & 33.1 \\
& NPO & 51.9 & 1.419 & 25.9 & 66.7 & 40.8 \\
& NPO + SAM & 53.4 & 1.307 & 29.2 & 61.9 & 32.7 \\
\cmidrule{2-7}
& RMU & 45.5 & 1.300 & 26.7 & 49.8 & 23.1 \\
& RMU + SAM & 51.9 & 1.183 & 27.2 & 49.0 & 21.8 \\
& \cellcolor{methodbg}RMU + \mname & \cellcolor{methodbg}46.5 & \cellcolor{methodbg}1.205 & \cellcolor{methodbg}28.0 & \cellcolor{methodbg}48.7 & \cellcolor{methodbg}20.7 \\
& RMU + CIR & 63.7 & 1.010 & 27.9 & 47.6 & 19.7 \\
& RMU + CIR + SAM & 57.2 & 1.006 & 24.2 & 33.0 & 8.8 \\
& \cellcolor{methodbg}RMU + CIR + \mname & \cellcolor{methodbg}63.5 & \cellcolor{methodbg}1.010 & \cellcolor{methodbg}32.2 & \cellcolor{methodbg}39.5 & \cellcolor{methodbg}7.3 \\
\cmidrule{2-7}
& MLP Breaking & 58.7 & 1.364 & 29.9 & 58.8 & 28.9 \\
& MLP Breaking + SAM & 57.5 & 1.212 & 21.8 & 48.5 & 26.7 \\
& \cellcolor{methodbg}MLP Breaking + \mname & \cellcolor{methodbg}60.3 & \cellcolor{methodbg}1.357 & \cellcolor{methodbg}29.1 & \cellcolor{methodbg}57.5 & \cellcolor{methodbg}28.5 \\
& MLP Breaking + CIR & 64.8 & 1.001 & 22.1 & 29.7 & 7.6 \\
& MLP Breaking + CIR + SAM & 64.9 & 1.002 & 25.9 & 31.4 & 5.5 \\
& \cellcolor{methodbg}MLP Breaking + CIR + \mname & \cellcolor{methodbg}64.5 & \cellcolor{methodbg}1.001 & \cellcolor{methodbg}22.8 & \cellcolor{methodbg}26.4 & \cellcolor{methodbg}3.6 \\
\midrule
\multirow{13}{*}{\rotatebox[origin=c]{90}{Years}}
& \cellcolor{basebg}Original model & \cellcolor{basebg}65.1 & \cellcolor{basebg}1.000 & \cellcolor{basebg}68.4 & \cellcolor{basebg}- & \cellcolor{basebg}- \\
\cmidrule{2-7}
& GA & 63.9 & 1.529 & 46.1 & 64.0 & 18.0 \\
& GA + SAM & 56.3 & 1.360 & 25.8 & 63.7 & 37.9 \\
& NPO & 58.5 & 1.404 & 27.9 & 63.4 & 35.6 \\
& NPO + SAM & 56.0 & 1.363 & 25.8 & 62.9 & 37.1 \\
\cmidrule{2-7}
& RMU & 56.8 & 1.204 & 33.0 & 64.3 & 31.4 \\
& RMU + SAM & 54.9 & 1.201 & 34.0 & 65.6 & 31.6 \\
& \cellcolor{methodbg}RMU + \mname & \cellcolor{methodbg}52.0 & \cellcolor{methodbg}1.193 & \cellcolor{methodbg}30.8 & \cellcolor{methodbg}54.9 & \cellcolor{methodbg}24.1 \\
& RMU + CIR & 57.1 & 1.102 & 30.7 & 37.4 & 6.7 \\
& RMU + CIR + SAM & 57.1 & 1.111 & 29.7 & 36.6 & 6.9 \\
& \cellcolor{methodbg}RMU + CIR + \mname & \cellcolor{methodbg}57.2 & \cellcolor{methodbg}1.101 & \cellcolor{methodbg}31.7 & \cellcolor{methodbg}33.7 & \cellcolor{methodbg}2.0 \\
\cmidrule{2-7}
& MLP Breaking & 60.8 & 1.239 & 27.0 & 51.5 & 24.5 \\
& MLP Breaking + SAM & 58.0 & 1.215 & 25.6 & 47.4 & 21.8 \\
& \cellcolor{methodbg}MLP Breaking + \mname & \cellcolor{methodbg}61.5 & \cellcolor{methodbg}1.214 & \cellcolor{methodbg}29.0 & \cellcolor{methodbg}49.5 & \cellcolor{methodbg}20.6 \\
& MLP Breaking + CIR & 64.6 & 1.010 & 32.7 & 39.4 & 6.7 \\
& MLP Breaking + CIR + SAM & 64.3 & 1.021 & 26.6 & 36.9 & 10.3 \\
& \cellcolor{methodbg}MLP Breaking + CIR + \mname & \cellcolor{methodbg}63.8 & \cellcolor{methodbg}1.010 & \cellcolor{methodbg}25.9 & \cellcolor{methodbg}30.9 & \cellcolor{methodbg}5.0 \\
\bottomrule
\end{tabular}
}
\label{tab:full_results}
\end{table*}

\section{Comparisons at Matched Forget Accuracy}
\label{app:matched_forget}

In \Cref{tab:results}, methods reach different Forget accuracies. Because $\Delta = \text{Relearn}-\text{Forget}$ is hard to interpret across different starting accuracies, those relearning gaps leave the comparison of recovered target knowledge unresolved. Held-out Relearn accuracy measures that recovered knowledge.

We trace both unlearning trajectories up to the same WikiText-loss budget, save every epoch, and independently apply RTT attack to every checkpoint. \Cref{tab:matched_forget} then compares checkpoints with nearly identical Forget accuracy.

\begin{table}[t]
\centering
\caption{Matched operating points for MLP Breaking + CIR and MLP Breaking + CIR + \mname. \textbf{Bold} marks the lower Relearn and $\Delta$.}
\label{tab:matched_forget}
\footnotesize
\setlength{\tabcolsep}{3pt}
\resizebox{\linewidth}{!}{%
\begin{tabular}{llccccc}
\toprule
\textbf{Dataset} & \textbf{Method} & \textbf{Steps} & \textbf{Forget $\downarrow$} & \textbf{Relearn $\downarrow$} & $\Delta$ $\downarrow$ & \textbf{WikiText} \\
\midrule
\multirow{2}{*}{WMDP-Cyber}
& MLP Breaking + CIR & 96 & 23.6 & 38.9 & 15.3 & 1.0001 \\
& \cellcolor{methodbg}MLP Breaking + CIR + \mname & \cellcolor{methodbg}208 & \cellcolor{methodbg}23.9 & \cellcolor{methodbg}\textbf{32.6} & \cellcolor{methodbg}\textbf{8.7} & \cellcolor{methodbg}0.9991 \\
\midrule
\multirow{2}{*}{WMDP-Bio}
& MLP Breaking + CIR & 154 & 23.2 & 29.6 & 6.5 & 1.0007 \\
& \cellcolor{methodbg}MLP Breaking + CIR + \mname & \cellcolor{methodbg}231 & \cellcolor{methodbg}22.8 & \cellcolor{methodbg}\textbf{25.4} & \cellcolor{methodbg}\textbf{2.5} & \cellcolor{methodbg}1.0005 \\
\bottomrule
\end{tabular}}
\end{table}

At the matched operating points, the two methods have essentially identical Forget accuracy and comparable WikiText loss. Adding \mname yields consistently lower post-attack Relearn accuracy and a smaller relearning gap $\Delta$ on both datasets. This shows that the robustness gain is preserved after controlling for initial Forget accuracy.

\section{Robustness under Adaptive Representation-Based Attacks}
\label{app:adaptive_attack}

Because MCU operates on internal representations, a natural concern is whether an adversary that directly targets representations, rather than the standard RTT loss, can defeat it.
We assume the attacker has access to the unlearned model, the original pre-unlearning model, and the forget set $T$, and fine-tunes the unlearned model to minimize the MSE between its MLP activations and those of the original model on $T$. This is a strictly stronger threat model than RTT, as it directly targets the representation-level modifications MCU introduces. Unlearning follows our main setup, and we report post-attack accuracy on the held-out split $V$.

\begin{table}[t]
\centering
\caption{Robustness against the adaptive representation-based attack on Llama-3.1-8B. The attacker directly aligns the unlearned model's activations with the original model's. \textbf{Bold} indicates the lowest $\Delta$ (best robustness).}
\label{tab:adaptive_attack}
\footnotesize
\setlength{\tabcolsep}{4pt}
\begin{tabular}{@{}lcccccc@{}}
\toprule
& \multicolumn{3}{c}{\textbf{WMDP-Cyber}} & \multicolumn{3}{c}{\textbf{WMDP-Bio}} \\
\cmidrule(lr){2-4} \cmidrule(lr){5-7}
\textbf{Method} & Forget & Relearn & $\Delta$ & Forget & Relearn & $\Delta$ \\
\midrule
Original model            & 57.6 & --   & --   & 64.0 & --   & --   \\
GA                        & 25.2 & 58.6 & 33.4 & 35.1 & 64.7 & 29.7 \\
MLP Breaking + CIR        & 19.1 & 35.8 & 16.6 & 20.5 & 31.8 & 11.3 \\
\rowcolor{methodbg}
MLP Breaking + CIR + \mname & 25.8 & 28.7 & \textbf{2.9} & 23.9 & 28.0 & \textbf{4.1} \\
\bottomrule
\end{tabular}
\vspace{-3mm}
\end{table}

\Cref{tab:adaptive_attack} shows a clear hierarchy of robustness. GA collapses entirely under this attack, with relearn accuracy nearly returning to the original model, confirming that output-level unlearning leaves internal representations essentially intact and trivially recoverable. MLP Breaking + CIR is markedly more robust but still loses a non-trivial amount of forgotten knowledge, indicating that recoverable traces persist in the dominant components even after gradient-level filtering. Adding \mname yields by far the strongest robustness, with $\Delta$ several times smaller than MLP Breaking + CIR on both datasets. This trend matches our representation-geometry analysis (\Cref{sec:understanding}): the adaptive attacker can re-establish the cross-sample dominant-component structure, but cannot reliably reconstruct modifications in the minor-component subspace, which lacks the regularity needed for reconstruction.

\section{Cross-Model Generality}
\label{app:cross_model}

To assess whether the benefits of MCU transfer beyond Llama-3.1-8B, we evaluate it on two additional model families, \textbf{Gemma2-9B} and \textbf{Qwen3-8B}, across all three forget datasets used in the paper. We use the same training and evaluation pipeline as in \Appref{app:additional_results}, and compare the strongest representation-based baseline (MLP Breaking + CIR) against the corresponding MCU variant (MLP Breaking + CIR + \mname).

\Cref{tab:cross_model} reports the results. Across both Gemma2-9B and Qwen3-8B, adding MCU consistently lowers the post-attack relearning gap $\Delta$ over the MLP Breaking + CIR baseline while keeping MMLU essentially unchanged. The improvement is substantial on WMDP-Cyber for both models (Gemma2-9B: $4.0 \rightarrow 1.9$; Qwen3-8B: $13.7 \rightarrow 7.4$) and on Years (Gemma2-9B: $6.3 \rightarrow 2.7$; Qwen3-8B: $1.1 \rightarrow 0.8$); on WMDP-Bio with Qwen3-8B, MCU even drives $\Delta$ slightly negative, indicating that the relearning attack fails to recover any forgotten knowledge above the post-unlearning level. These results corroborate our main finding that explicitly redirecting forgetting into the minor-component subspace yields more robust unlearning, and that this benefit is not specific to a single base model.

\begin{table}[t]
\centering
\caption{Cross-model evaluation on Gemma2-9B and Qwen3-8B across all three datasets. MCU is applied on top of the best MLP Breaking + CIR setting. Lower $\Delta$ is better.}
\setlength{\tabcolsep}{4pt}
\resizebox{\linewidth}{!}{%
\begin{tabular}{@{}lllcccc@{}}
\toprule
\textbf{Model} & \textbf{Dataset} & \textbf{Method} & \textbf{MMLU} & \textbf{Forget} & \textbf{Relearn} & $\Delta$ \\
\midrule
\multirow{6}{*}{Gemma2-9B}
 & \multirow{2}{*}{WMDP-Cyber} & MLP Breaking + CIR & 69.8 & 26.3 & 30.3 & 4.0 \\
 &                             & \cellcolor{methodbg}MLP Breaking + CIR + \mname & \cellcolor{methodbg}69.9 & \cellcolor{methodbg}22.4 & \cellcolor{methodbg}24.3 & \cellcolor{methodbg}\textbf{1.9} \\
\cmidrule{2-7}
 & \multirow{2}{*}{WMDP-Bio}   & MLP Breaking + CIR & 69.7 & 27.8 & 33.3 & 5.5 \\
 &                             & \cellcolor{methodbg}MLP Breaking + CIR + \mname & \cellcolor{methodbg}69.4 & \cellcolor{methodbg}23.0 & \cellcolor{methodbg}28.1 & \cellcolor{methodbg}\textbf{5.1} \\
\cmidrule{2-7}
 & \multirow{2}{*}{Years}      & MLP Breaking + CIR & 69.6 & 30.0 & 36.3 & 6.3 \\
 &                             & \cellcolor{methodbg}MLP Breaking + CIR + \mname & \cellcolor{methodbg}68.9 & \cellcolor{methodbg}26.2 & \cellcolor{methodbg}28.9 & \cellcolor{methodbg}\textbf{2.7} \\
\midrule
\multirow{6}{*}{Qwen3-8B}
 & \multirow{2}{*}{WMDP-Cyber} & MLP Breaking + CIR & 75.7 & 28.0 & 41.7 & 13.7 \\
 &                             & \cellcolor{methodbg}MLP Breaking + CIR + \mname & \cellcolor{methodbg}75.8 & \cellcolor{methodbg}29.0 & \cellcolor{methodbg}36.4 & \cellcolor{methodbg}\textbf{7.4} \\
\cmidrule{2-7}
 & \multirow{2}{*}{WMDP-Bio}   & MLP Breaking + CIR & 73.4 & 26.6 & 35.5 & 8.9 \\
 &                             & \cellcolor{methodbg}MLP Breaking + CIR + \mname & \cellcolor{methodbg}75.1 & \cellcolor{methodbg}29.5 & \cellcolor{methodbg}26.1 & \cellcolor{methodbg}\textbf{-3.4} \\
\cmidrule{2-7}
 & \multirow{2}{*}{Years}      & MLP Breaking + CIR & 75.4 & 39.4 & 40.6 & 1.1 \\
 &                             & \cellcolor{methodbg}MLP Breaking + CIR + \mname & \cellcolor{methodbg}75.3 & \cellcolor{methodbg}30.1 & \cellcolor{methodbg}30.9 & \cellcolor{methodbg}\textbf{0.8} \\
\bottomrule
\end{tabular}
}
\label{tab:cross_model}
\end{table}

\section{Consistency of Observations~2--3 Across Unlearning Losses}
\label{app:obs_consistency_losses}

In \Cref{sec:understanding}, the change-ratio (\Eqref{eq:change_ratio}) and recovery-ratio plots in \Cref{fig:pca_analysis} are reported for GA. Here we show that the same qualitative pattern -- unlearning concentrates changes in dominant components, and relearning preferentially recovers them -- holds across the full set of unlearning losses considered in our experiments: NPO, RMU, MLP Breaking, and their CIR-augmented variants (RMU + CIR, MLP Breaking + CIR). All experiments use full fine-tuning on Llama-3.1-8B with the WMDP-Cyber forget set, following the setup in \Cref{sec:understanding}.

\Cref{fig:obs_consistency_losses} reports, for each method, (left) the unlearn change ratio across principal-component indices and (right) the corresponding recovery ratio after the RTT relearning attack. Across all five methods, the change-ratio mass is concentrated in the first few PCs, and the recovery ratio is high precisely for these dominant components and decays toward the minor components. This confirms that Observations~2 and 3 are not specific to GA, but reflect a property shared by representation-level (RMU, MLP Breaking) and output-level (NPO) unlearning losses, with and without CIR-style gradient filtering. Equivalently, the dominant-component vulnerability that motivates MCU is a generic property of current LLM unlearning pipelines rather than an artifact of any particular loss.

\begin{figure}[t]
\centering
\setlength{\tabcolsep}{2pt}
\renewcommand{\arraystretch}{0.68}
\begin{tabular}{@{}c@{\hspace{4pt}}c@{}}
\textbf{\footnotesize Unlearn Change Ratio} & \textbf{\footnotesize Recovery Ratio} \\[2pt]
\multicolumn{2}{c}{\footnotesize\textit{MLP Breaking + CIR}} \\
\includegraphics[width=0.44\linewidth]{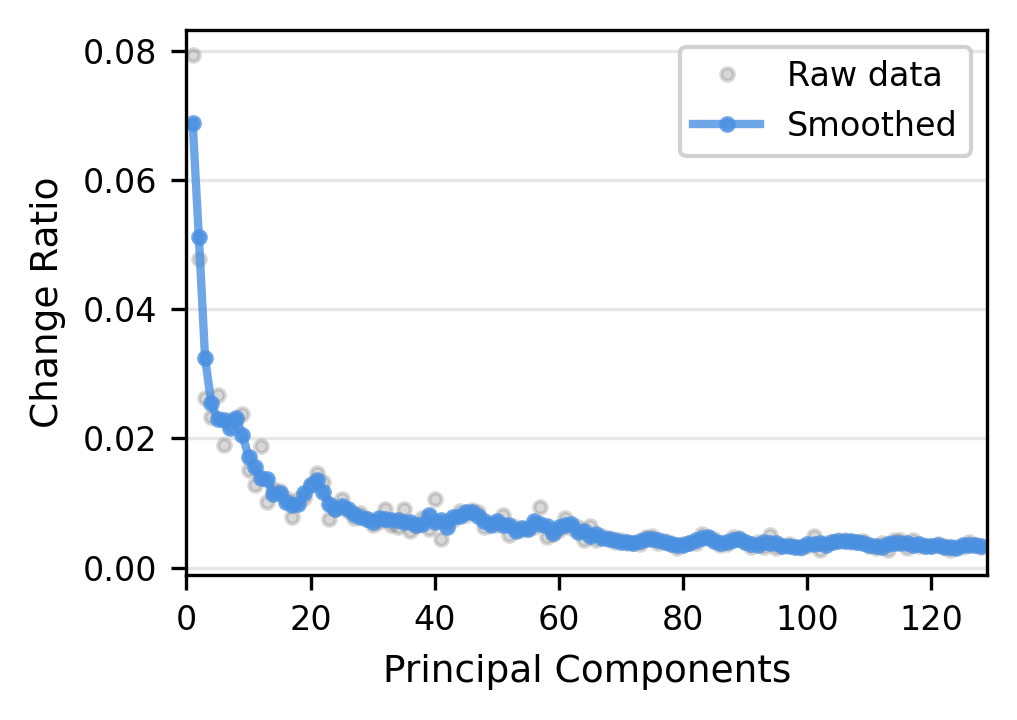} &
\includegraphics[width=0.44\linewidth]{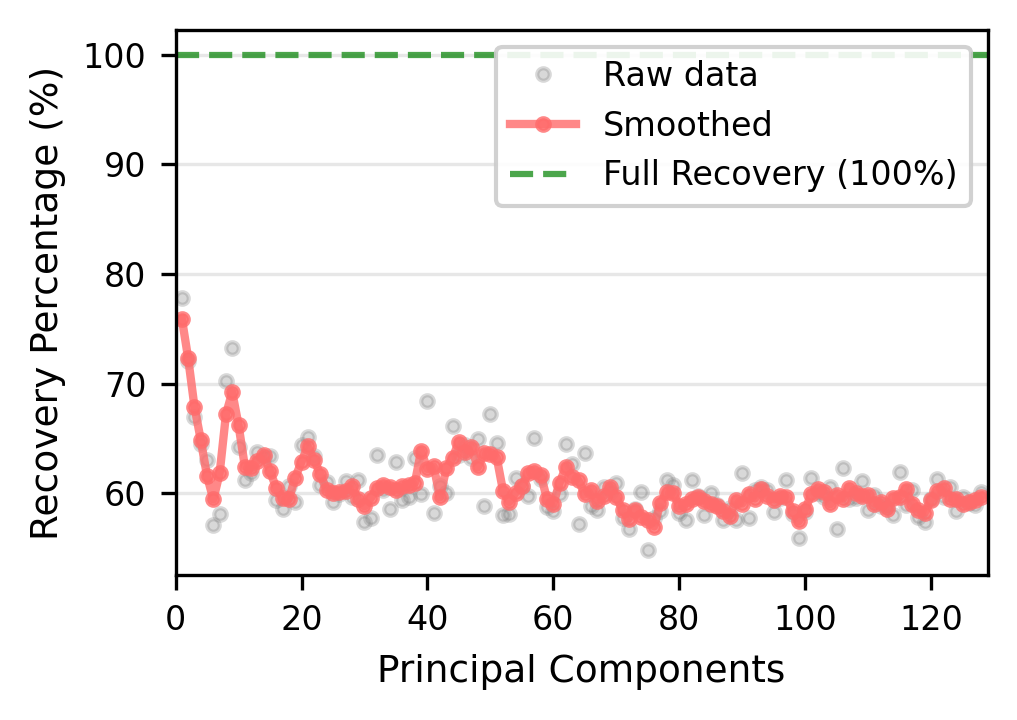} \\[1pt]
\multicolumn{2}{c}{\footnotesize\textit{RMU + CIR}} \\
\includegraphics[width=0.44\linewidth]{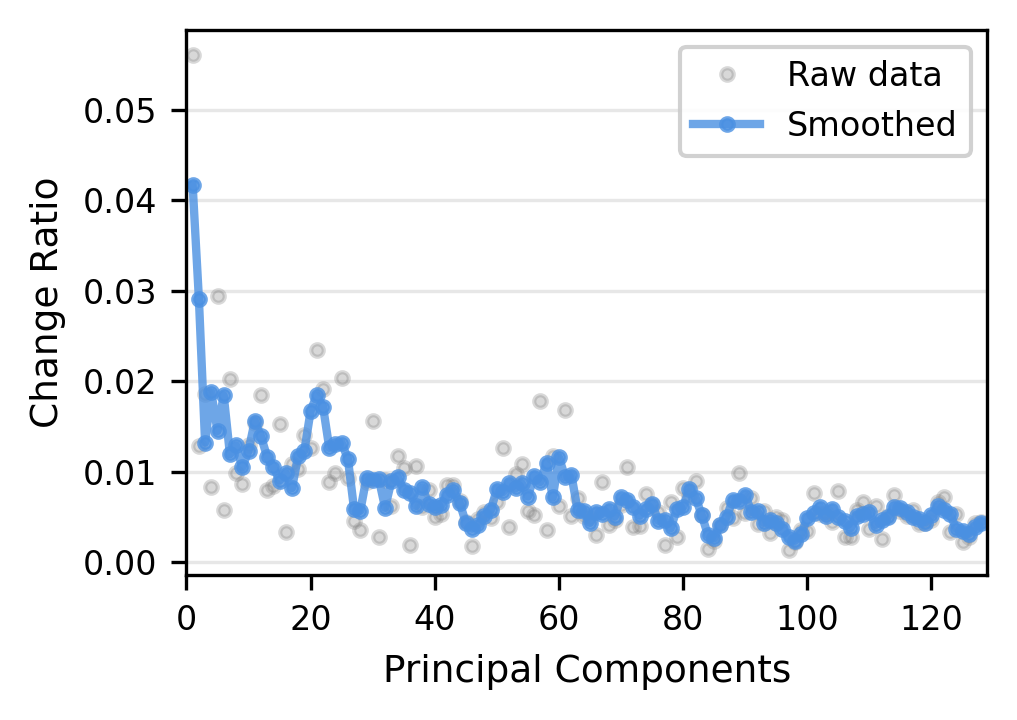} &
\includegraphics[width=0.44\linewidth]{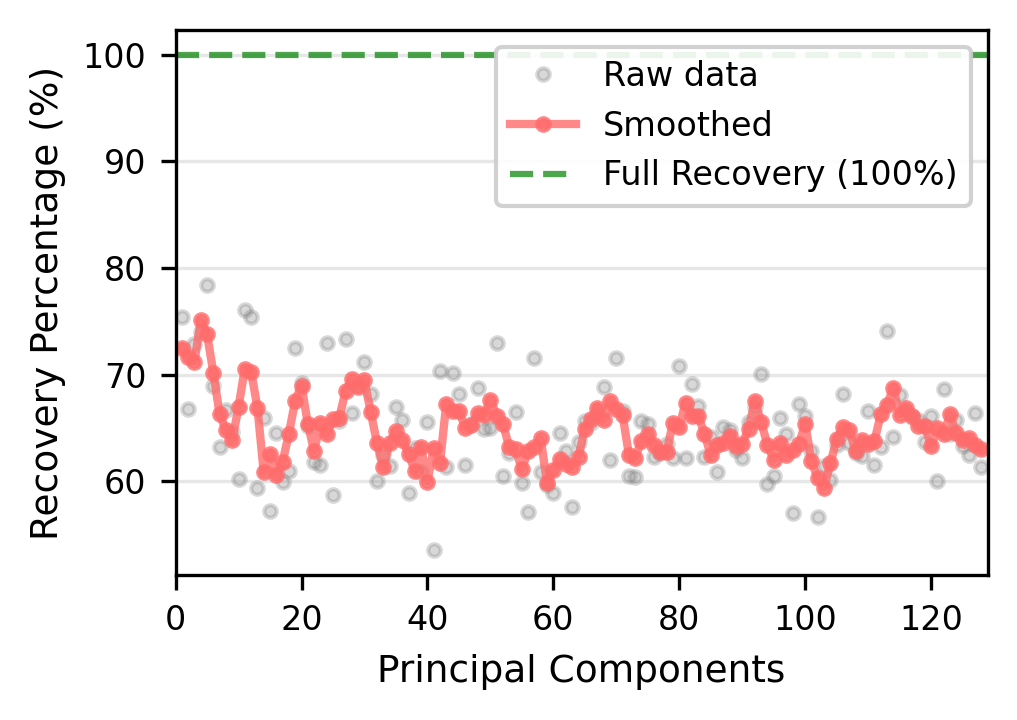} \\[1pt]
\multicolumn{2}{c}{\footnotesize\textit{NPO}} \\
\includegraphics[width=0.44\linewidth]{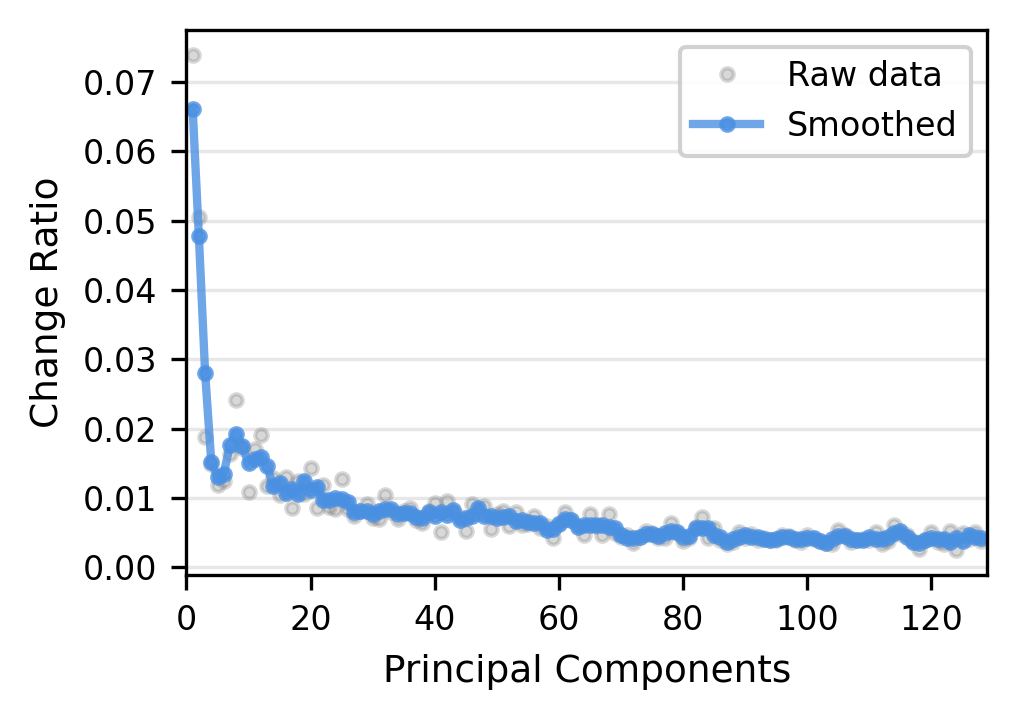} &
\includegraphics[width=0.44\linewidth]{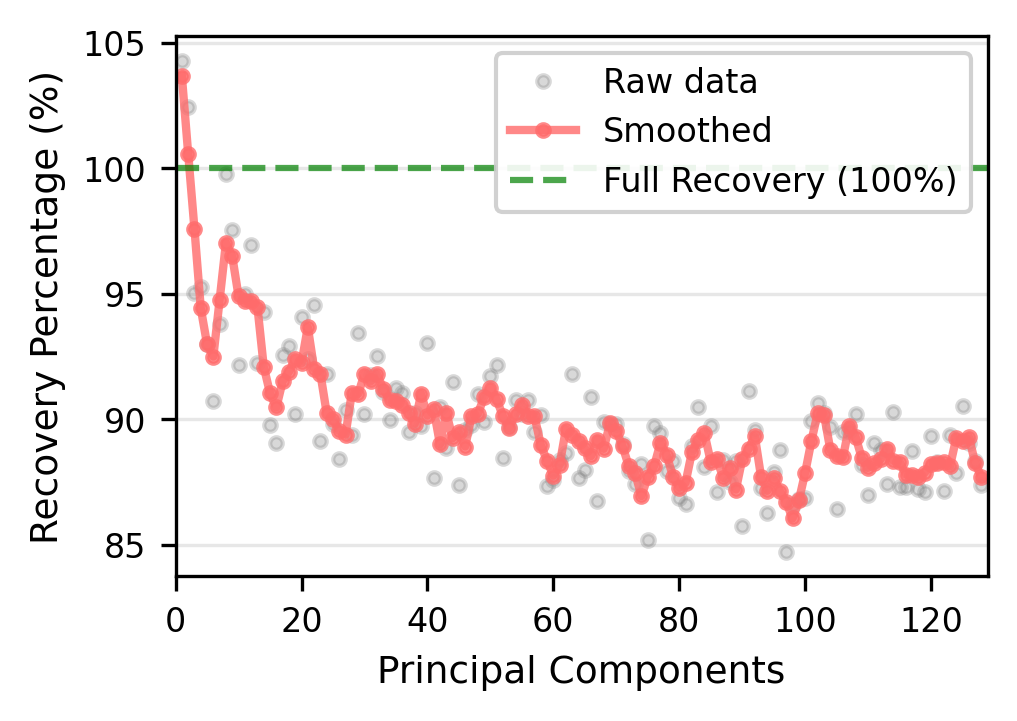} \\
\end{tabular}
\caption{Consistency of Observations~2--3 across unlearning losses on WMDP-Cyber (Llama-3.1-8B, full fine-tuning), part 1. For each method, the left plot shows the per-PC change ratio induced by unlearning and the right plot shows the per-PC recovery ratio after the RTT relearning attack.}
\label{fig:obs_consistency_losses}
\end{figure}

\begin{figure}[t]
\centering
\setlength{\tabcolsep}{2pt}
\renewcommand{\arraystretch}{0.68}
\begin{tabular}{@{}c@{\hspace{4pt}}c@{}}
\textbf{\footnotesize Unlearn Change Ratio} & \textbf{\footnotesize Recovery Ratio} \\[2pt]
\multicolumn{2}{c}{\footnotesize\textit{MLP Breaking}} \\
\includegraphics[width=0.44\linewidth]{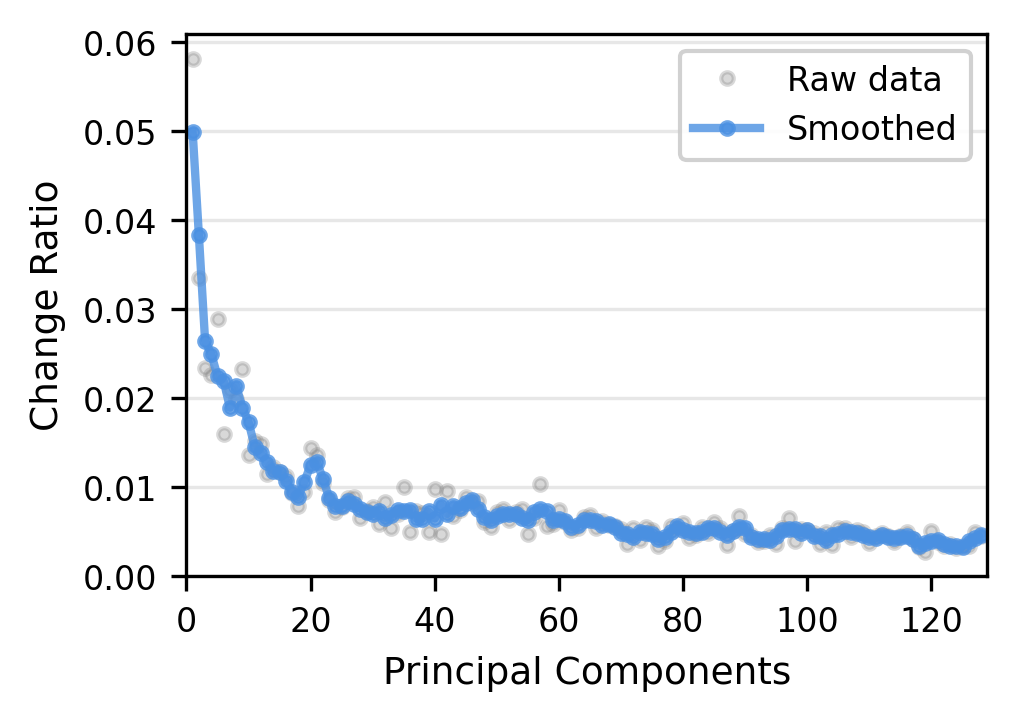} &
\includegraphics[width=0.44\linewidth]{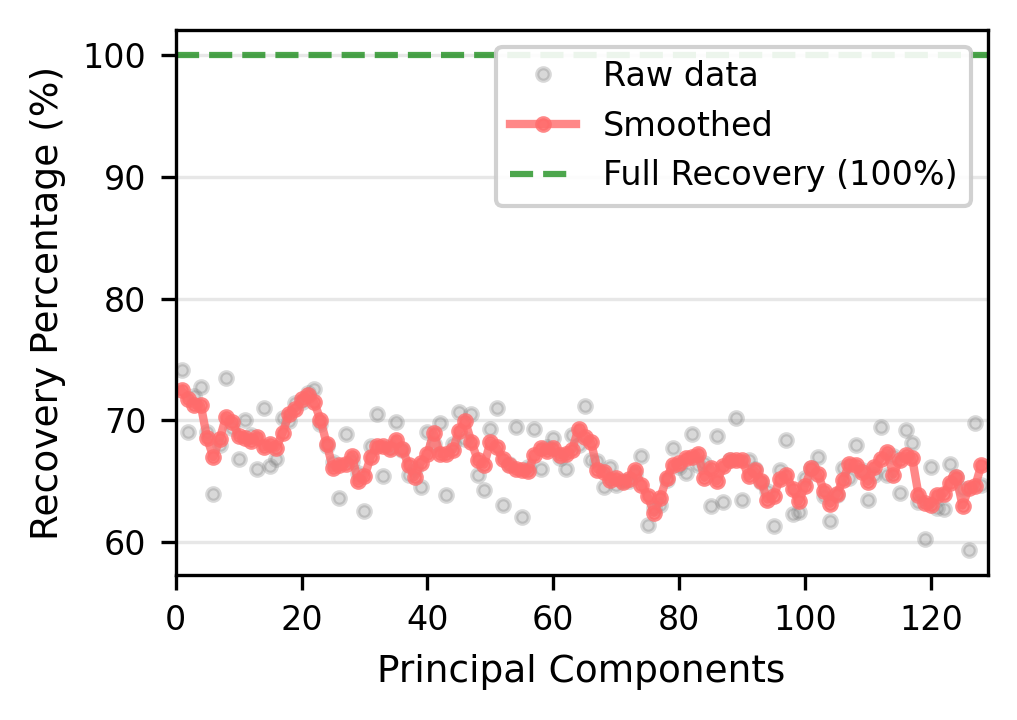} \\[1pt]
\multicolumn{2}{c}{\footnotesize\textit{RMU}} \\
\includegraphics[width=0.44\linewidth]{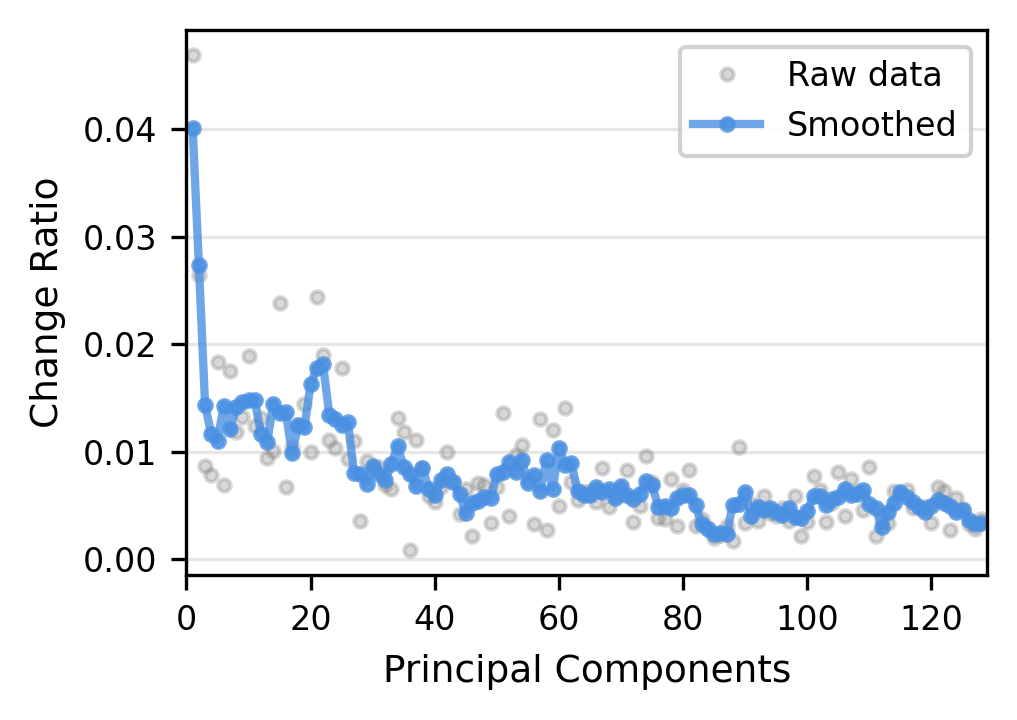} &
\includegraphics[width=0.44\linewidth]{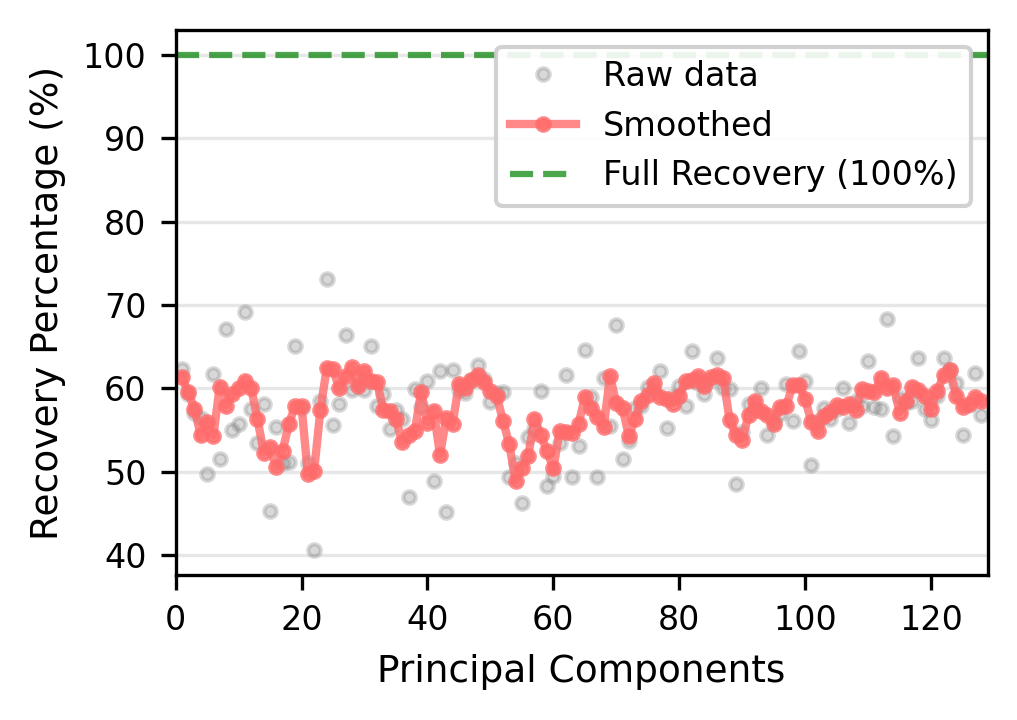} \\
\end{tabular}
\caption{Consistency of Observations~2--3 across unlearning losses on WMDP-Cyber (Llama-3.1-8B, full fine-tuning), part 2. The dominant components consistently absorb most of the unlearning change and exhibit the highest recovery, regardless of the specific unlearning loss.}
\label{fig:obs_consistency_losses_cont}
\end{figure}

\section{Robustness of Observations to Forget-Set Size}
\label{app:obs_forget_size}

A natural concern is whether Observations~1--3 are artifacts of a particular forget-set size, or could be driven by the size of the sample used to fit PCA. To rule this out, we run two complementary experiments on WMDP-Cyber with GA and Llama-3.1-8B (full fine-tuning), varying the forget-set fraction in $\{25\%, 50\%, 75\%, 100\%\}$.

\paragraph{Experiment A: fixed model, varying PCA-fit subset.}
We keep the unlearned and relearned models fixed (trained on the full forget set) and only vary the size of the subset used to fit PCA. This isolates whether Observation~1 (variance concentration in the dominant components) is a consequence of using too few samples for PCA. \Cref{fig:forget_size_A_explained_variance} shows the explained-variance curves: the spectrum is essentially indistinguishable across $25\%$, $50\%$, $75\%$, and $100\%$ subsets, so variance concentration is not a sample-size artifact. Correspondingly, \Cref{fig:forget_size_A} reports the change-ratio and recovery-ratio plots for the four subset sizes; both retain the same dominant-component-heavy pattern.

\begin{figure}[t]
\centering
\setlength{\tabcolsep}{2pt}
\renewcommand{\arraystretch}{0.68}
\begin{tabular}{@{}c@{\hspace{3pt}}c@{\hspace{3pt}}c@{}}
\textbf{\footnotesize down\_proj} & \textbf{\footnotesize gate\_proj} & \textbf{\footnotesize up\_proj} \\[2pt]
\multicolumn{3}{c}{\footnotesize\textit{$f=25\%$}} \\
\includegraphics[width=0.31\linewidth]{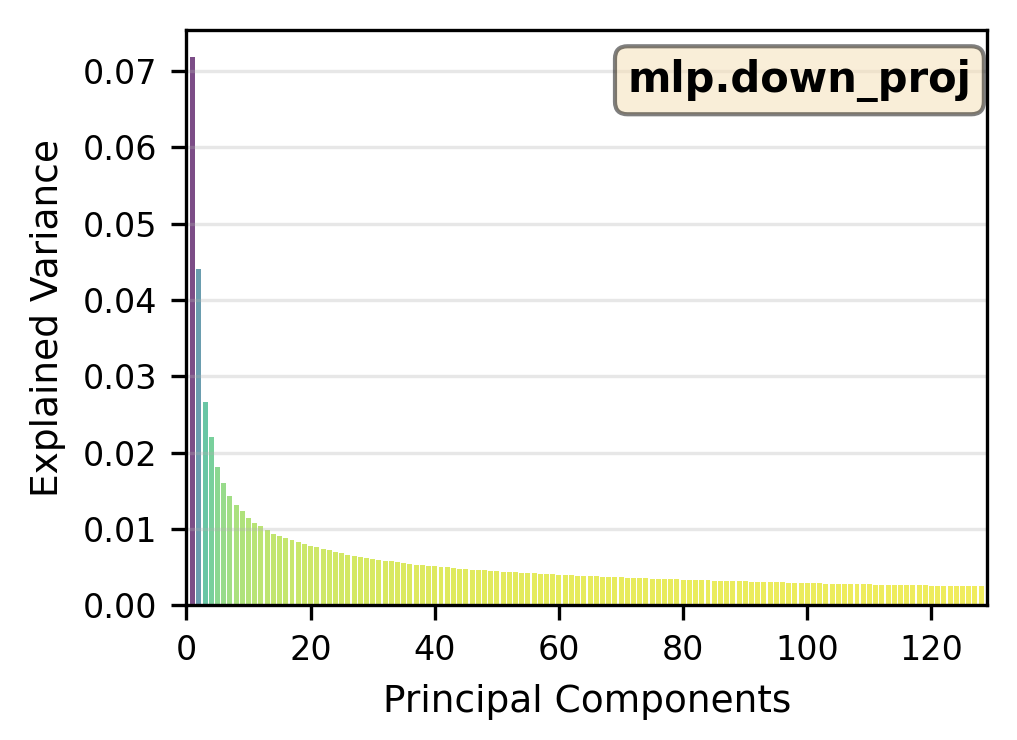} &
\includegraphics[width=0.31\linewidth]{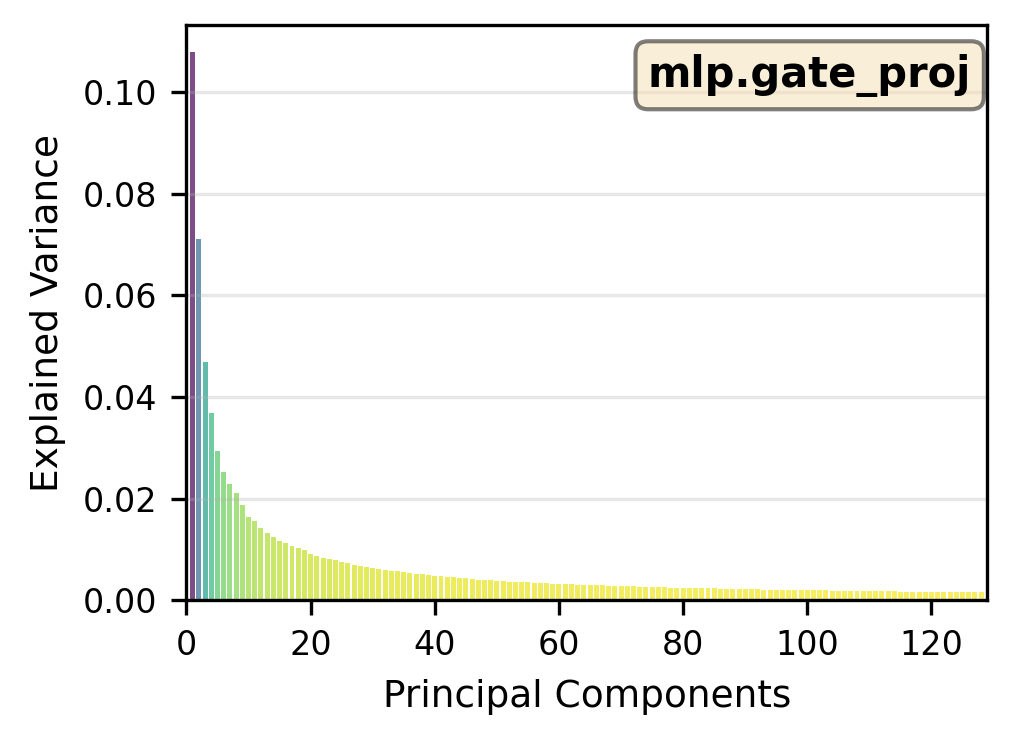} &
\includegraphics[width=0.31\linewidth]{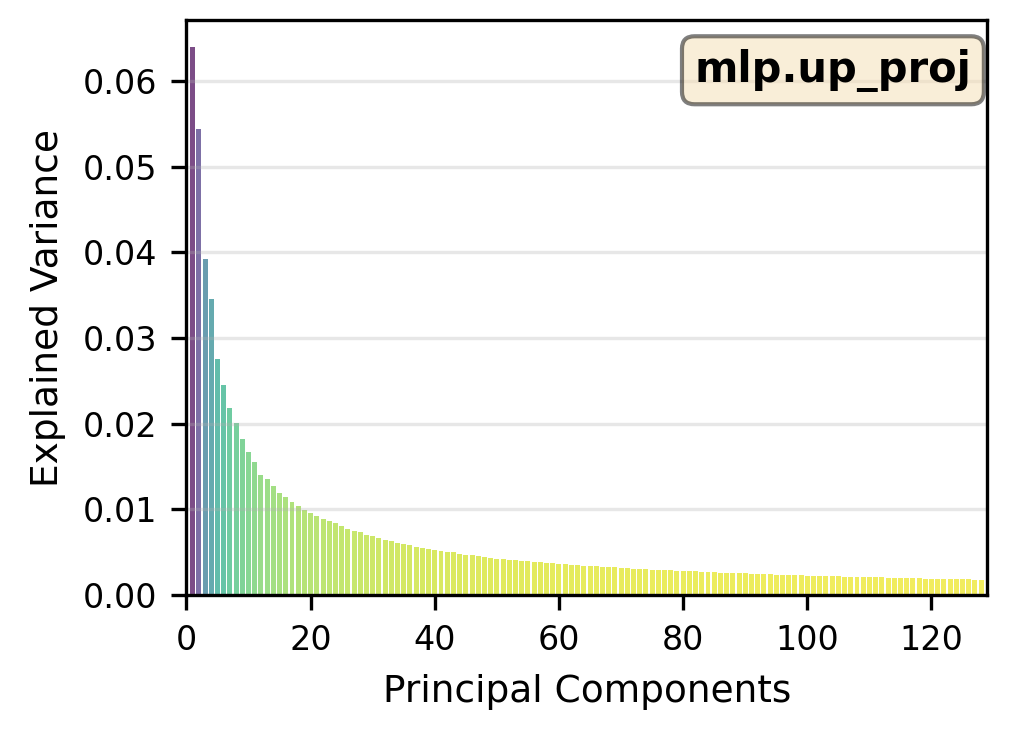} \\[1pt]
\multicolumn{3}{c}{\footnotesize\textit{$f=50\%$}} \\
\includegraphics[width=0.31\linewidth]{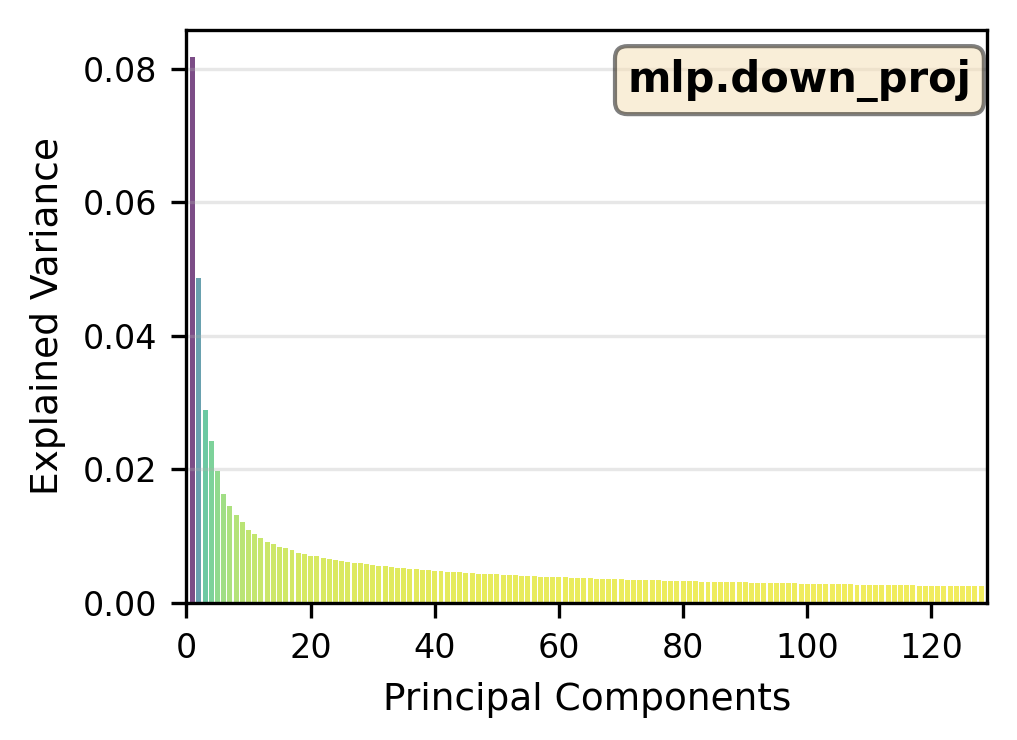} &
\includegraphics[width=0.31\linewidth]{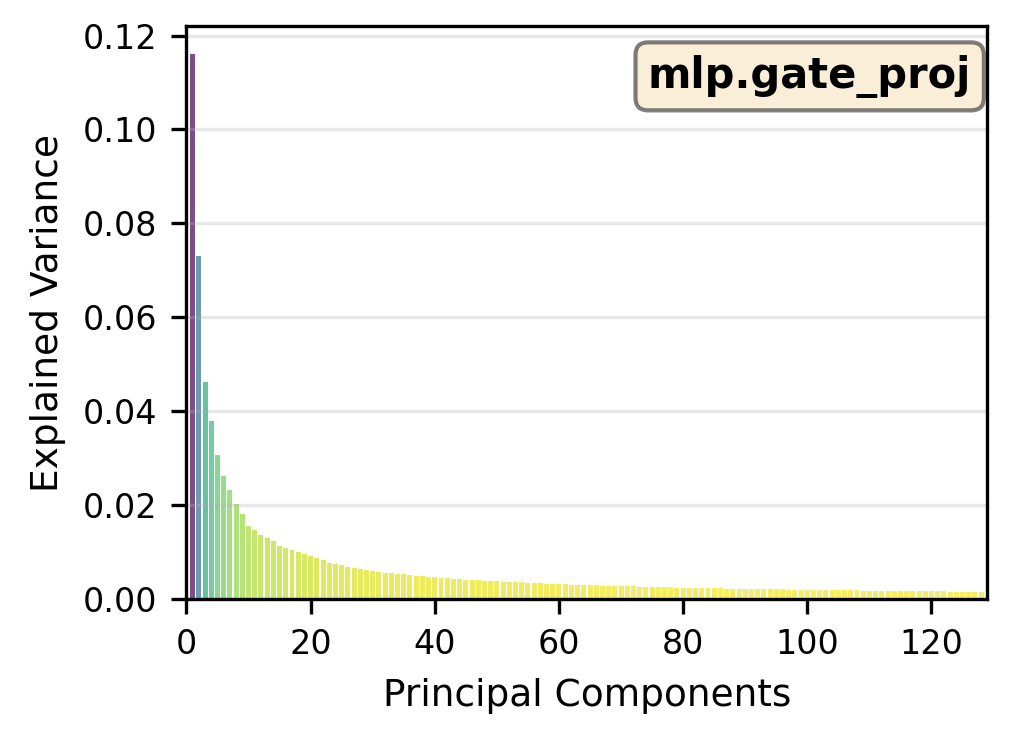} &
\includegraphics[width=0.31\linewidth]{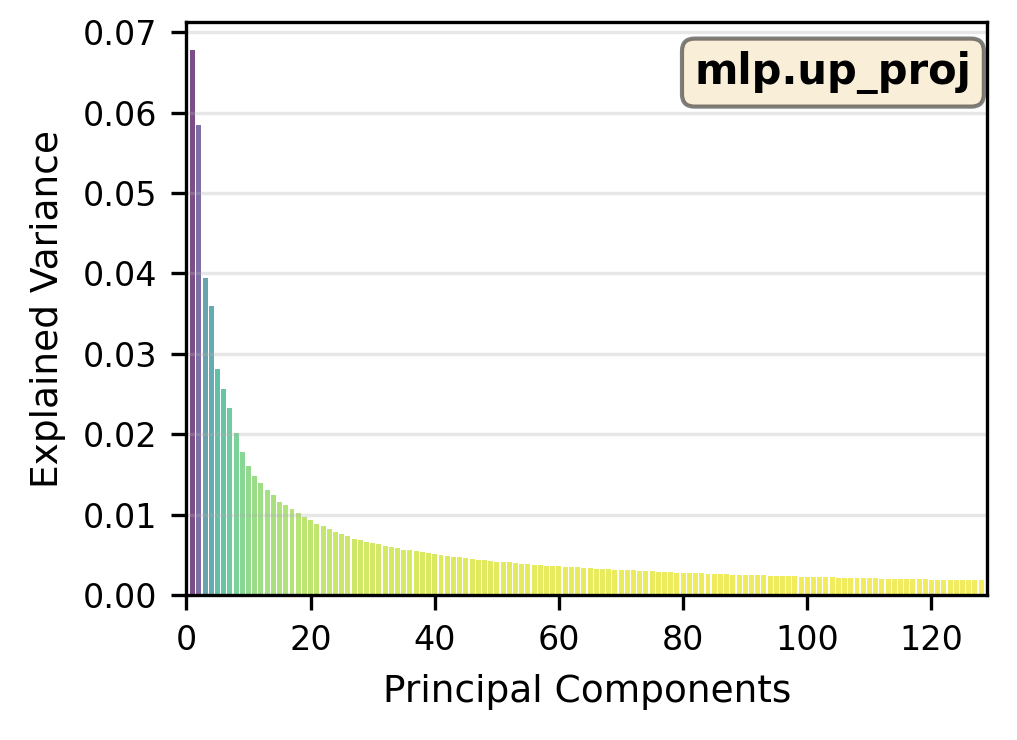} \\[1pt]
\multicolumn{3}{c}{\footnotesize\textit{$f=75\%$}} \\
\includegraphics[width=0.31\linewidth]{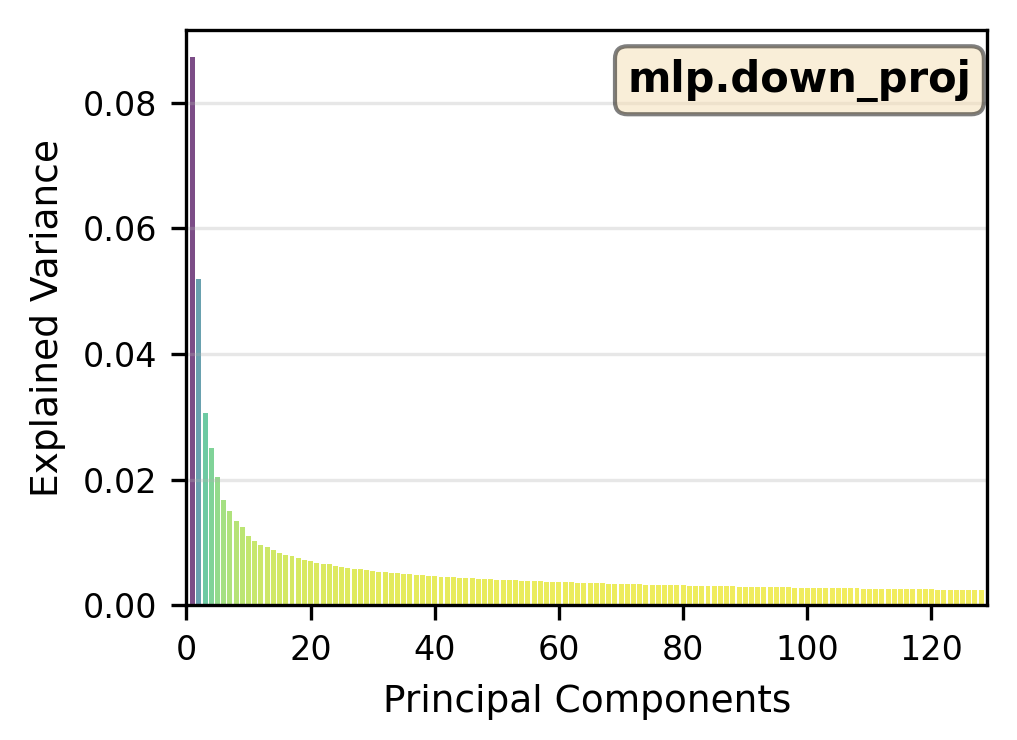} &
\includegraphics[width=0.31\linewidth]{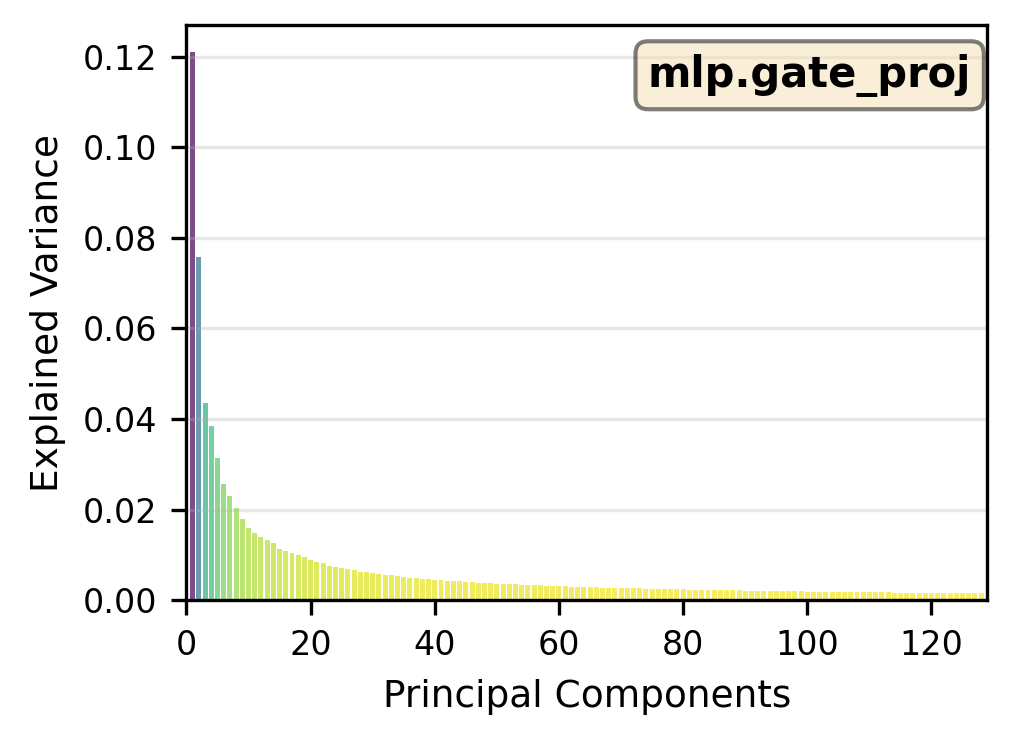} &
\includegraphics[width=0.31\linewidth]{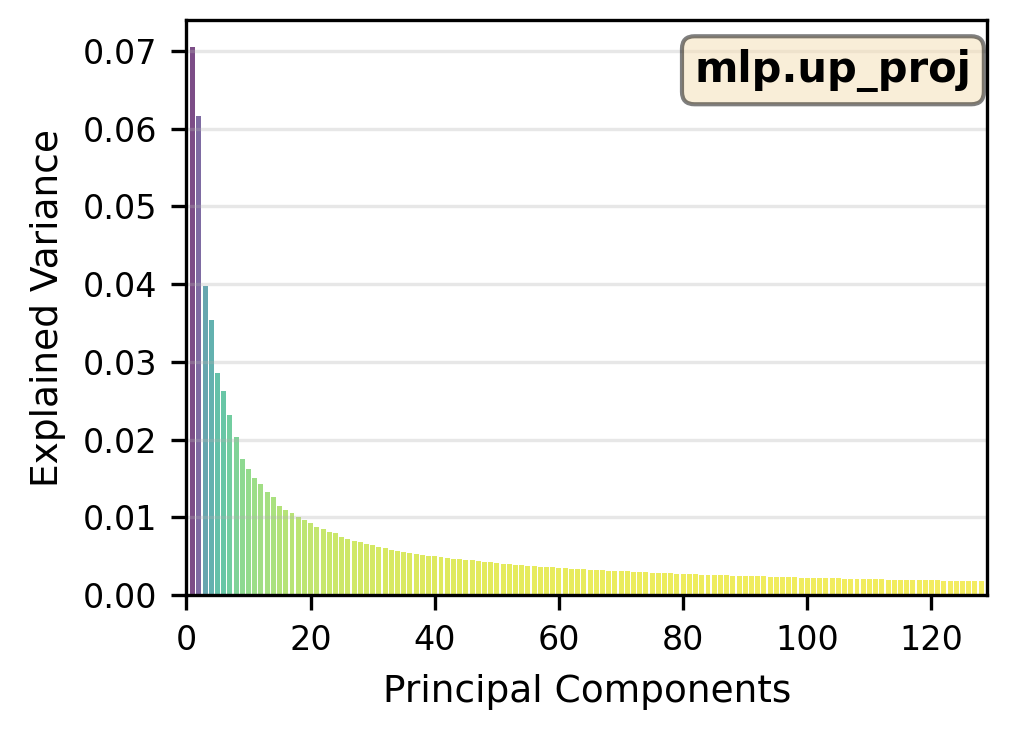} \\[1pt]
\multicolumn{3}{c}{\footnotesize\textit{$f=100\%$}} \\
\includegraphics[width=0.31\linewidth]{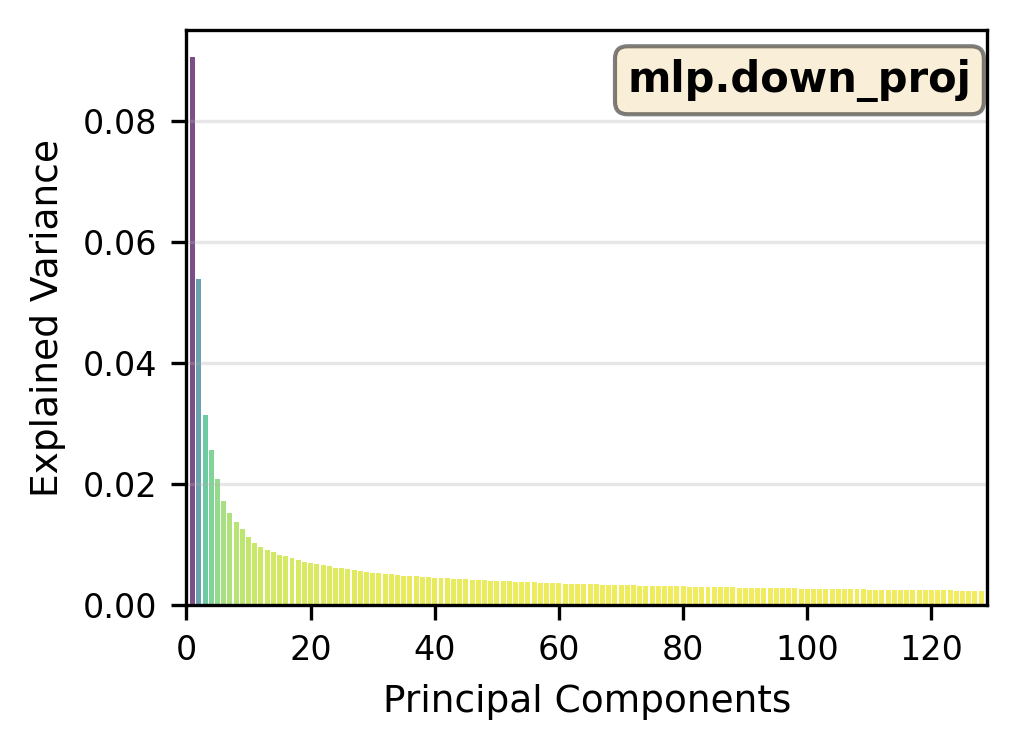} &
\includegraphics[width=0.31\linewidth]{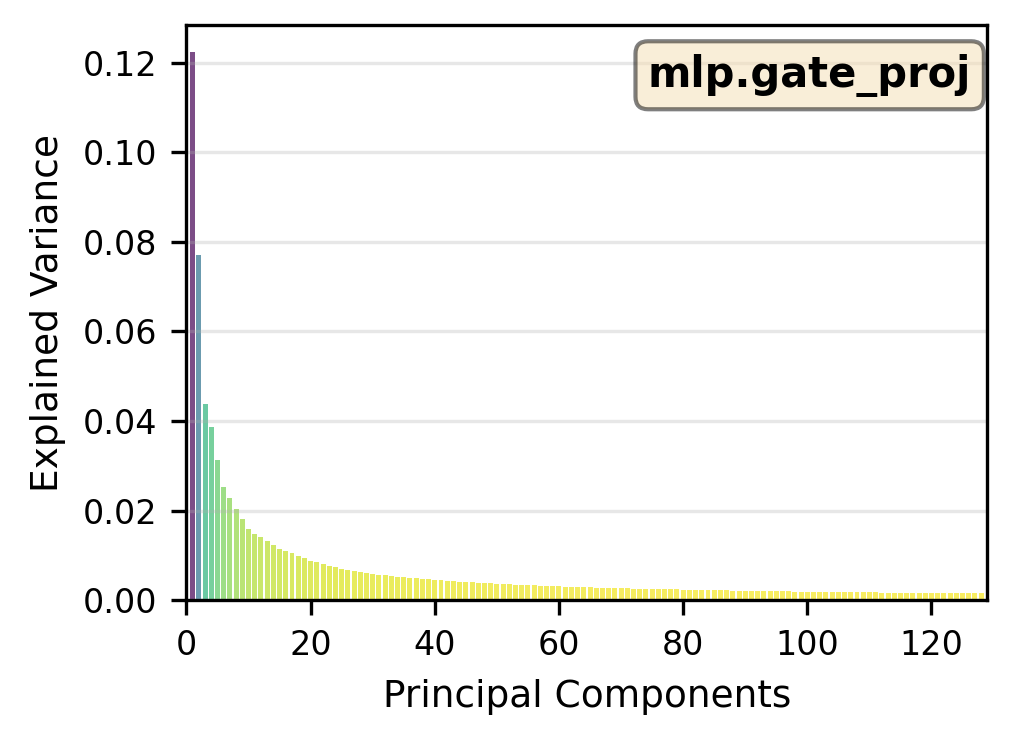} &
\includegraphics[width=0.31\linewidth]{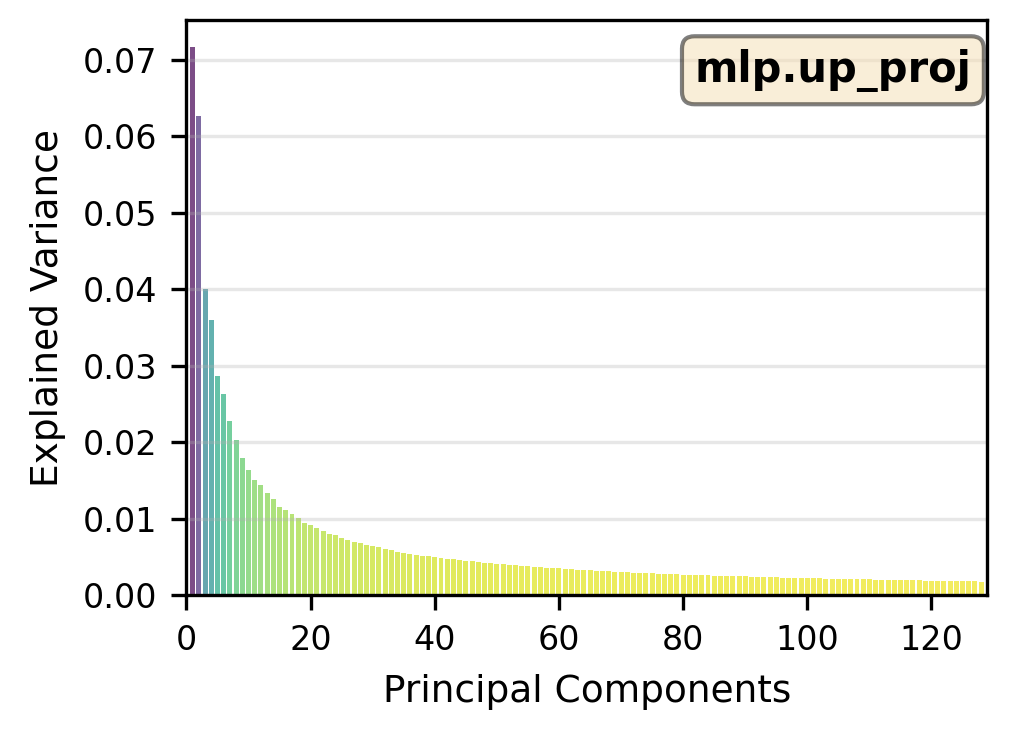} \\
\end{tabular}
\caption{Experiment A: explained variance under varying PCA-fit subset sizes (fixed unlearned/relearned model). The spectrum is essentially identical across $25$--$100\%$ subsets, so variance concentration is not driven by sample size.}
\label{fig:forget_size_A_explained_variance}
\end{figure}

\begin{figure}[t]
\centering
\setlength{\tabcolsep}{2pt}
\renewcommand{\arraystretch}{0.68}
\begin{tabular}{@{}c@{\hspace{4pt}}c@{}}
\textbf{\footnotesize Unlearn Change Ratio} & \textbf{\footnotesize Recovery Ratio} \\[2pt]
\multicolumn{2}{c}{\footnotesize\textit{$f=25\%$}} \\
\includegraphics[width=0.44\linewidth]{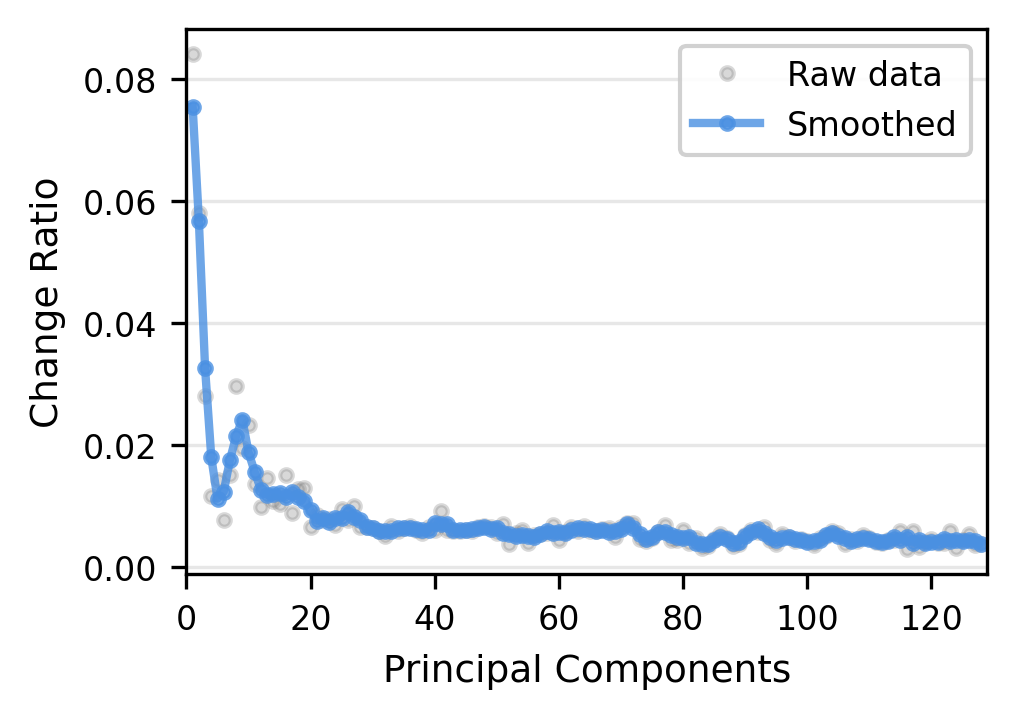} &
\includegraphics[width=0.44\linewidth]{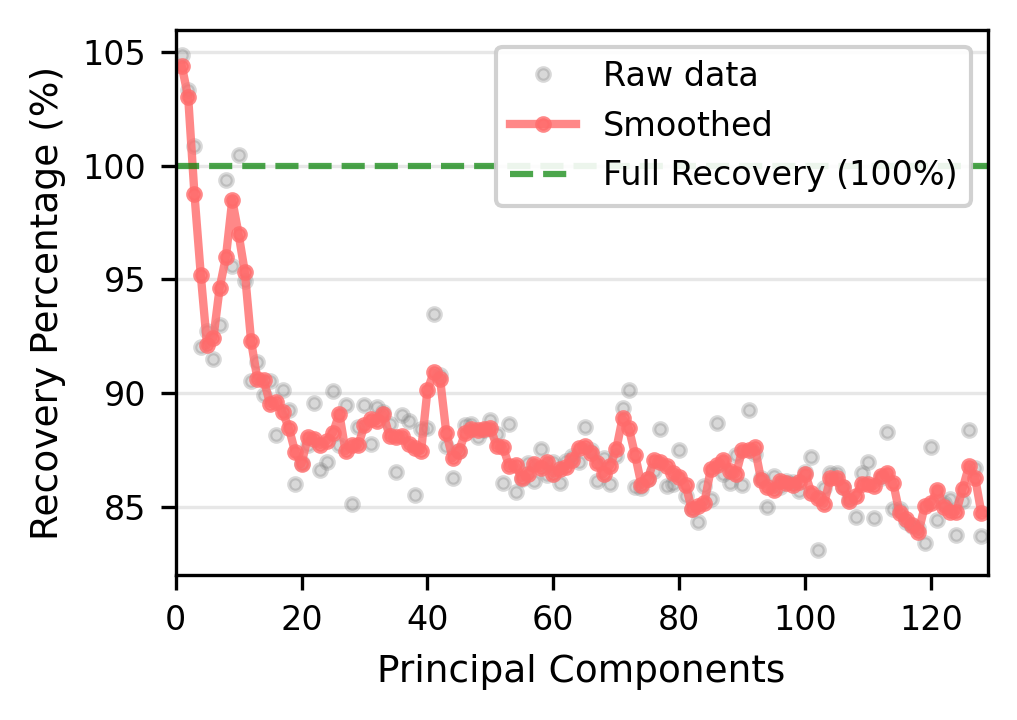} \\[1pt]
\multicolumn{2}{c}{\footnotesize\textit{$f=50\%$}} \\
\includegraphics[width=0.44\linewidth]{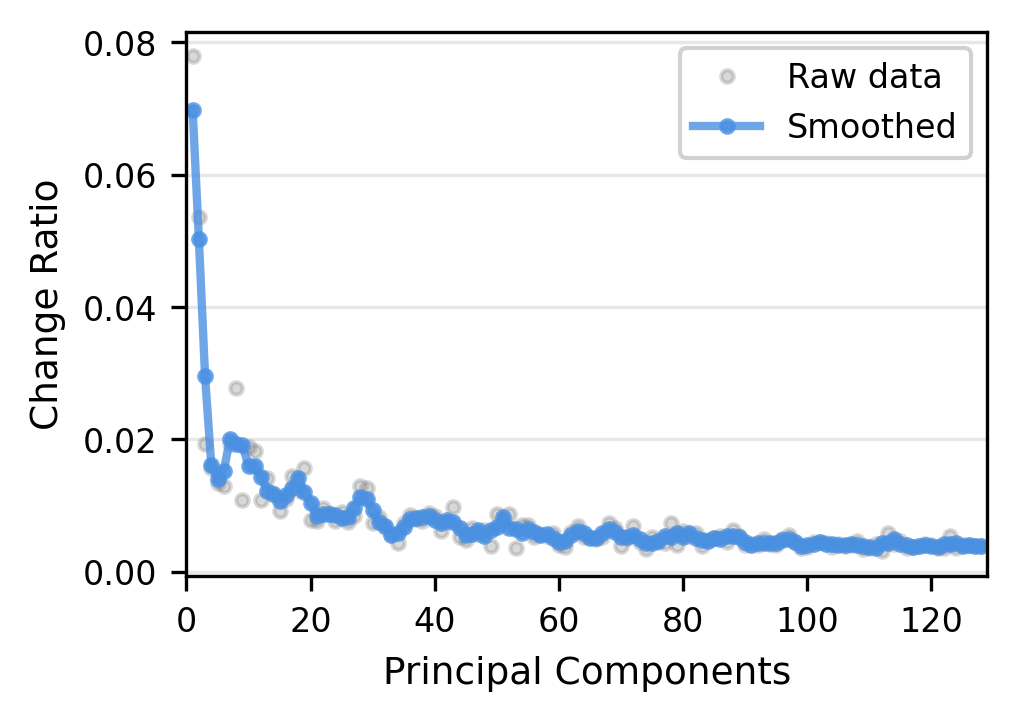} &
\includegraphics[width=0.44\linewidth]{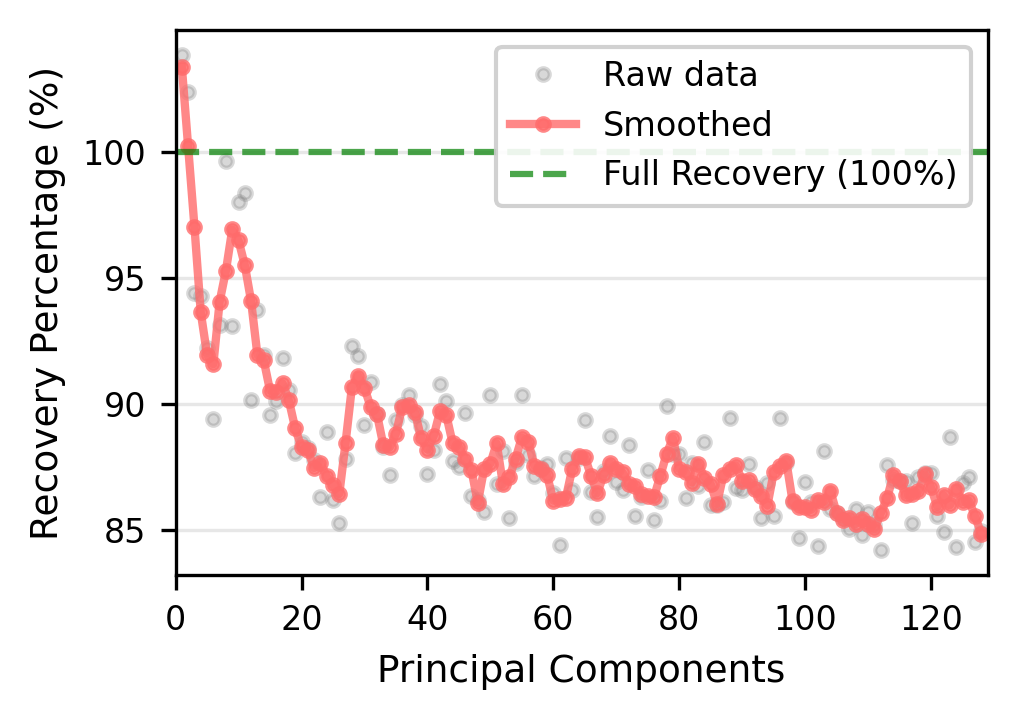} \\[1pt]
\multicolumn{2}{c}{\footnotesize\textit{$f=75\%$}} \\
\includegraphics[width=0.44\linewidth]{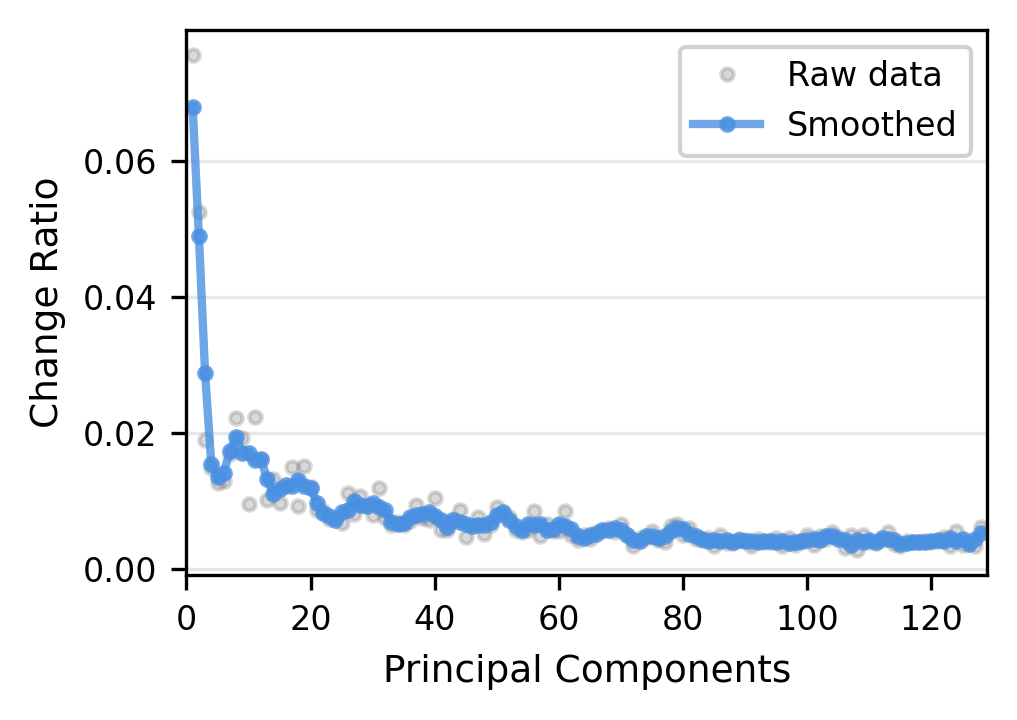} &
\includegraphics[width=0.44\linewidth]{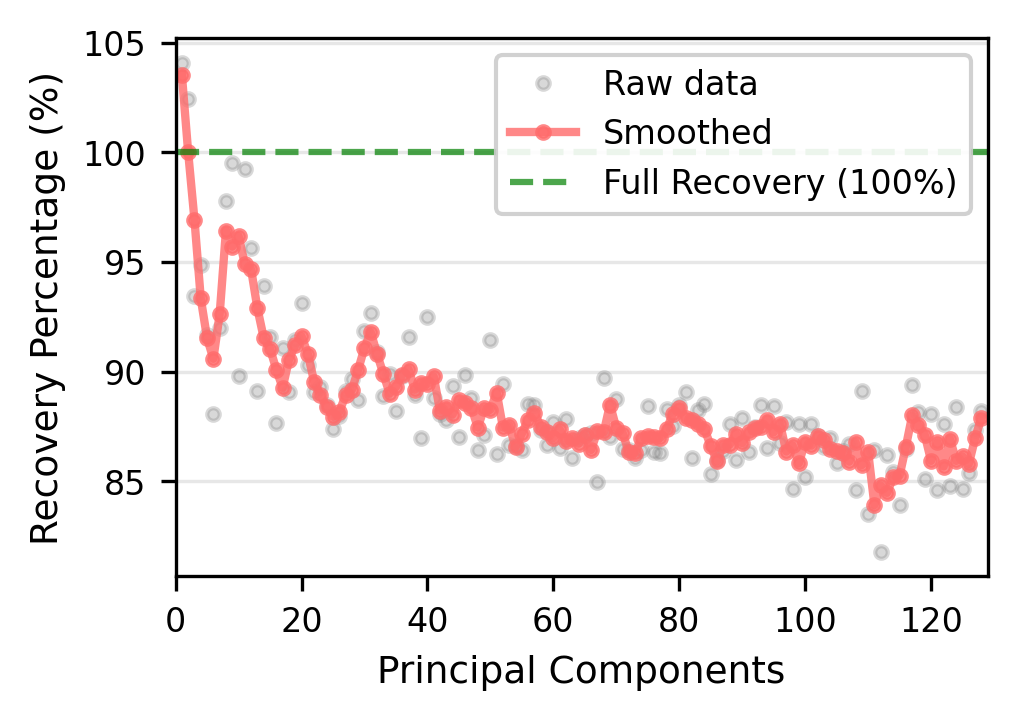} \\[1pt]
\multicolumn{2}{c}{\footnotesize\textit{$f=100\%$}} \\
\includegraphics[width=0.44\linewidth]{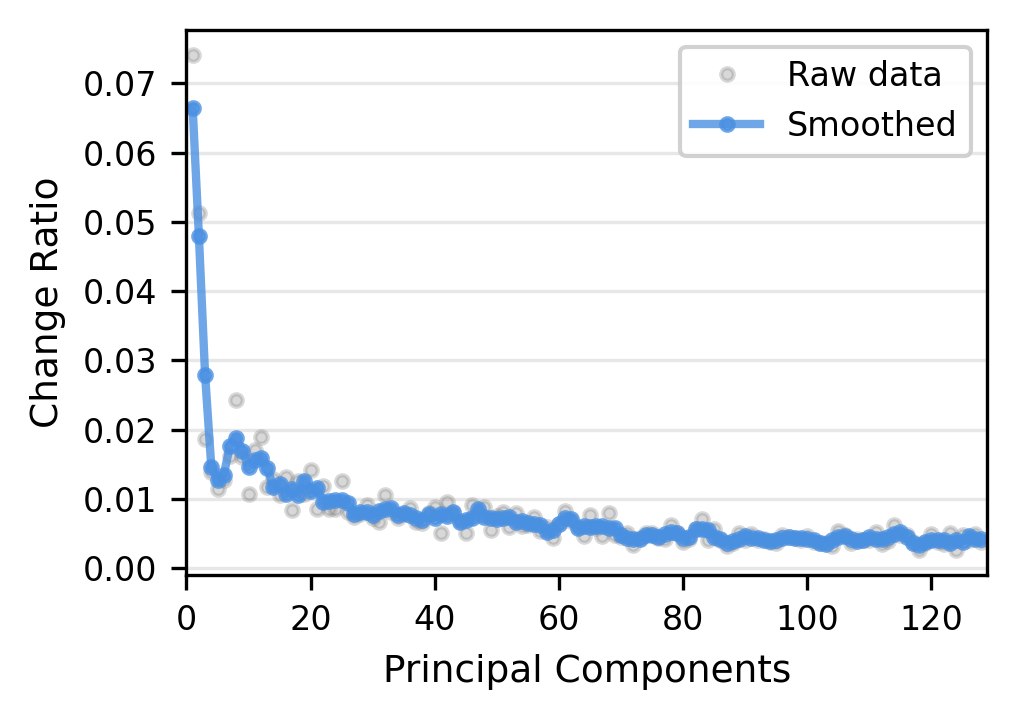} &
\includegraphics[width=0.44\linewidth]{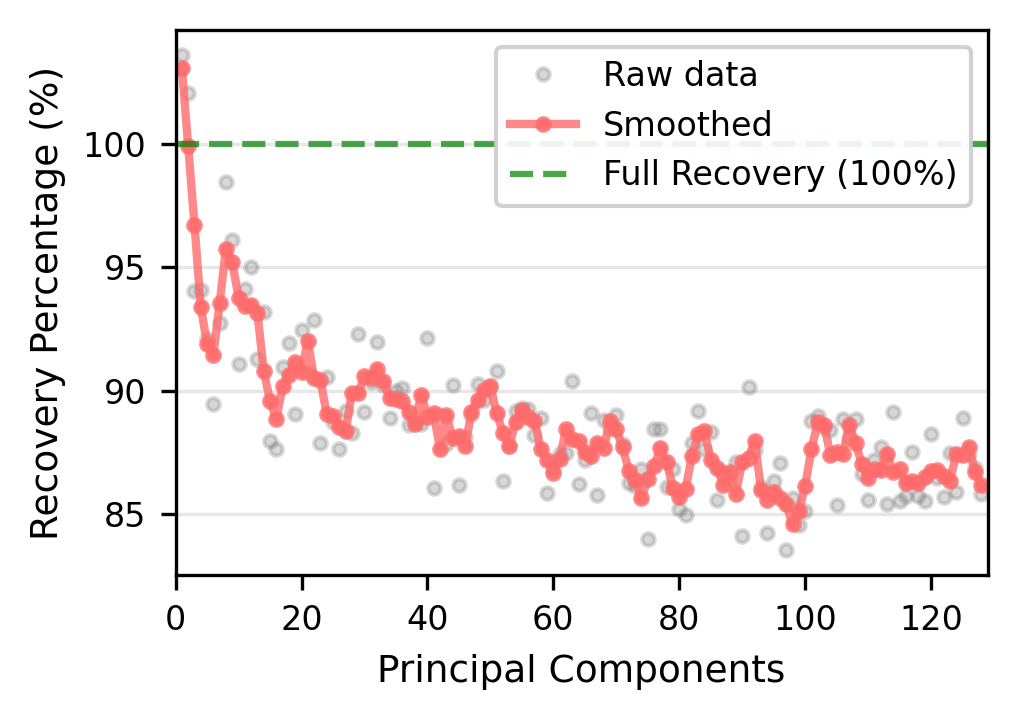} \\
\end{tabular}
\caption{Experiment A: per-PC unlearn change ratio (left) and recovery ratio (right) as a function of the PCA-fit subset size (rows: $f=25,50,75,100\%$), with the unlearned/relearned model held fixed. Observations~2 and 3 are stable across PCA-fit sizes.}
\label{fig:forget_size_A}
\end{figure}

\paragraph{Experiment B: end-to-end varying forget size.}
We then re-run the full unlearning + relearning pipeline on each forget-set fraction, so both the model and the PCA fit are tied to the same subset. This tests whether the observations hold when the unlearning procedure itself is varied. \Cref{fig:forget_size_B} shows that the change-ratio and recovery-ratio patterns remain consistent across all subset sizes: a small number of dominant PCs continue to absorb the bulk of unlearning-induced change and to be preferentially recovered after RTT.

\begin{figure}[t]
\centering
\setlength{\tabcolsep}{2pt}
\renewcommand{\arraystretch}{0.68}
\begin{tabular}{@{}c@{\hspace{4pt}}c@{}}
\textbf{\footnotesize Unlearn Change Ratio} & \textbf{\footnotesize Recovery Ratio} \\[2pt]
\multicolumn{2}{c}{\footnotesize\textit{$f=25\%$}} \\
\includegraphics[width=0.44\linewidth]{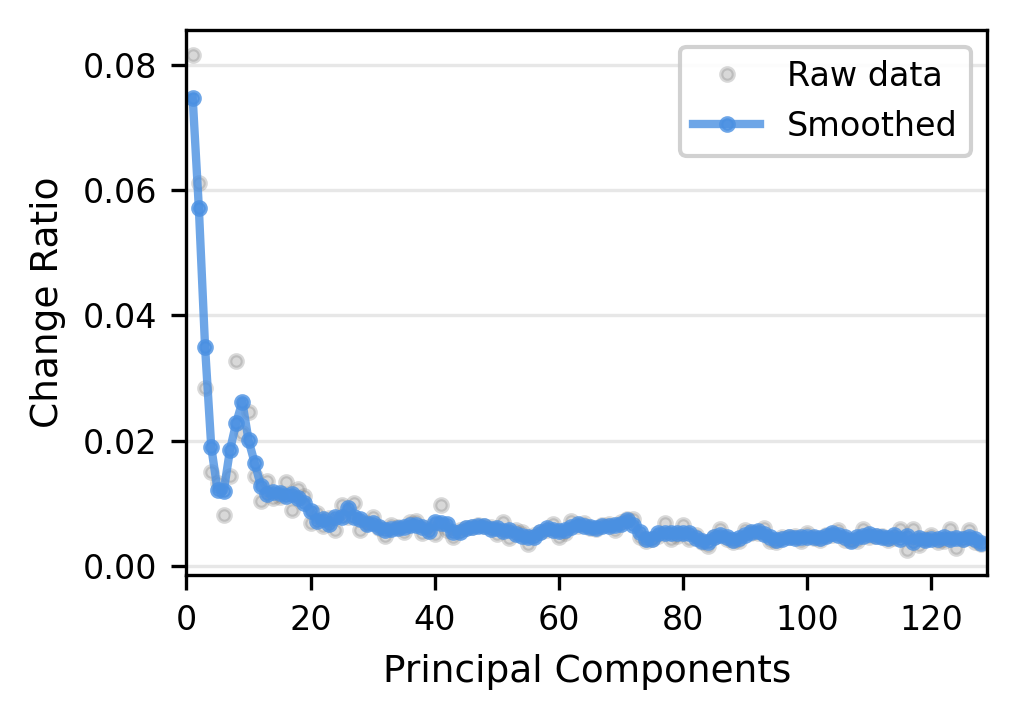} &
\includegraphics[width=0.44\linewidth]{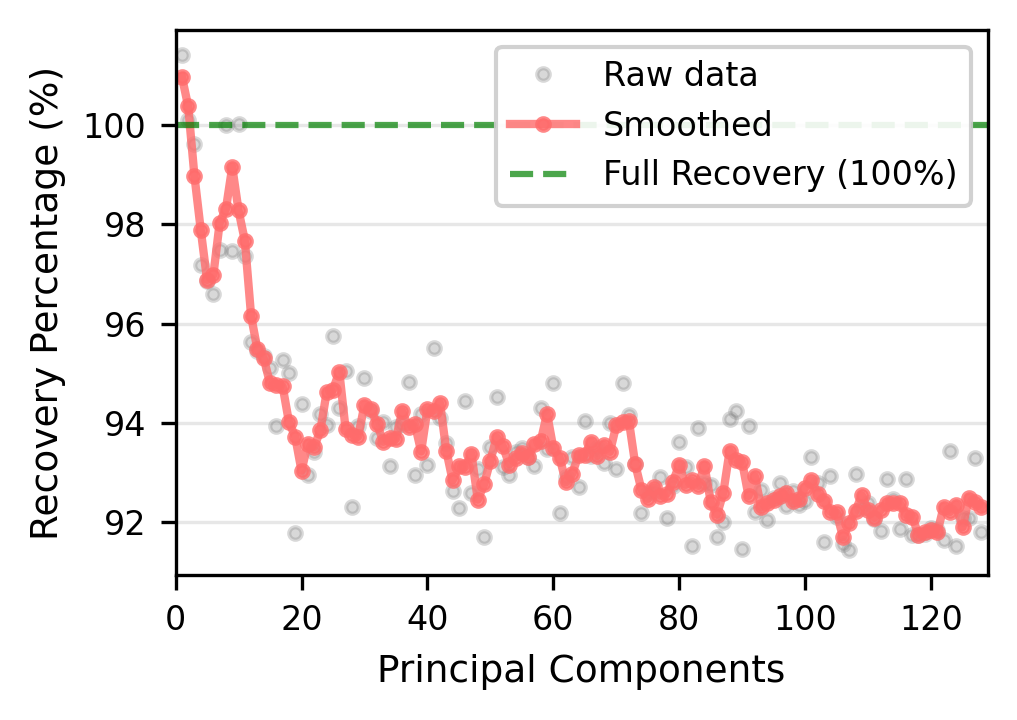} \\[1pt]
\multicolumn{2}{c}{\footnotesize\textit{$f=50\%$}} \\
\includegraphics[width=0.44\linewidth]{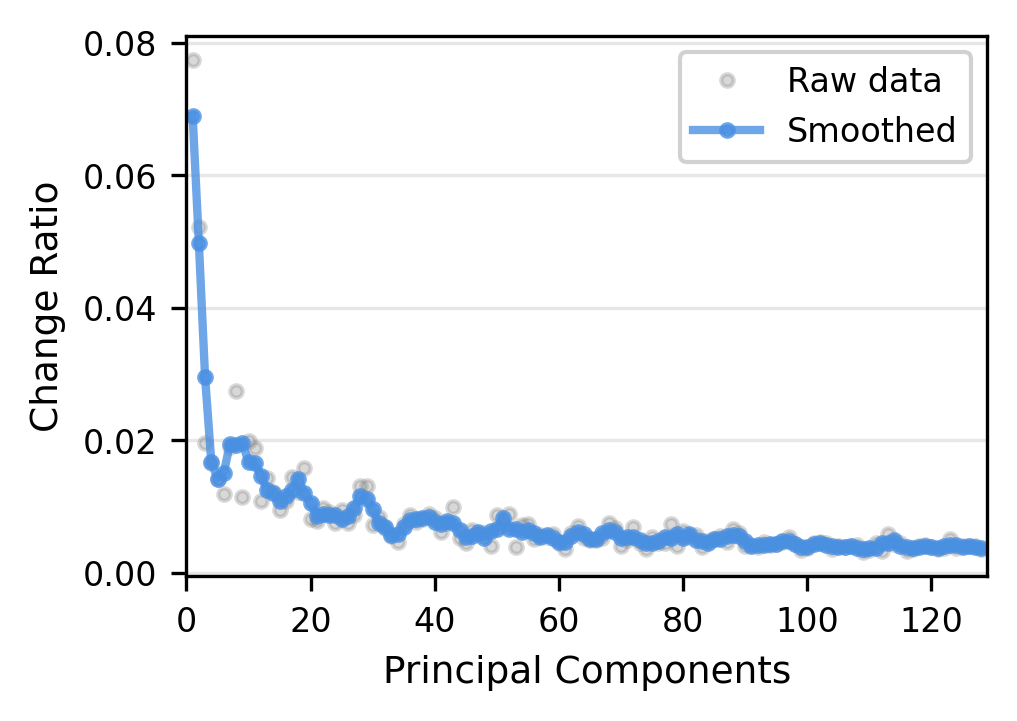} &
\includegraphics[width=0.44\linewidth]{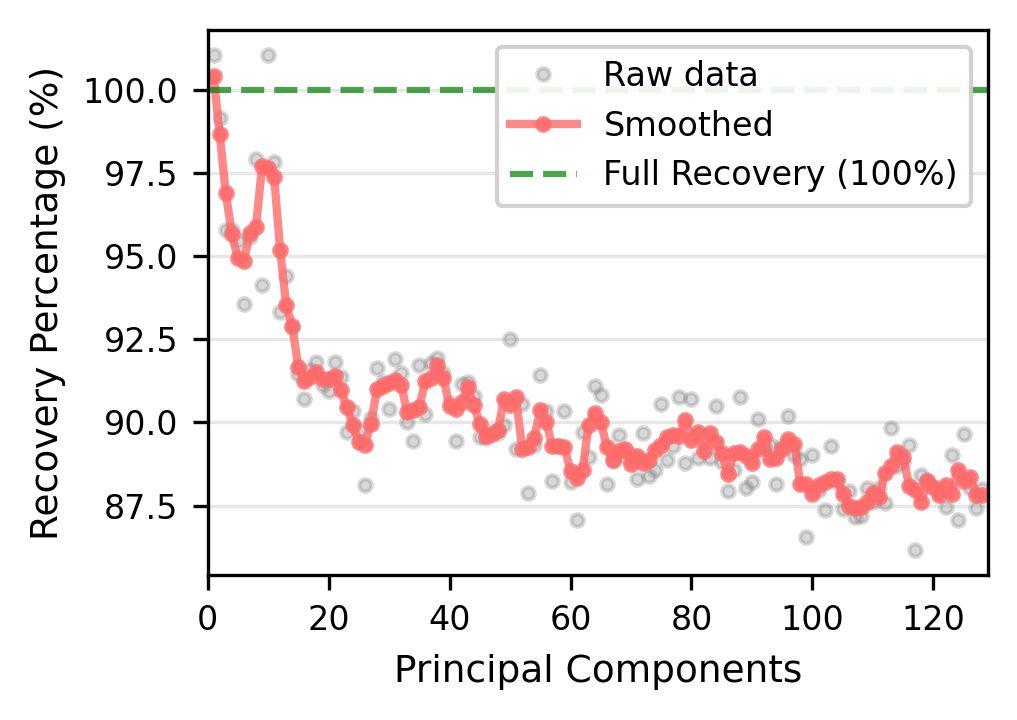} \\[1pt]
\multicolumn{2}{c}{\footnotesize\textit{$f=75\%$}} \\
\includegraphics[width=0.44\linewidth]{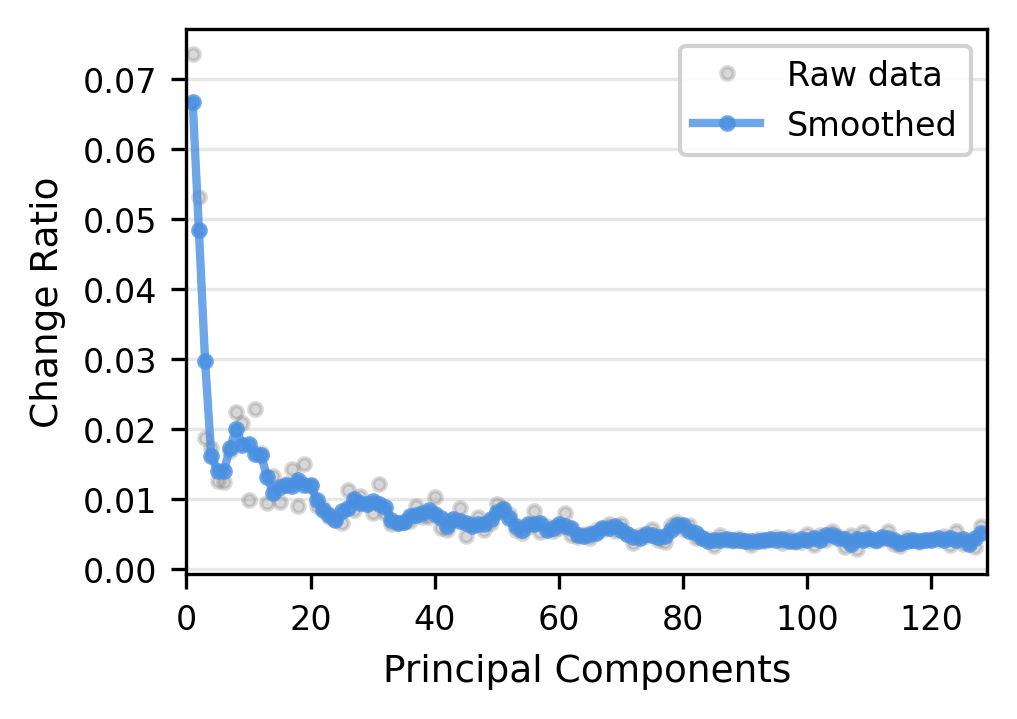} &
\includegraphics[width=0.44\linewidth]{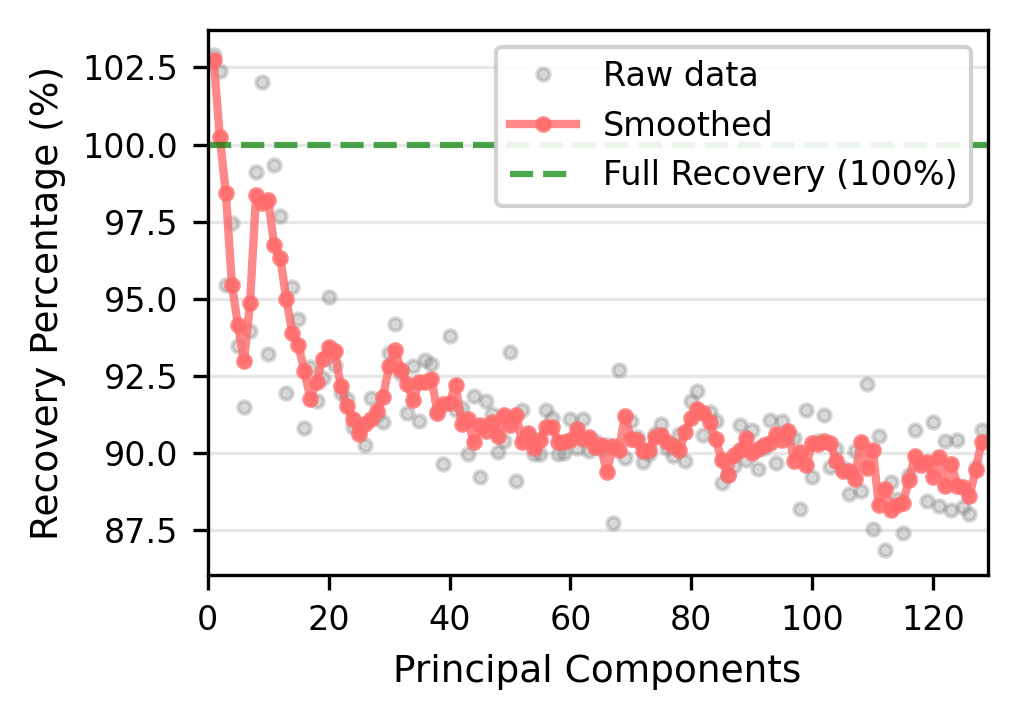} \\[1pt]
\multicolumn{2}{c}{\footnotesize\textit{$f=100\%$}} \\
\includegraphics[width=0.44\linewidth]{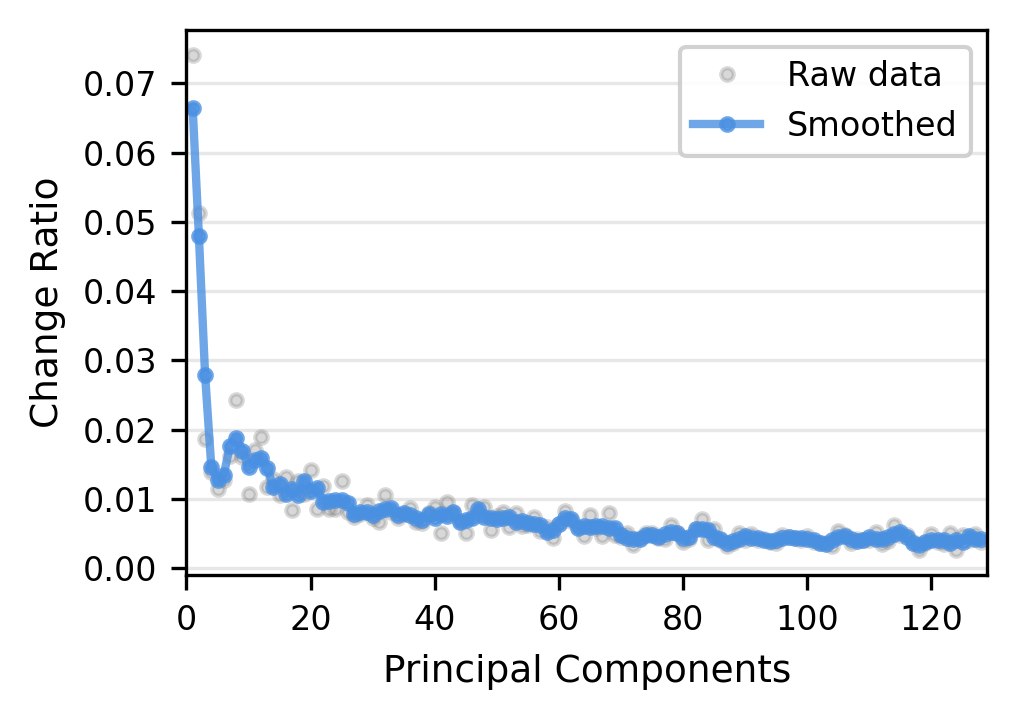} &
\includegraphics[width=0.44\linewidth]{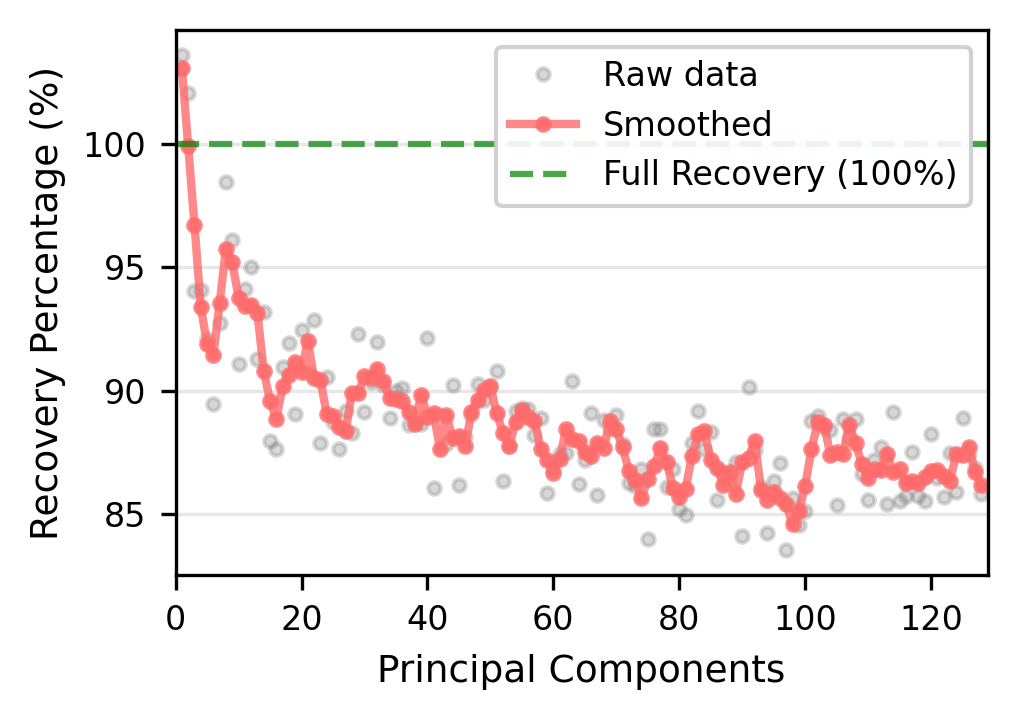} \\
\end{tabular}
\caption{Experiment B: per-PC unlearn change ratio (left) and recovery ratio (right) under end-to-end varying forget size (rows: $f=25,50,75,100\%$). The unlearning + relearning pipeline is rerun on each subset. Observations~2 and 3 hold across all forget-set sizes.}
\label{fig:forget_size_B}
\end{figure}

Together, Experiments A and B confirm that all three observations are robust to the size of the forget set, and that the dominant-component vulnerability is a structural property of LLM representations rather than a consequence of a specific sampling regime.

\section{Robustness of Observations to Parameter-Efficient Fine-Tuning}
\label{app:obs_lora}

Our main analysis in \Cref{sec:understanding} uses full fine-tuning. To rule out the possibility that Observations~2--3 are an artifact of full-parameter optimization, we additionally evaluate them under \textbf{LoRA} fine-tuning at ranks $r \in \{8, 16, 32, 64\}$, applied to four unlearning losses (GA, NPO, RMU, MLP Breaking). Observation~1 concerns the representation geometry of the pre-trained model itself and is therefore independent of the fine-tuning method, so we focus on Observations~2 and 3.

\Cref{fig:lora_graddiff,fig:lora_npo,fig:lora_rmu,fig:lora_mlpconfuse} report, for each unlearning loss, the per-PC unlearn change ratio (left) and recovery ratio (right) under Full FT and the four LoRA ranks. In every row, the dominant components again absorb the majority of the unlearning-induced change and exhibit the highest recovery, matching the full fine-tuning pattern in \Cref{fig:pca_analysis}. This holds uniformly across the four unlearning losses and the four LoRA ranks, confirming that the dominant-component vulnerability is a property of LLM representation geometry rather than of the specific optimization regime, and that MCU's design (targeting the minor-component subspace) is therefore equally motivated for LoRA-based unlearning pipelines.

\begin{figure}[t]
\centering
\setlength{\tabcolsep}{2pt}
\renewcommand{\arraystretch}{0.68}
\begin{tabular}{@{}c@{\hspace{4pt}}c@{}}
\textbf{\footnotesize Unlearn Change Ratio} & \textbf{\footnotesize Recovery Ratio} \\[2pt]
\multicolumn{2}{c}{\footnotesize\textit{Full FT}} \\
\includegraphics[width=0.385\linewidth]{figures/unlearn_change_ratio_averaged.pdf} &
\includegraphics[width=0.385\linewidth]{figures/recovery_ratio_averaged.pdf} \\[1pt]
\multicolumn{2}{c}{\footnotesize\textit{$r=8$}} \\
\includegraphics[width=0.385\linewidth]{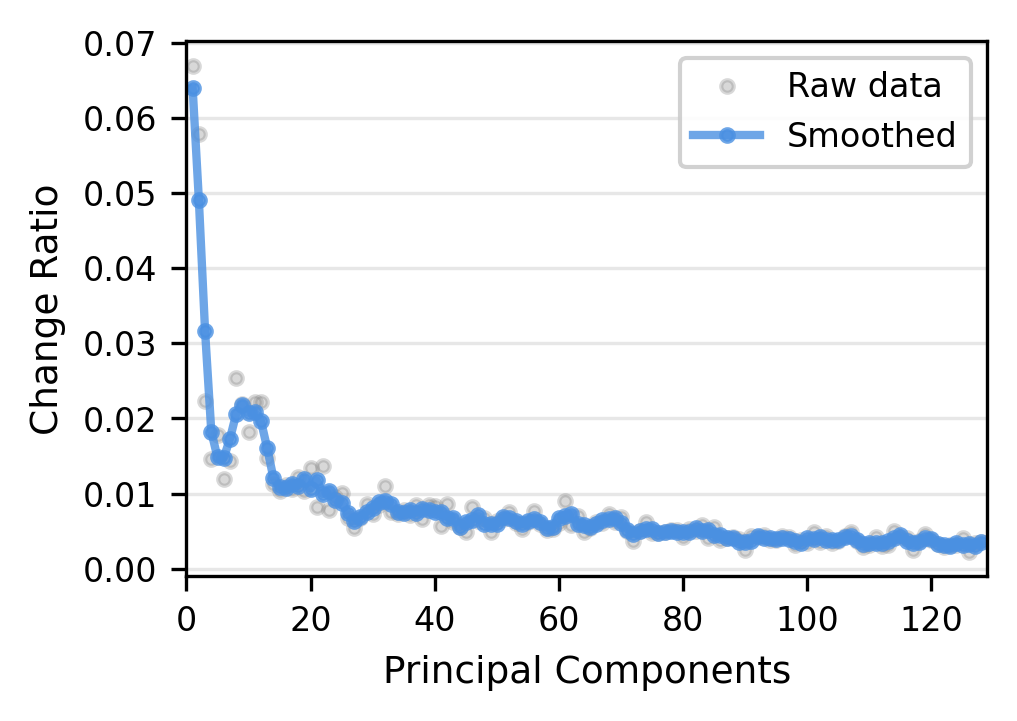}  &
\includegraphics[width=0.385\linewidth]{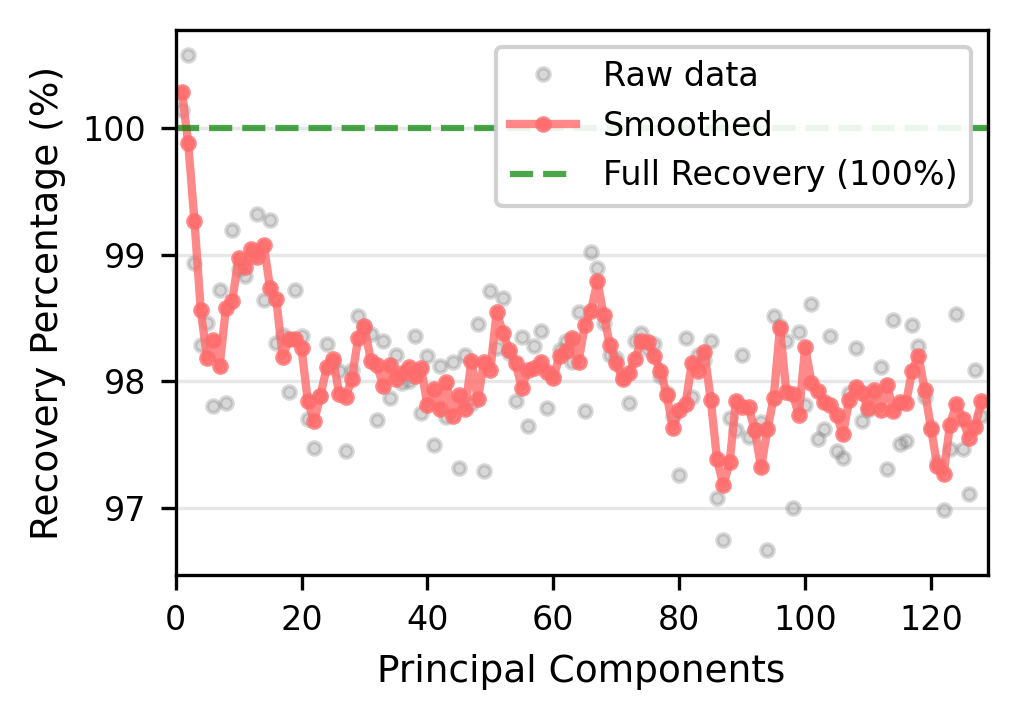}  \\[1pt]
\multicolumn{2}{c}{\footnotesize\textit{$r=16$}} \\
\includegraphics[width=0.385\linewidth]{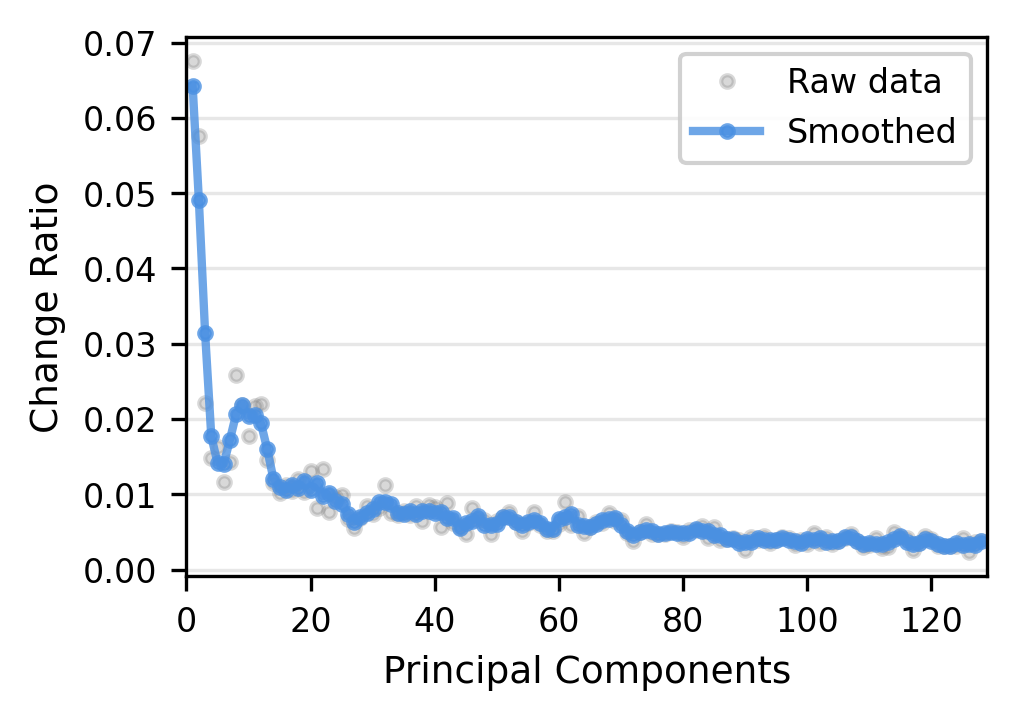} &
\includegraphics[width=0.385\linewidth]{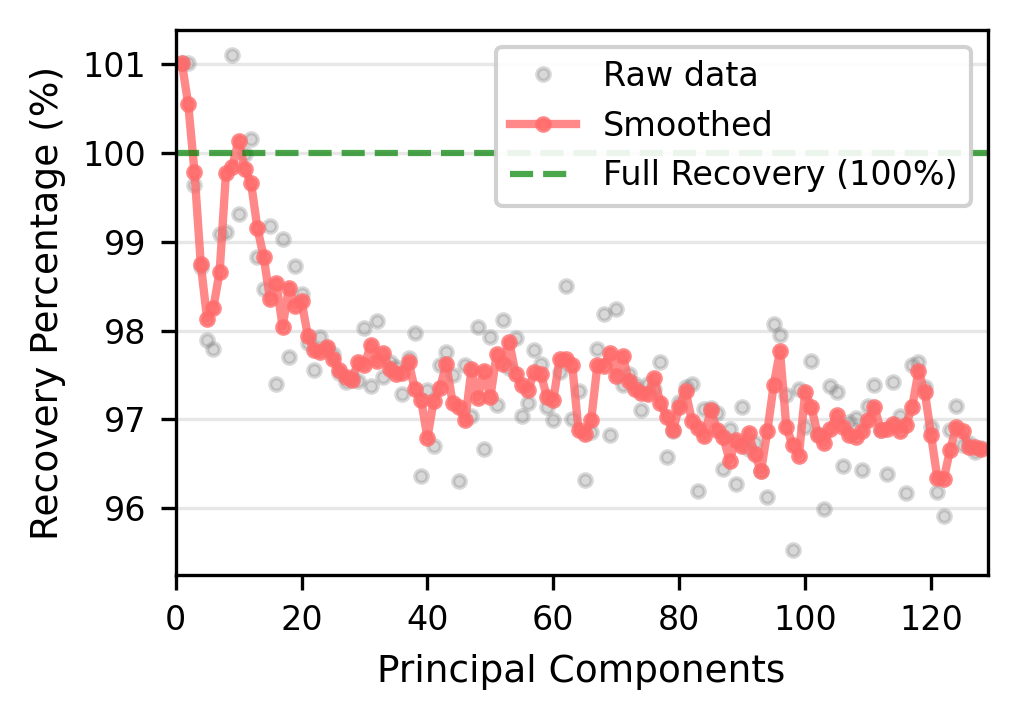} \\[1pt]
\multicolumn{2}{c}{\footnotesize\textit{$r=32$}} \\
\includegraphics[width=0.385\linewidth]{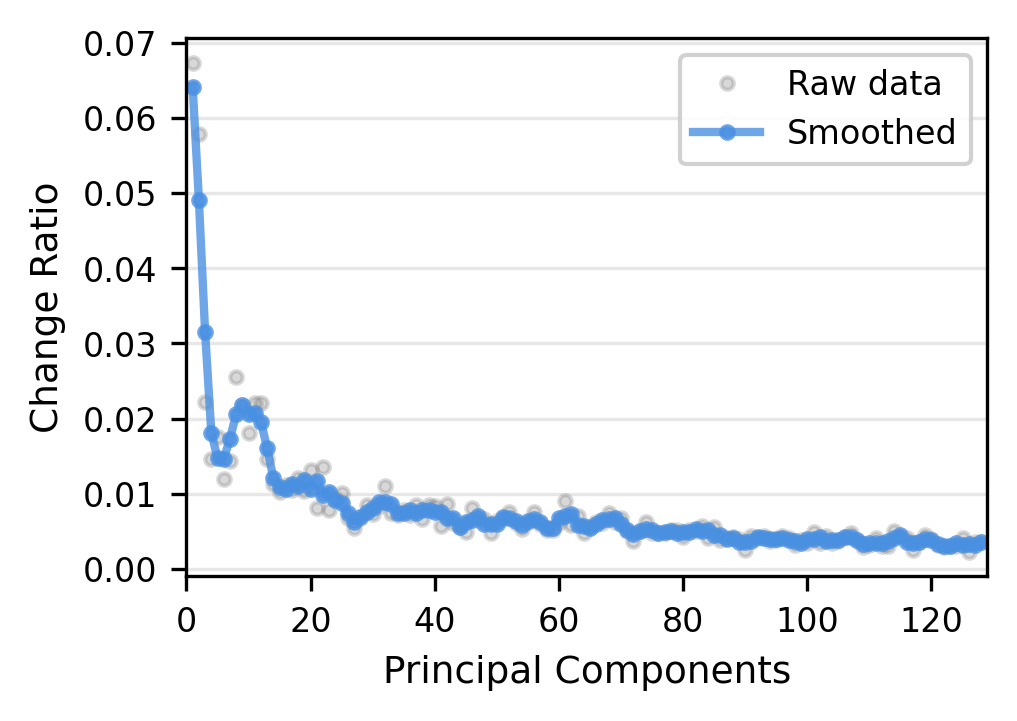} &
\includegraphics[width=0.385\linewidth]{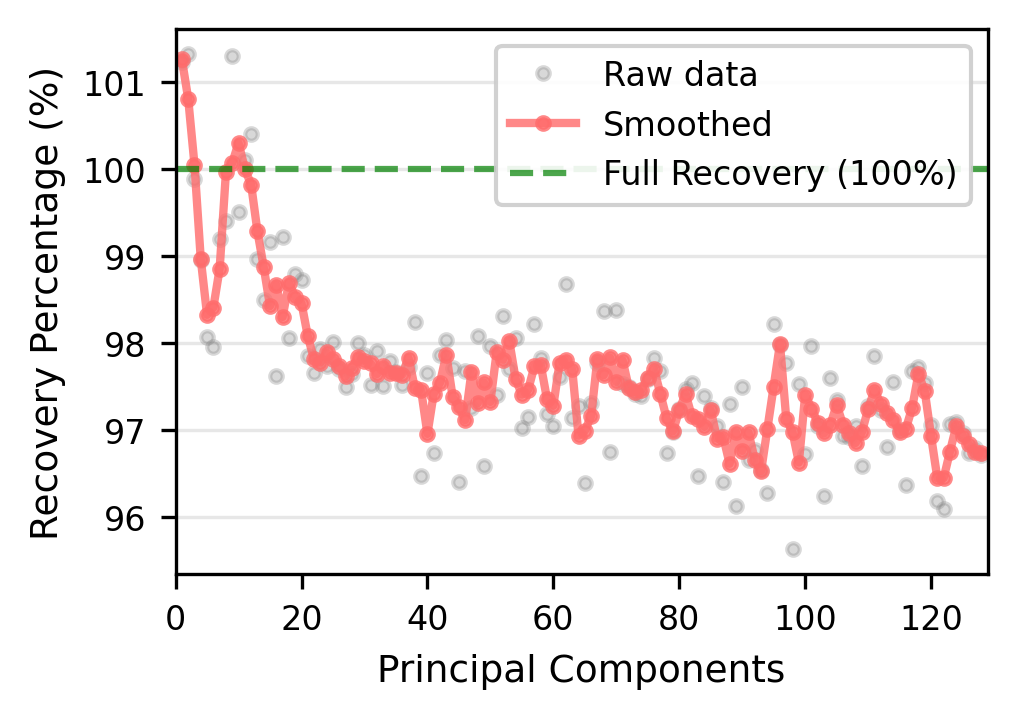} \\[1pt]
\multicolumn{2}{c}{\footnotesize\textit{$r=64$}} \\
\includegraphics[width=0.385\linewidth]{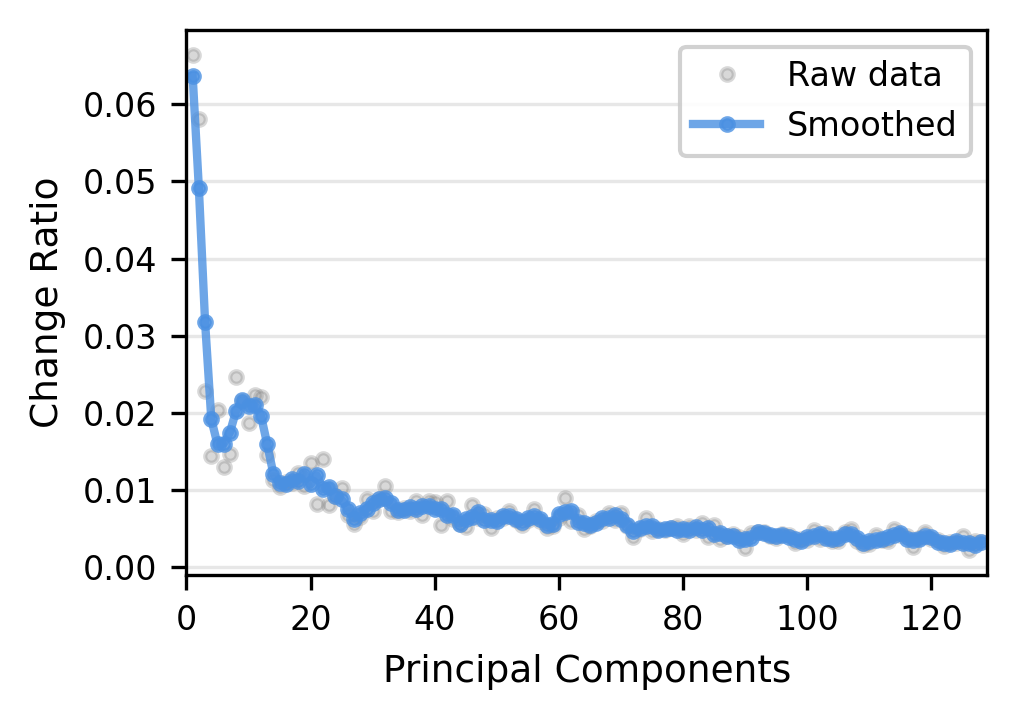} &
\includegraphics[width=0.385\linewidth]{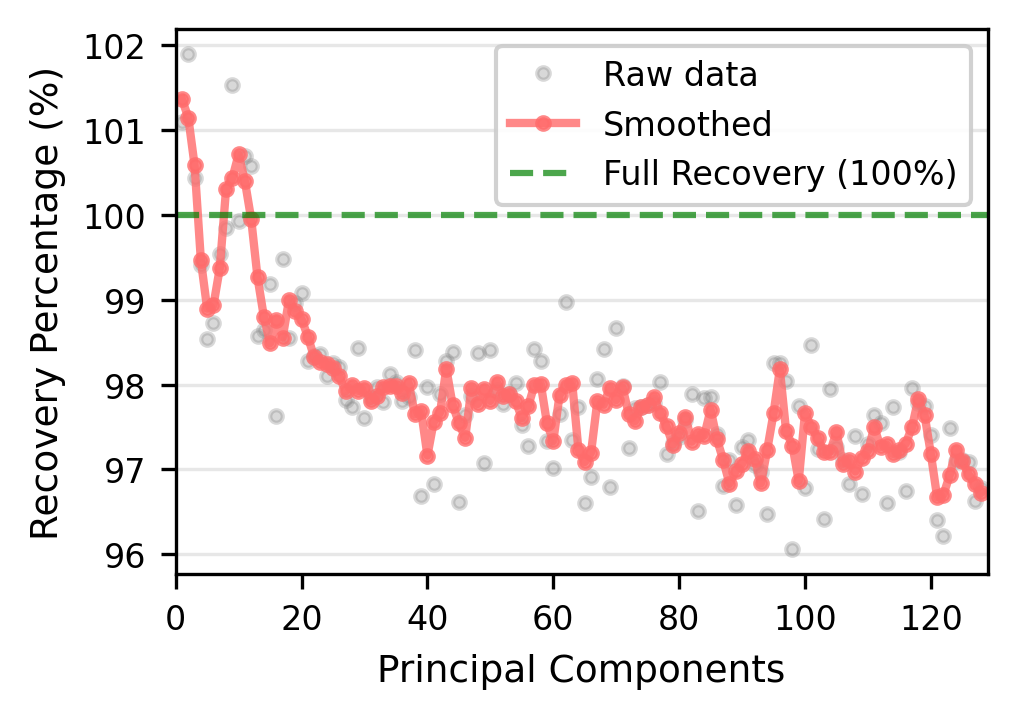} \\
\end{tabular}
\caption{GradDiff under Full FT and LoRA at four ranks. Per-PC unlearn change ratio (left) and recovery ratio (right) on WMDP-Cyber (Llama-3.1-8B). Dominant components dominate both change and recovery for every optimization regime.}
\label{fig:lora_graddiff}
\end{figure}

\begin{figure}[t]
\centering
\setlength{\tabcolsep}{2pt}
\renewcommand{\arraystretch}{0.68}
\begin{tabular}{@{}c@{\hspace{4pt}}c@{}}
\textbf{\footnotesize Unlearn Change Ratio} & \textbf{\footnotesize Recovery Ratio} \\[2pt]
\multicolumn{2}{c}{\footnotesize\textit{Full FT}} \\
\includegraphics[width=0.385\linewidth]{figures/loss_consistency/npo_unlearn_change_ratio.png} &
\includegraphics[width=0.385\linewidth]{figures/loss_consistency/npo_recovery_ratio.png} \\[1pt]
\multicolumn{2}{c}{\footnotesize\textit{$r=8$}} \\
\includegraphics[width=0.385\linewidth]{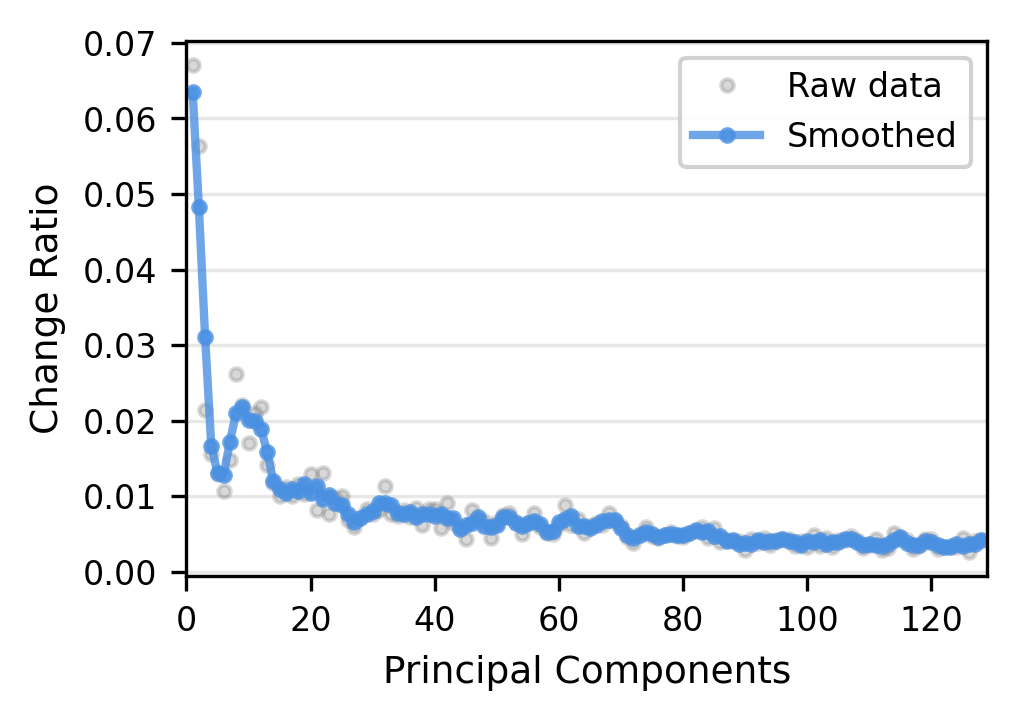}  &
\includegraphics[width=0.385\linewidth]{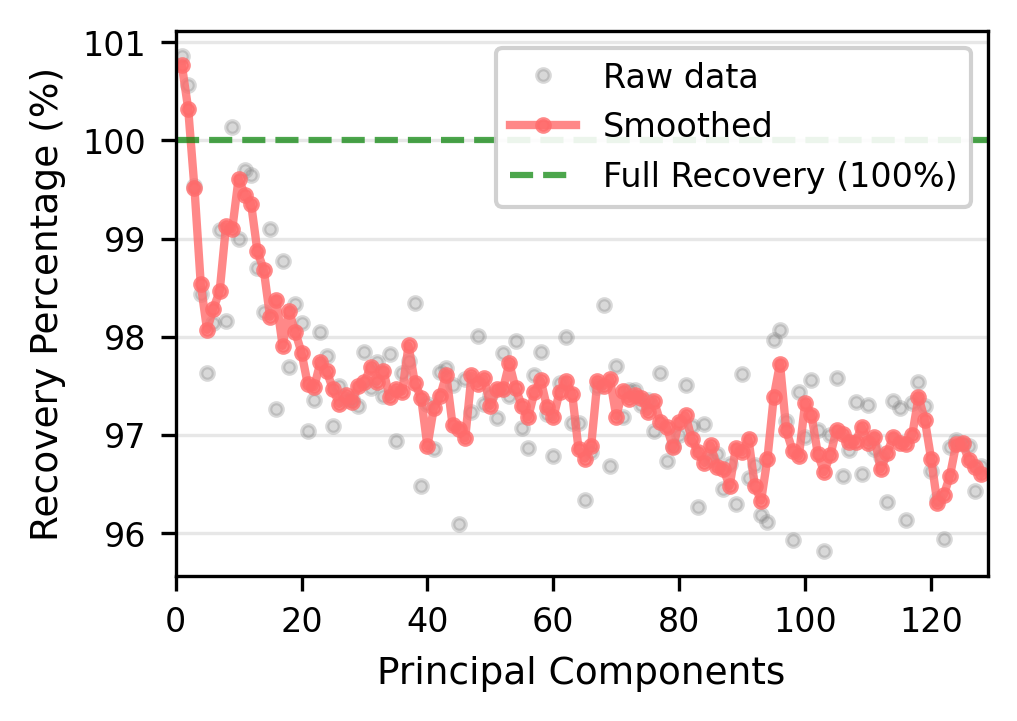}  \\[1pt]
\multicolumn{2}{c}{\footnotesize\textit{$r=16$}} \\
\includegraphics[width=0.385\linewidth]{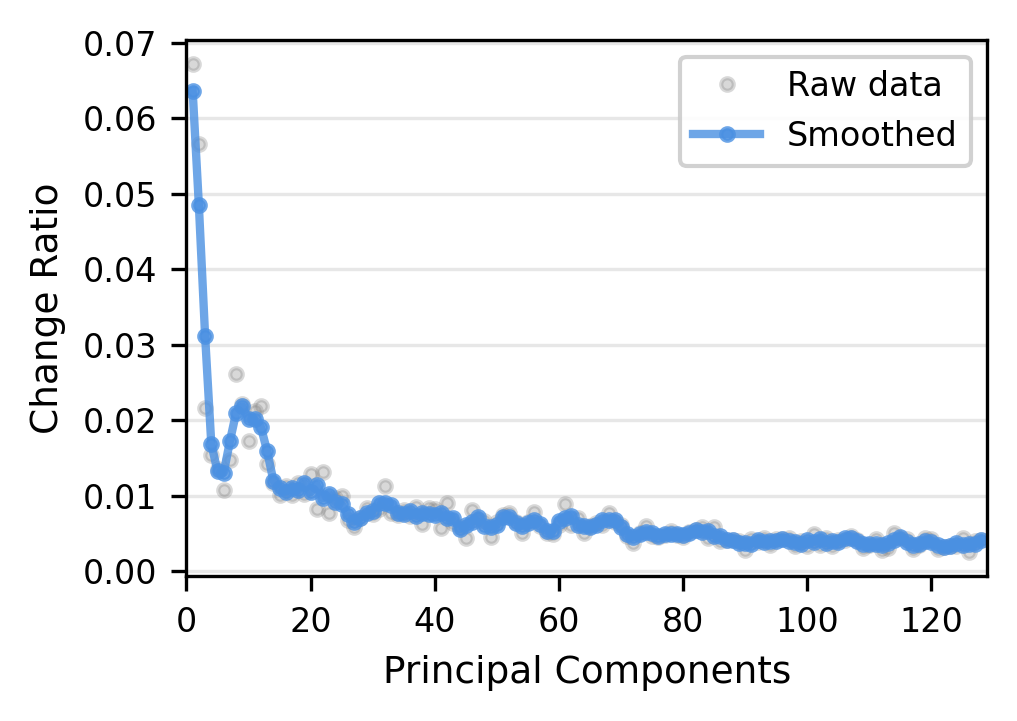} &
\includegraphics[width=0.385\linewidth]{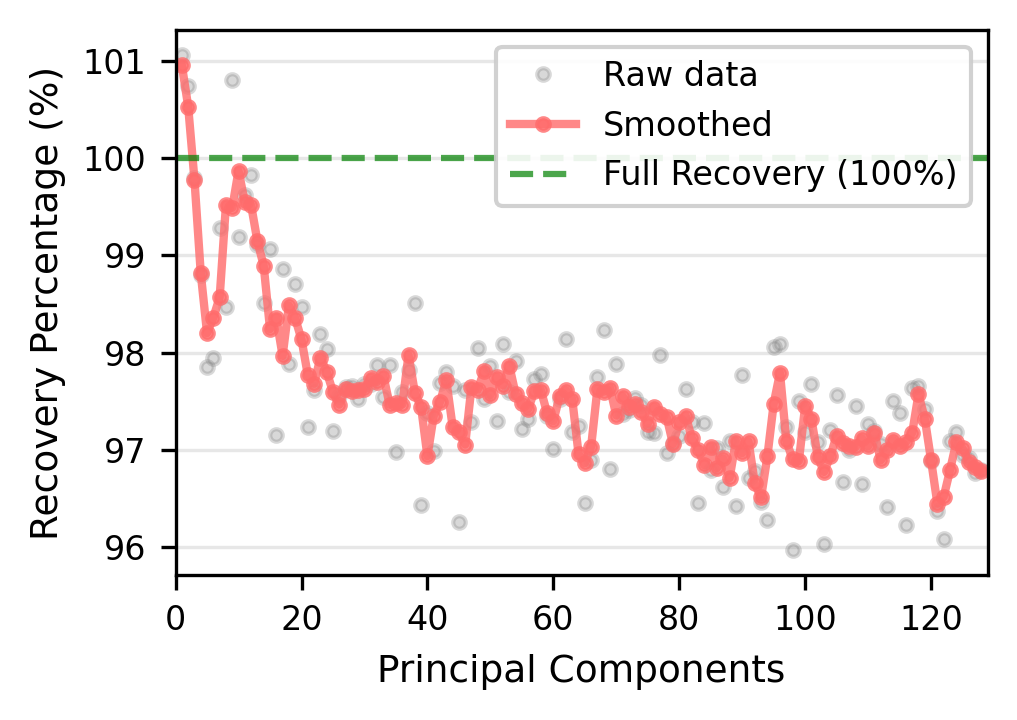} \\[1pt]
\multicolumn{2}{c}{\footnotesize\textit{$r=32$}} \\
\includegraphics[width=0.385\linewidth]{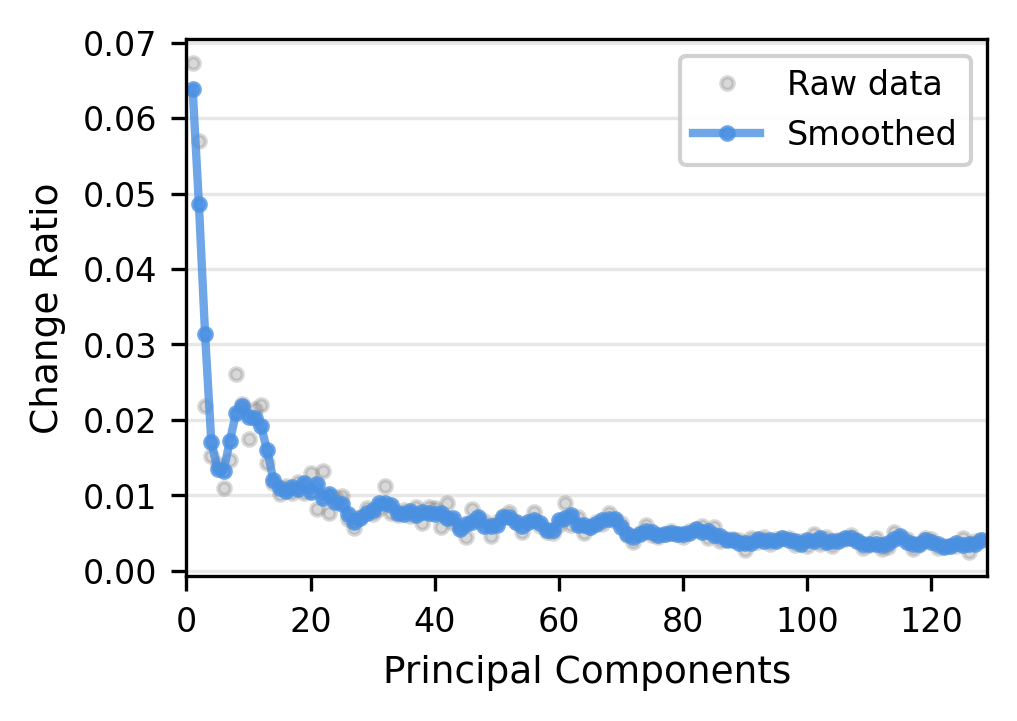} &
\includegraphics[width=0.385\linewidth]{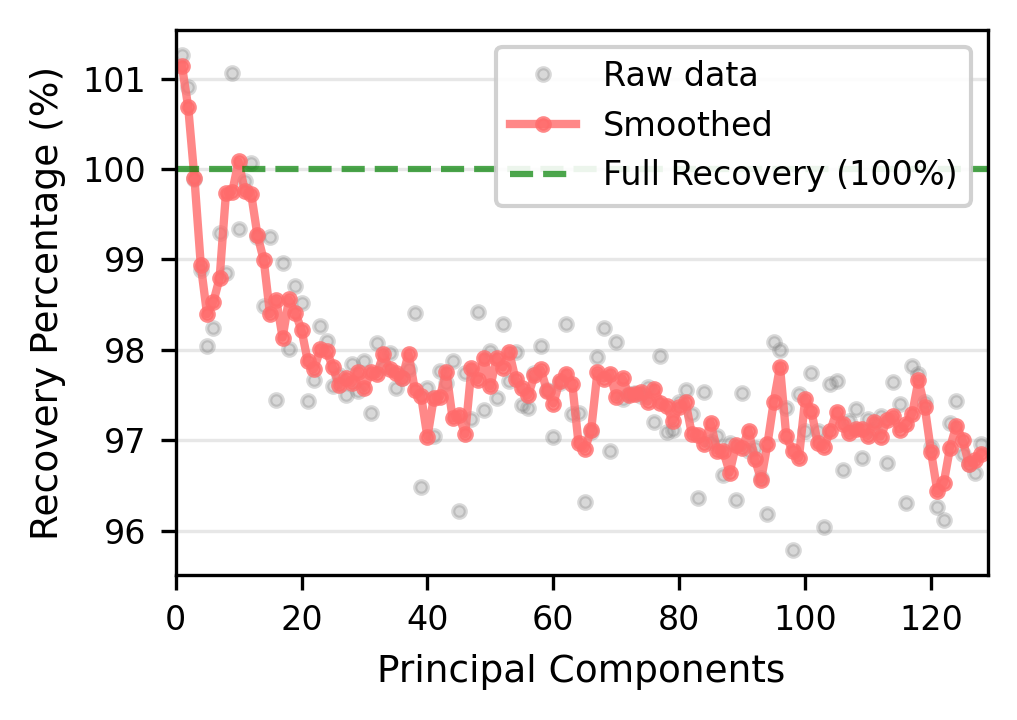} \\[1pt]
\multicolumn{2}{c}{\footnotesize\textit{$r=64$}} \\
\includegraphics[width=0.385\linewidth]{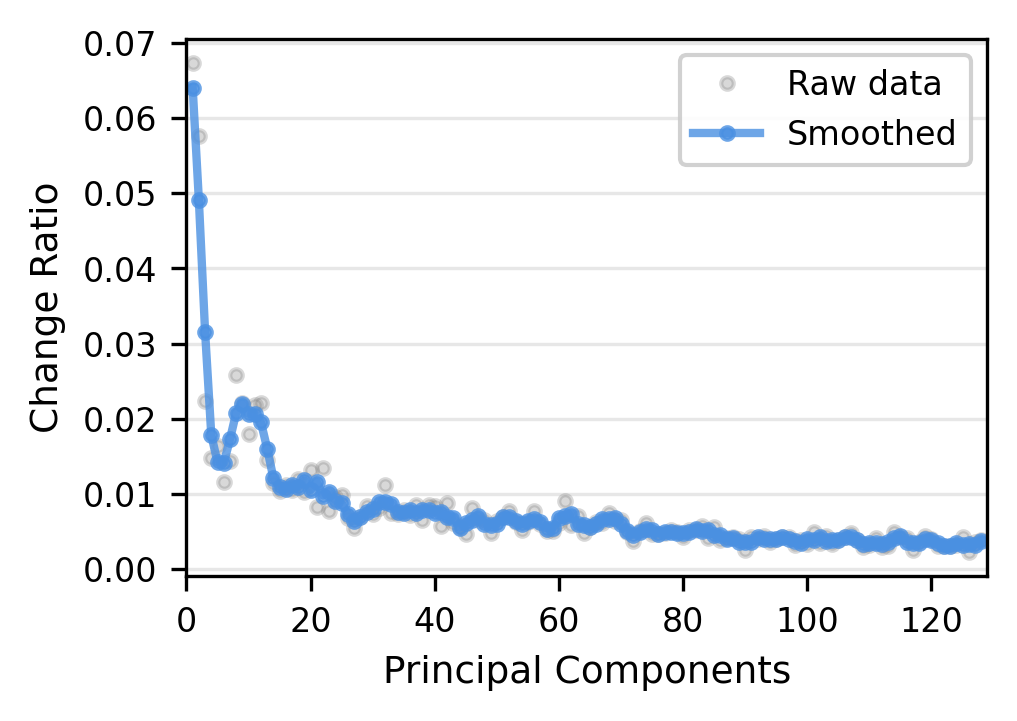} &
\includegraphics[width=0.385\linewidth]{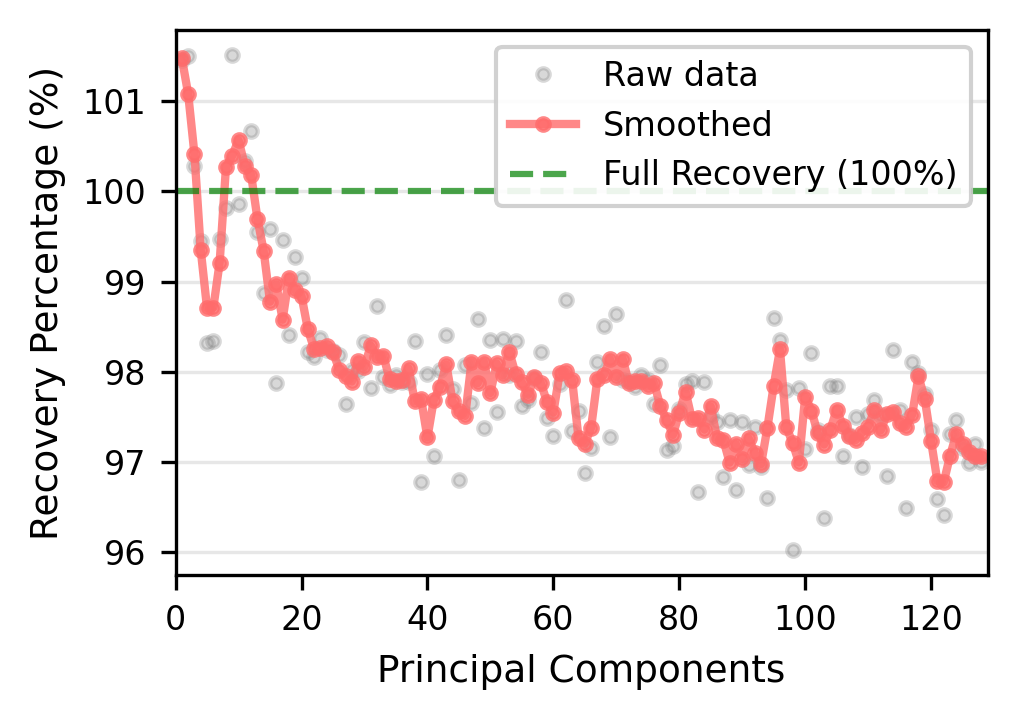} \\
\end{tabular}
\caption{NPO under Full FT and LoRA at four ranks. Same layout as \Cref{fig:lora_graddiff}.}
\label{fig:lora_npo}
\end{figure}

\begin{figure}[t]
\centering
\setlength{\tabcolsep}{2pt}
\renewcommand{\arraystretch}{0.68}
\begin{tabular}{@{}c@{\hspace{4pt}}c@{}}
\textbf{\footnotesize Unlearn Change Ratio} & \textbf{\footnotesize Recovery Ratio} \\[2pt]
\multicolumn{2}{c}{\footnotesize\textit{Full FT}} \\
\includegraphics[width=0.385\linewidth]{figures/loss_consistency/rmu_unlearn_change_ratio.png} &
\includegraphics[width=0.385\linewidth]{figures/loss_consistency/rmu_recovery_ratio.png} \\[1pt]
\multicolumn{2}{c}{\footnotesize\textit{$r=8$}} \\
\includegraphics[width=0.385\linewidth]{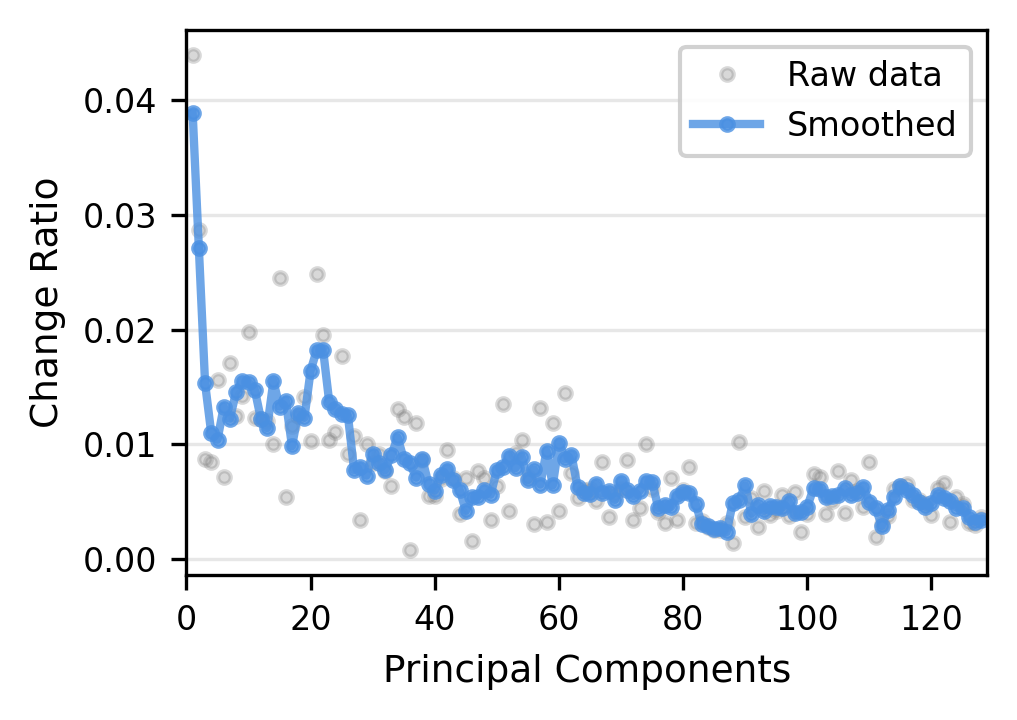}  &
\includegraphics[width=0.385\linewidth]{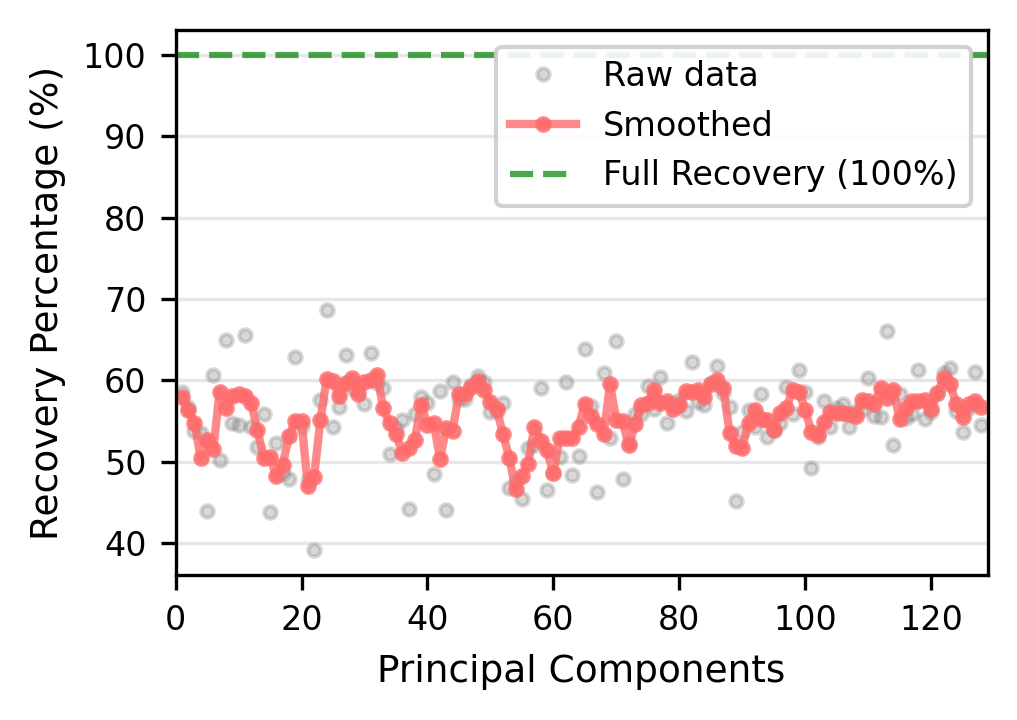}  \\[1pt]
\multicolumn{2}{c}{\footnotesize\textit{$r=16$}} \\
\includegraphics[width=0.385\linewidth]{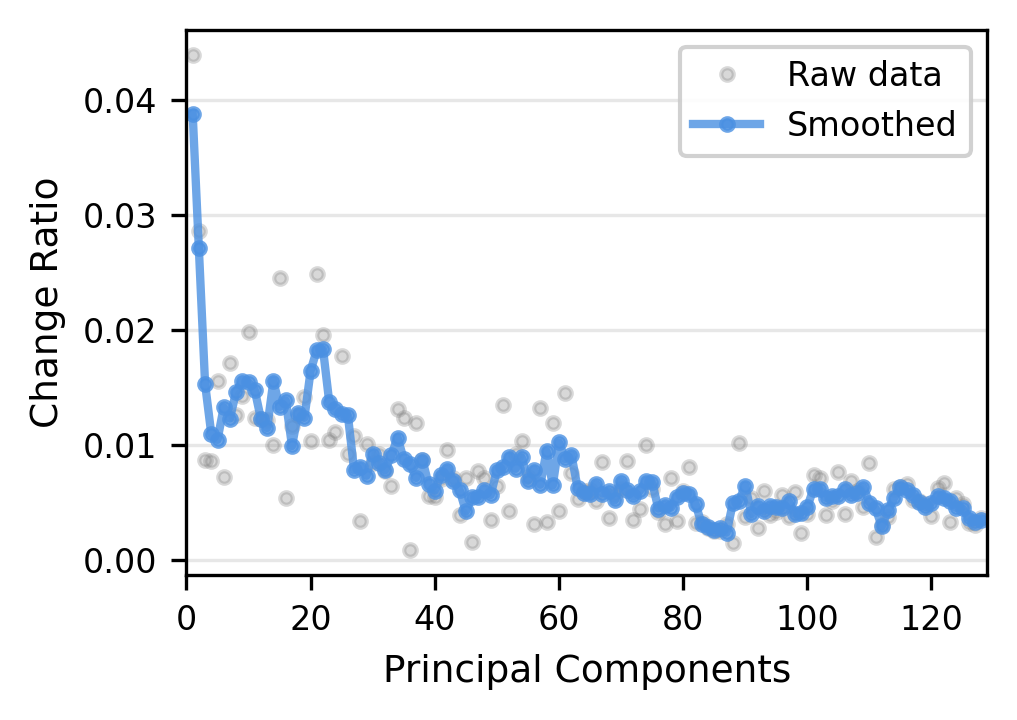} &
\includegraphics[width=0.385\linewidth]{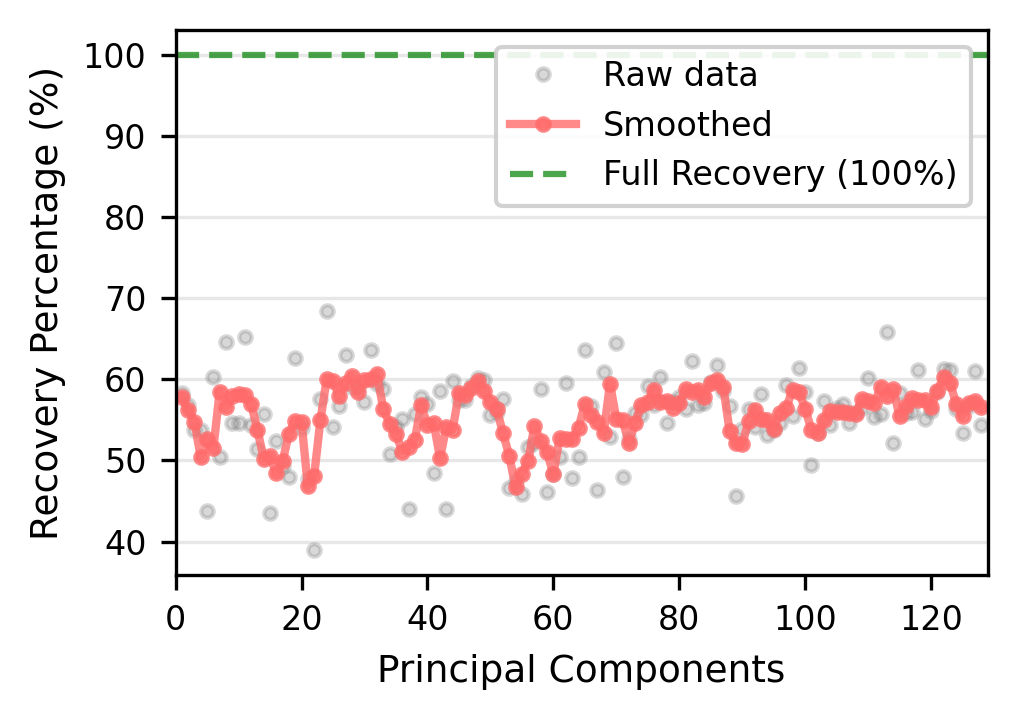} \\[1pt]
\multicolumn{2}{c}{\footnotesize\textit{$r=32$}} \\
\includegraphics[width=0.385\linewidth]{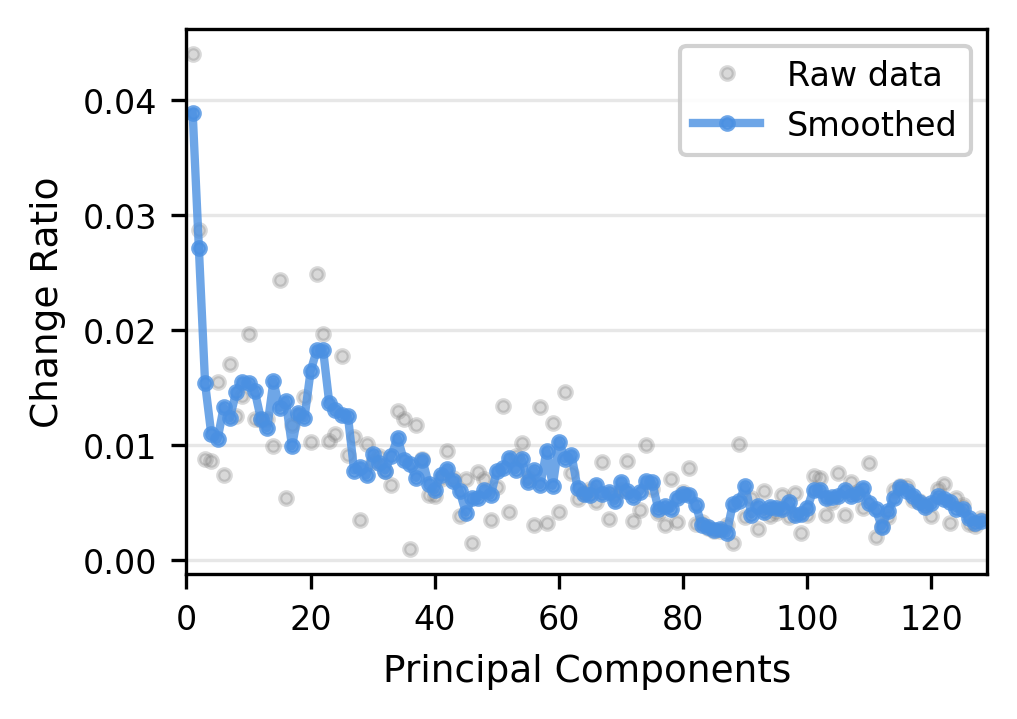} &
\includegraphics[width=0.385\linewidth]{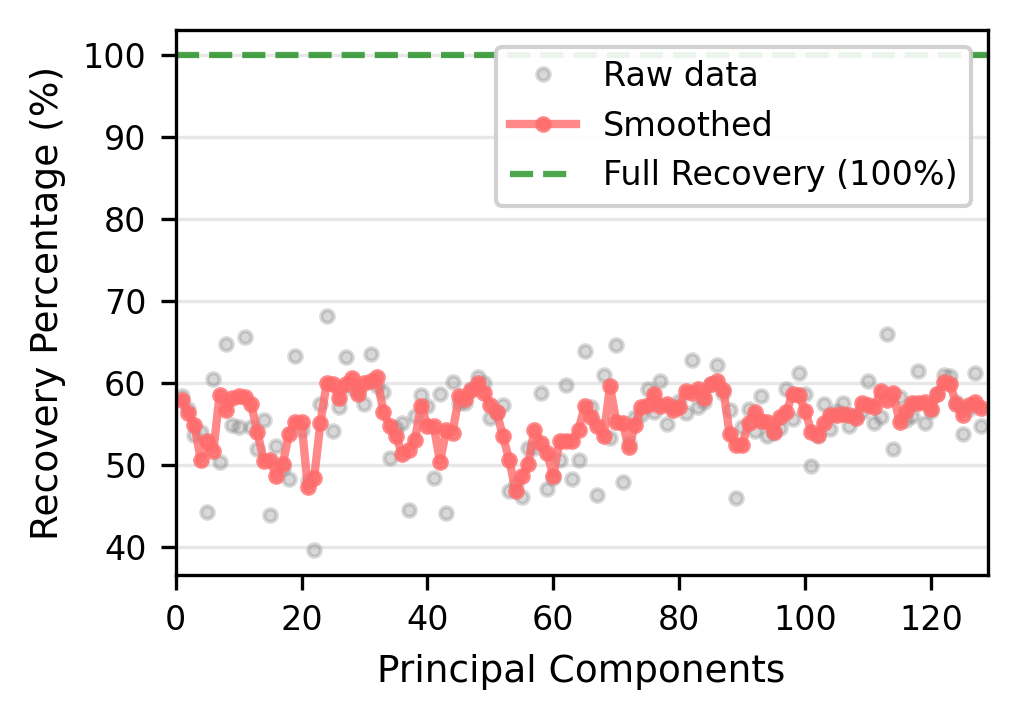} \\[1pt]
\multicolumn{2}{c}{\footnotesize\textit{$r=64$}} \\
\includegraphics[width=0.385\linewidth]{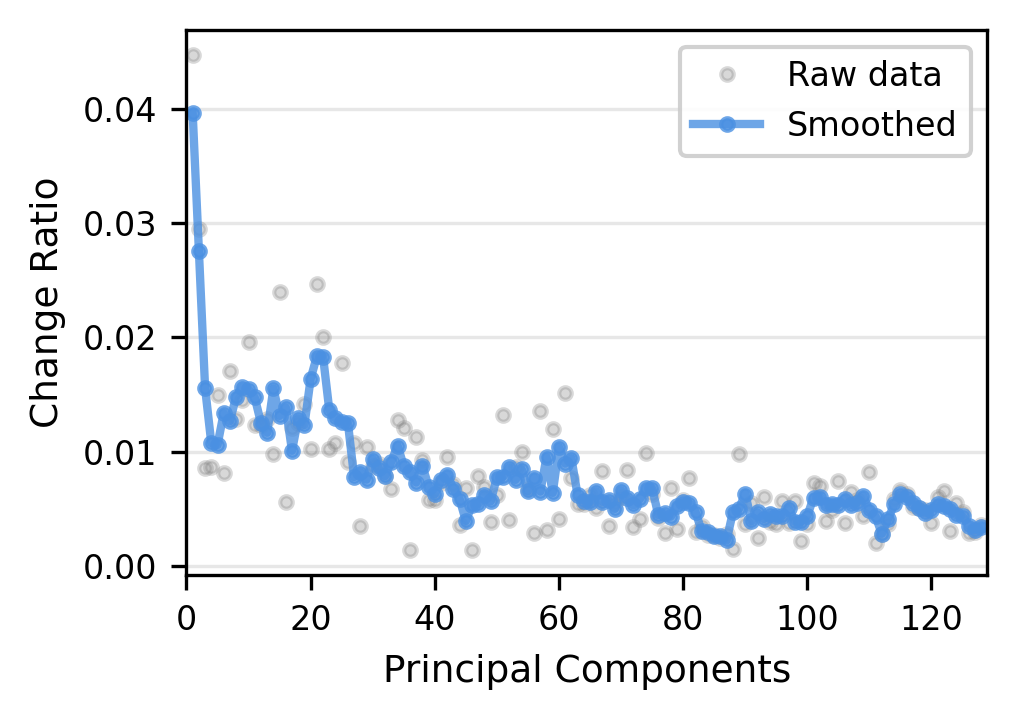} &
\includegraphics[width=0.385\linewidth]{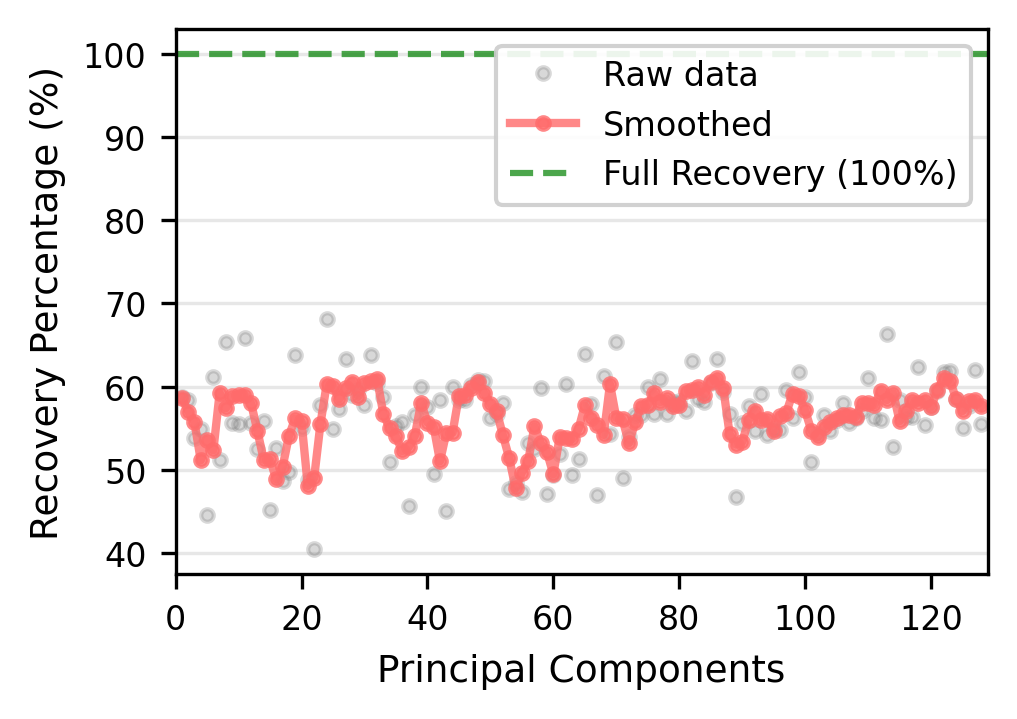} \\
\end{tabular}
\caption{RMU under Full FT and LoRA at four ranks. Same layout as \Cref{fig:lora_graddiff}.}
\label{fig:lora_rmu}
\end{figure}

\begin{figure}[t]
\centering
\setlength{\tabcolsep}{2pt}
\renewcommand{\arraystretch}{0.68}
\begin{tabular}{@{}c@{\hspace{4pt}}c@{}}
\textbf{\footnotesize Unlearn Change Ratio} & \textbf{\footnotesize Recovery Ratio} \\[2pt]
\multicolumn{2}{c}{\footnotesize\textit{Full FT}} \\
\includegraphics[width=0.385\linewidth]{figures/loss_consistency/mlpbreak_unlearn_change_ratio.png} &
\includegraphics[width=0.385\linewidth]{figures/loss_consistency/mlpbreak_recovery_ratio.png} \\[1pt]
\multicolumn{2}{c}{\footnotesize\textit{$r=8$}} \\
\includegraphics[width=0.385\linewidth]{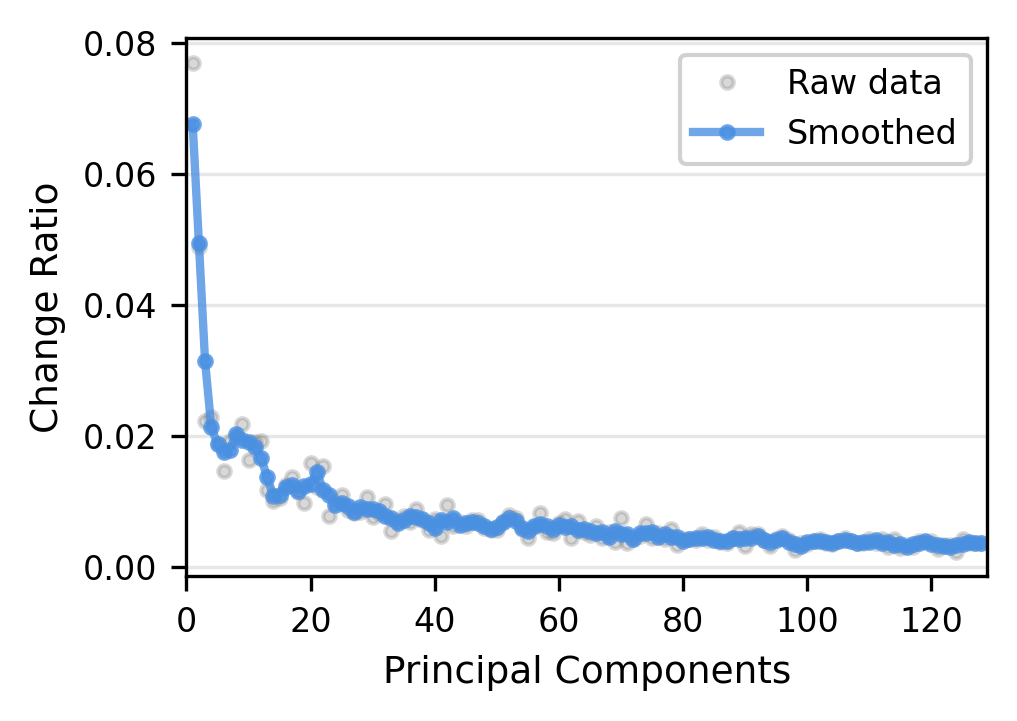}  &
\includegraphics[width=0.385\linewidth]{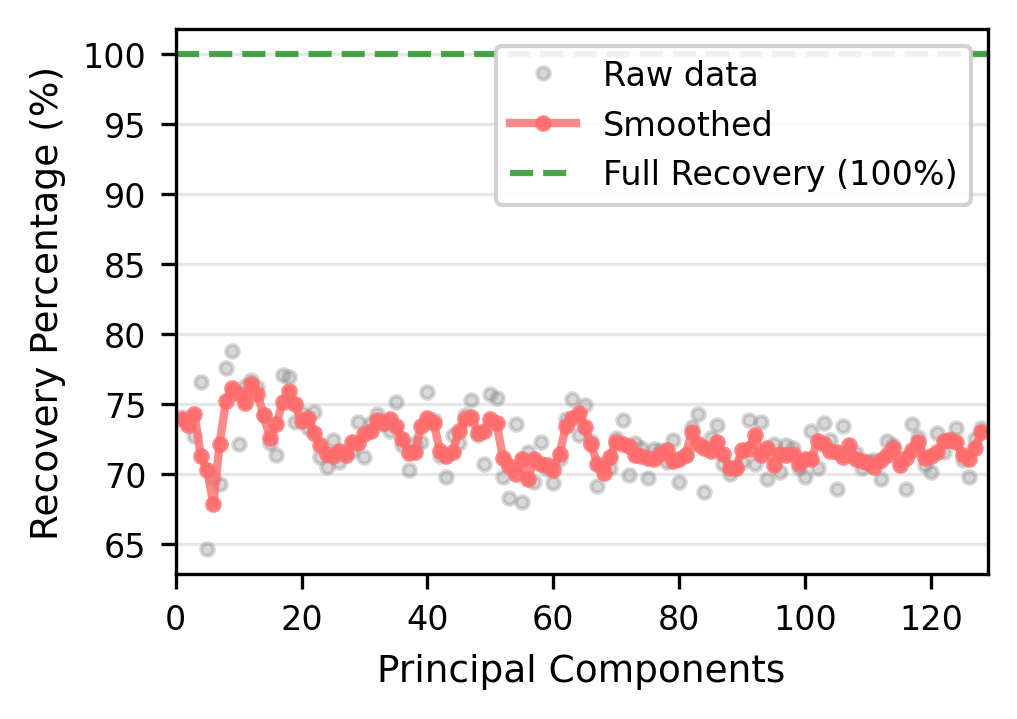}  \\[1pt]
\multicolumn{2}{c}{\footnotesize\textit{$r=16$}} \\
\includegraphics[width=0.385\linewidth]{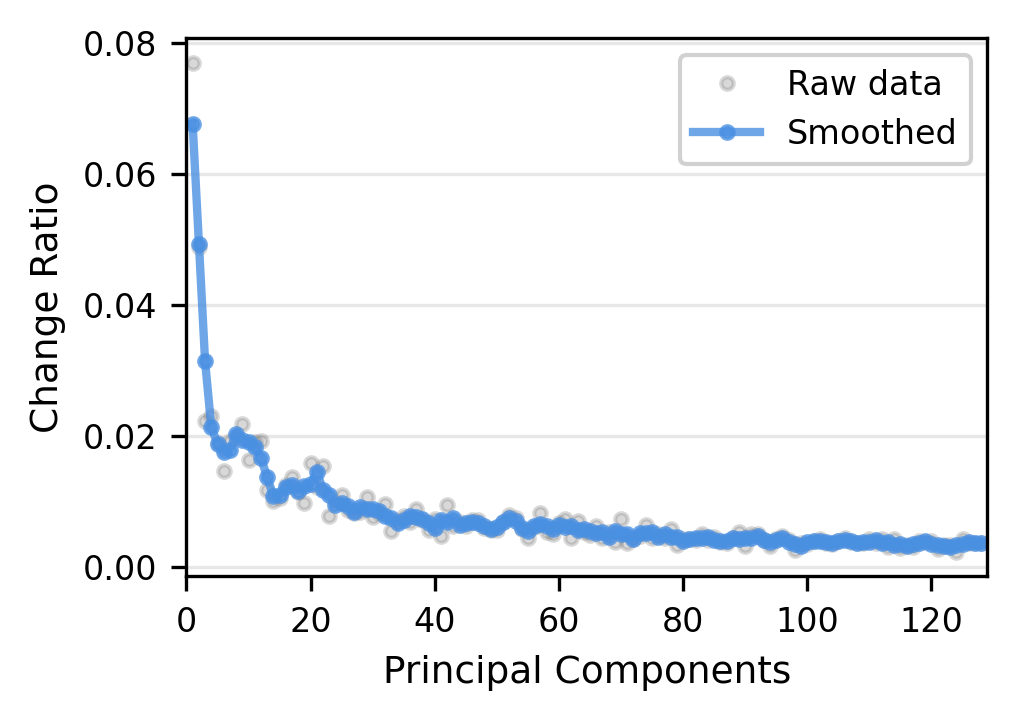} &
\includegraphics[width=0.385\linewidth]{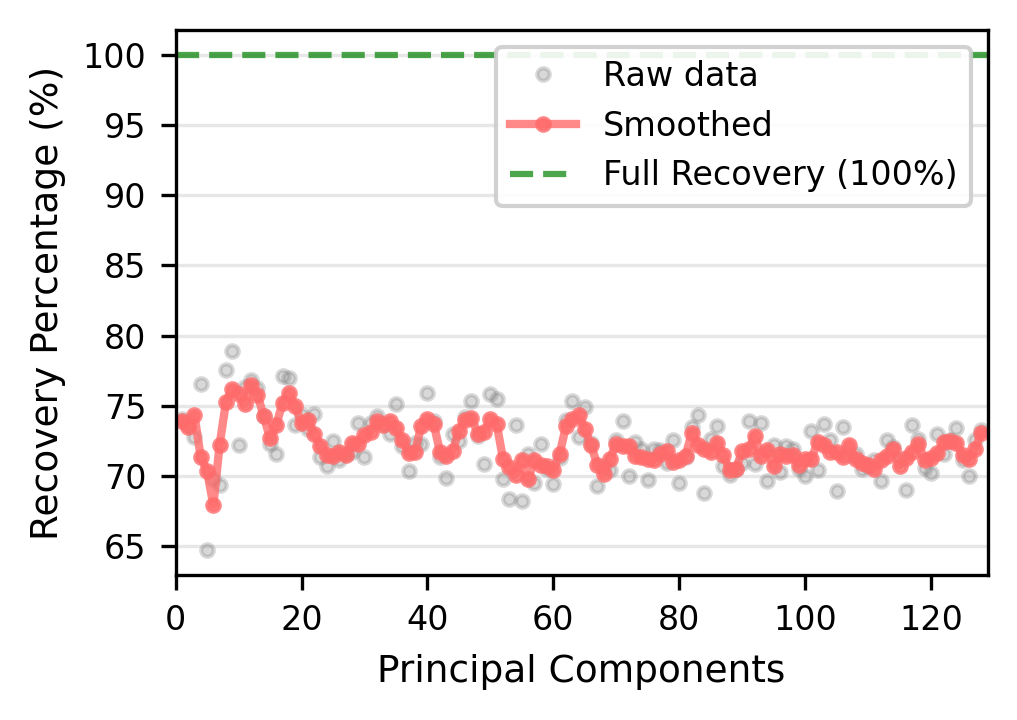} \\[1pt]
\multicolumn{2}{c}{\footnotesize\textit{$r=32$}} \\
\includegraphics[width=0.385\linewidth]{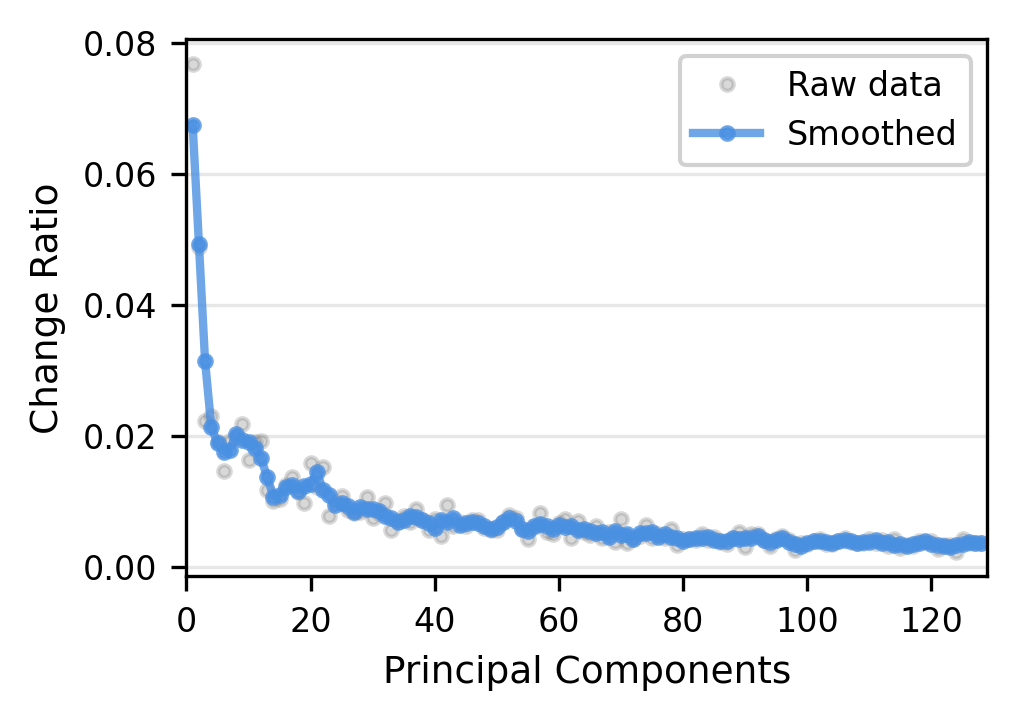} &
\includegraphics[width=0.385\linewidth]{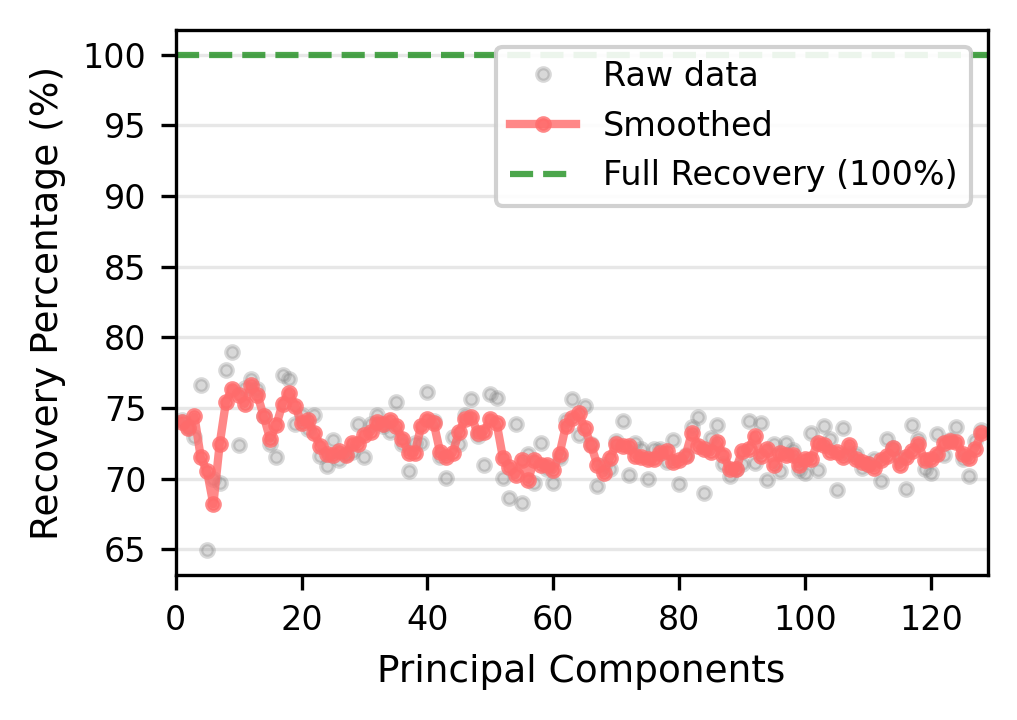} \\[1pt]
\multicolumn{2}{c}{\footnotesize\textit{$r=64$}} \\
\includegraphics[width=0.385\linewidth]{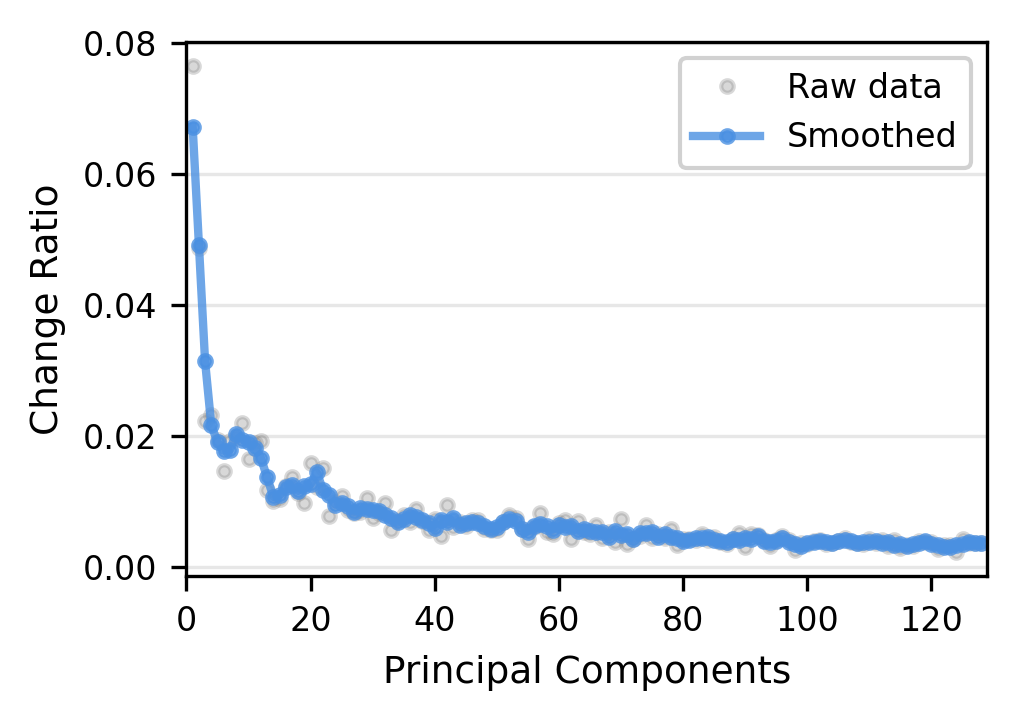} &
\includegraphics[width=0.385\linewidth]{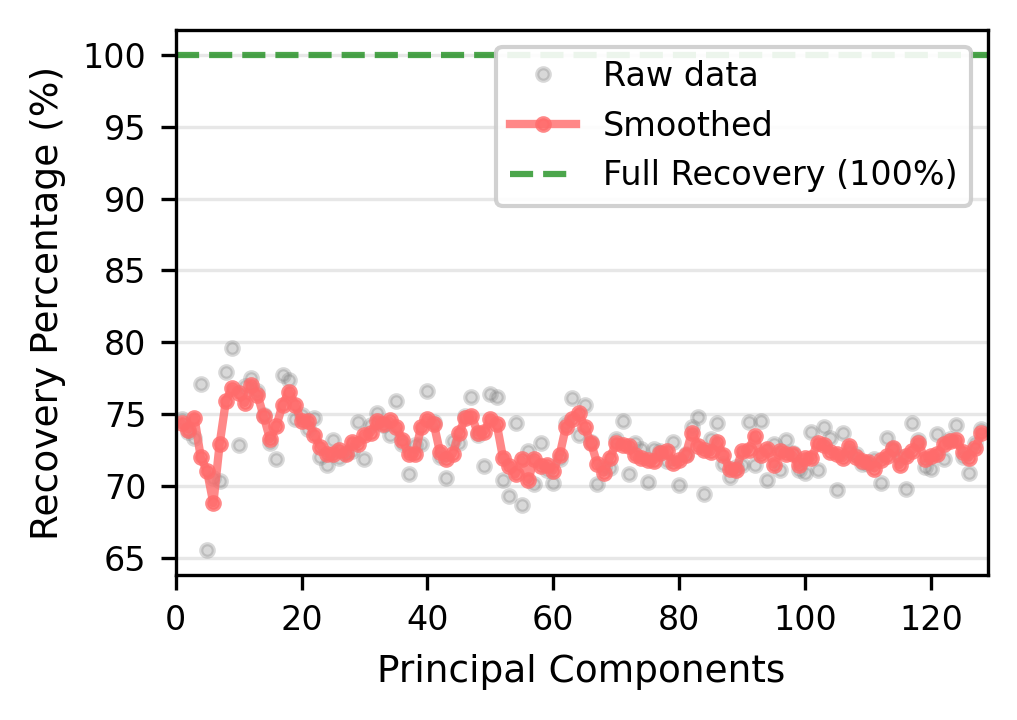} \\
\end{tabular}
\caption{MLP Breaking under Full FT and LoRA at four ranks. Same layout as \Cref{fig:lora_graddiff}.}
\label{fig:lora_mlpconfuse}
\end{figure}

\end{document}